\documentclass[11pt,a4paper]{article}
\usepackage[utf8]{inputenc}

\usepackage{amssymb}
\usepackage{amsmath}
\usepackage{amsthm}
\usepackage{amsfonts}
\usepackage[langfont=typewriter,funcfont=slant]{complexity}
\usepackage[left=1in,bottom=1in,top=1in,right=1in]{geometry}
\usepackage[colorlinks=true]{hyperref}
\usepackage{framed}
\usepackage{algorithm}
\usepackage[noend]{algpseudocode}
\usepackage{soul}
\usepackage{thmtools,thm-restate}
 
 \usepackage[disable]{todonotes} 

\usepackage{commath} 

\theoremstyle{definition}
\newtheorem{defi}{Definition}[section]

\theoremstyle{remark}
\newtheorem*{rem}{Remark}

\theoremstyle{plain}
\newtheorem{theorem}{Theorem}[section]

\theoremstyle{plain}
\newtheorem{cor}{Corollary}[theorem]

\theoremstyle{plain}

\theoremstyle{plain}
\newtheorem{lem}[theorem]{Lemma}
\newtheorem{claim}[theorem]{Claim}
\newtheorem{fact}[theorem]{Fact}

\newclass{\NPH}{NPH}

\newcommand{\br}{\mathbb{R}}

\newcommand{\ip}[2]{\left\langle #1, #2 \right\rangle}

\DeclareMathOperator{\Ex}{\mathbb{E}}
\DeclareMathOperator{\tr}{\mathsf{tr}}
\DeclareMathOperator{\dt}{\mathsf{det}}

\title{\textbf{Non-Gaussian Component Analysis using Entropy Methods}}
\author{Navin Goyal \thanks{Microsoft Research India, Email: navingo@microsoft.com} \and Abhishek Shetty \thanks{Microsoft Research India, Email: ashetty1995@gmail.com}}

\date{}
\begin{document}
	\maketitle

	\begin{abstract}
        Non-Gaussian component analysis (NGCA) is a problem in multidimensional data analysis which, since its formulation in 2006, has attracted considerable attention in statistics and machine learning. 
        In this problem, we have a random variable $X$ in $n$-dimensional Euclidean space. There is an unknown subspace $\Gamma$ of the $n$-dimensional Euclidean space such that the orthogonal projection of $X$ onto $\Gamma$ is standard multidimensional Gaussian and the orthogonal projection of $X$ onto $\Gamma^{\perp}$, the orthogonal complement of $\Gamma$, is non-Gaussian, in the sense that all its one-dimensional marginals are different from the Gaussian in a certain metric defined in terms of moments. The NGCA problem is to approximate the non-Gaussian subspace $\Gamma^{\perp}$ given samples of $X$. 
        
         Vectors in $\Gamma^{\perp}$ correspond to `interesting' directions, whereas vectors in $\Gamma$ correspond to the directions where data is very noisy. 
        The most interesting applications of the NGCA model is for the case when the magnitude of the noise is comparable to that of the true signal, a setting in which traditional noise reduction techniques such as PCA don't apply directly. 
        NGCA is also related to dimension reduction and to other data analysis problems such as ICA. 
        NGCA-like problems have been studied in statistics for a long time using techniques such as projection pursuit. 
        
         We give an algorithm that takes polynomial time in the dimension $n$ and has an inverse polynomial dependence on the error parameter measuring the angle distance between the non-Gaussian subspace and the subspace output by the algorithm. 
         Our algorithm is based on relative entropy as the contrast function and fits under the projection pursuit framework. 
         The techniques we develop for analyzing our algorithm maybe of use for other related problems.
	\end{abstract}
	
	\newpage
	
	\section{Introduction} 
    \subsection{Motivation}
	The problem of finding `interesting directions' in high-dimensional data is a basic problem in multivariate data analysis. 
	More precisely, we want to find unit vectors $u$ such that the one-dimensional marginal distributions along $u$ obtained by projecting the data to	$u$ satisfy the required properties. 
	This has been formalized in various ways, for example PCA (principal component analysis), ICA (independent component analysis) and factor analysis. 
	 Such questions are also studied under the heading of dimension reduction. 
     \todo[inline]{Write about subspace Junta testing}
	The focus of the present paper is a different formalization called \emph{non-Gaussian component analysis} or NGCA. \cite{blanchard2006search}. 
	
	In the NGCA problem, one assumes that the data in the interesting directions is
	non-Gaussian, whereas in the other directions it is noisy which is modeled as Gaussian \emph{independent} of the projections in interesting directions. 
    Moreover, one assumes that the Gaussian components form a subspace. The subspace orthogonal to this subspace is the non-Gaussian subspace. 
    The problem is then to find these subspaces. 
    We are interested in the setting where the magnitude of the noise is comparable to the signal and consequently applying magnitude based de-noising algorithms are not suitable. 
    One then needs to resort to structural characterizations of the noise to isolate noisy directions.
	Let us formally state the problem: Let $X \in \br^n$ be an isotropic random vector, i.e., $\Ex X =0$ and $\Ex XX^T = I_n$. 
    Assume that there is a subspace $\Gamma \subset \br^n$ such that $X = (Z,\tilde{X}) \in \Gamma \oplus \Gamma^{\perp }$ where $ Z \sim \mathcal{N}(0,I_{\dim{\Gamma}})$ (the standard Gaussian distribution in $\Gamma $), and $\tilde{X}$ is non-Gaussian (we explain this shortly). 
	Furthermore, $\tilde{X}$ and $Z$ are independent.
    The problem is to estimate the non-Gaussian subspace, $\Gamma^{\perp}$, from samples of $X$.
    The assumption of isotropy and orthogonality is without loss of generality in our setting because one can whiten the data to arrive at this model. 
    We would like our algorithm to have computational and sample complexity polynomial in the dimension $n$ and other problem parameters. 
    
    In previous work, the NGCA problem has been stated in several equivalent ways. 
    %
    A prevalent way of stating the problem is the following:
    Consider a subspace $ E \subset \br^n $. Let $ Y  $ be a random variable in $\br^n$ supported on $ E $ and let $ Z  $ be a Gaussian in $\br^n$ with arbitrary (invertible) covariance matrix. Now, given samples to $ W = Y + Z$, we need to recover $ E $. 
    One can show that $ \Sigma_W^{-\frac{1}{2}} W   $, where $ \Sigma_W = \mathbb{E}\,  WW^{T}$ is the covariance matrix of $ W $, satisfies the \hyperref[Iso_NGCA]{ isotropic NGCA}  model as defined below. 
    For further discussion on the equivalence of the various models, see \cite{tan2017polynomial}. 
    
	One needs to specify some measure of difference between $\tilde{X}$ and the multidimensional Gaussian.
    One way of doing it is to say that all the one-dimensional marginals (projections) of $\tilde{X}$ are different from the one dimensional Gaussian.
    Note that in order to recover the entire Gaussian subspace, we make an identifiabilility assumption on each marginal.
    For example, as in \cite{tan2017polynomial}, if we only assume that at least one direction is far from Gaussian, we cannot hope to recover the entire original subspace. 
    One can also easily construct examples to show that notions such as total variation distance from the Gaussian supported on the subspace do not suffice to recover the original subspace. 
    A standard condition for distinguishability is that at least one of the first $r$ moments of the marginals of $\tilde{X}$ is sufficiently different from the corresponding moments of $Z$. 
    Indeed,  \cite{tan2017polynomial} make an assumption of this type for their algorithm for recovering the full non-Gaussian subspace (see \hyperref[sec:prev_work]{Section \ref*{sec:prev_work}} for more details).
    We will use such a condition in our analysis.
    We emphasize that apart from the above constraints of being different from the Gaussian, no other constraints are imposed on $\tilde{X}$, and there can be arbitrary dependencies between its components.
    This sets NGCA apart from ICA where one assumes that in the right coordinate system $X$ can be represented as $(X_1, \ldots, X_n)$ where $ X_i $ are mutually independent.
    Though many efficient algorithms for ICA (i.e., the problem of finding the right change of coordinate system) are known, these techniques do not seem to apply directly to the problem at hand.
	
	Projection pursuit is a generic technique for finding interesting directions in high-dimensional data (see e.g. \cite{huber1985} for an expository account). 
	Here one looks at one-dimensional marginals of data and computes a function of this marginal known as projection index or contrast function. 
    One then tries to find directions that maximize or minimize this contrast function by an optimization procedure. 
    The	choice of the contrast function is made so that these extrema characterize interesting directions.
    In particular, this general idea is the basis of many algorithms for ICA \cite{HyvarinenBook} and has also been used to give heuristic algorithms for NGCA. \todo[inline]{references?}
	

	\subsection{Related Work} \label{sec:prev_work}
	Since the formal definition of the NGCA problem by \cite{blanchard2006search} a steady line of work, e.g., 
	\cite{Kawanabe2006, Kawanabe2007, diederichs2010sparse, diederichs2013sparse, phdthesisBean, SasakiNS16, Virta2016} has provided a number of algorithms for this problem. 
    Some of these are based on projection pursuit while others use different methods.
    In particular, \cite{blanchard2006search} provide an algorithm based on Stein's characterization of the Gaussian random variable and \cite{diederichs2013sparse} use semi-definite programming. 
	We remark that while the NGCA problem resembles ICA problem, the algorithms for the latter do not seem to be directly applicable to the former.
    The main difference between the two problems is that in NGCA $\tilde{X}$ is not assumed to have independent components. 
	
%
    
    In this paper we are concerned with provably efficient algorithms both in terms of computation and the number of samples needed.
    Our main focus is on getting polynomial dependence on the dimension $n$ and the error parameter.
    Most of the existing papers do not provide such an analysis (see also the discussion in \cite{tan2017polynomial}). 
    We will now briefly discuss the two works \cite{vempala2011structure, tan2017polynomial} that do provide provable bounds with finite sample complexity; however these do not solve the full NGCA problem in polynomial time. 
    
    	 \cite{vempala2011structure} were the first to give algorithms with rigorous guarantees to the best of our knowledge. They give a projection pursuit style algorithm with the moments of the projection serving as the contrast function which works in polynomial time when the non-Gaussian component is
    	 constant dimensional. This work was independent of the line of work on NGCA and independently formalized the problem along with reasonable moment-based assumptions that have been influential in subsequent theoretical work.
    Their algorithm has running time depending doubly exponentially on the dimension of the non-Gaussian subspace. 
    The non-Gaussian part is assumed to be log-concave and the difference from the Gaussian is quantified using a moment condition in each direction.
    This moment	condition is similar but not identical to ours. 

	In a follow up work, \cite{tan2017polynomial} provide an elegant spectral algorithm for NGCA, which they dub ``Reweighted PCA," achieving similar guarantees as of \cite{vempala2011structure}. 
    The notion of ``Reweighted PCA" appeared in the the work of \cite{brubaker2008isotropic}, where it was used not for NGCA but for clustering. 
	 Their algorithm arises from a characterization of the Gaussian random variable $X \in \br^n$ in terms of the distribution of the norm $ \norm{X} $ and the distribution of the inner product of independent copies of the random variable $ \left\langle X, X' \right\rangle$.
	 Using a matrix version of the above claim, they show that their algorithm, under the assumption at least one of the directions in the non-Gaussian space is far from Gaussian in terms of moments, recovers at least one non-Gaussian component. 
    Under the added assumption that all the non-Gaussian marginals are far from Gaussian, they show that iterating their algorithm recovers the whole subspace, albeit in time exponential in the dimension of the non-Gaussian subspace.
    They conjecture that under this assumption their algorithm recovers the whole subspace in time polynomial in the dimension of the non-Gaussian subspace.
    
	
 \todo[inline]{Though XV say that they are using "one of the first k moments" condition, even their main "structure from local optima" claim needs lower moments to match. It is not even clear that using OU like argument as in ours this can be remedied (as, like we have already noted, the moments could potentially behave in strange ways when lower moments are non-zero). We should note that the fact the entropy has only local minima in either the Gaussian or non-Guassian subspaces is kind of special, and follows from monotonicity under the OU action and unless we are looking at a quantity that is indeed monotone under this action, it is unlikely that such structure results are true (even if it is true it is much harder to prove than the proof suggest in XV, since paths themselves don't carry enough information to make this claim. One needs to somehow prove that these paths don't "conspire" to have minima at the same location. This needs more delicate analysis. At least I don't immediately see how to do this, unless I am missing a property of moments implicitly assumed in XV that I haven't taken into account.) }
	
	\paragraph{Cumulant-based approaches.} Here we sketch cumulant-based approaches and one concrete algorithm for NGCA. The algorithm was suggested by an anonymous reviewer of a previous version of this paper; while the algorithm was attributed to the ICA literature we are not aware of any specific reference. The notion of cumulants, which is closely related to moments, has found wide use in independent component analysis; see \cite{comon, HyvarinenBook, ComonJutten}. Cumulant-based ICA algorithms are also used to find the Gaussian subspace in the ICA problem, however, these algorithms tend to use 
	the fact that the non-Gaussian components satisfy the ICA structure, and thus are not directly applicable to NGCA. The algorithm itself
	is arguably more complex than ours though its analysis might be simpler and lead to similar guarantees. 
    We provide a sketch of the algorithm in \autoref{sec:Cumulant}. 
	
	As it turns out, the cumulant generating function is the convex conjugate of the relative entropy function and in this sense the two approaches are naturally related by duality. 
	This duality is well known and has led to many important theorems in probability such as the Donsker--Varadhan theorem and in statistical mechanics where these functionals naturally can be interpreted as physical quantities such as free energy and entropy. 
	    \begin{fact}
		Let $ P $ be a probability measure on the real line. Then, 
		\begin{equation*}
		S(P) = \sup_{f}  \mathbb{E}_{P} f - \log \mathbb{E} e^{f}.
		\end{equation*}
		where the supremum is over all functions with $ \mathbb{E} e^{f} < \infty  $. 
	\end{fact}

	\subsection{Our Contribution}
	We provide a simple algorithm with computational and sample complexity polynomial in the dimension $n$ of the problem, improving over existing work which requires at least exponential time when no assumptions on the dimension of the non-Gaussian subspace are made.   
	The algorithm is based on projection pursuit and uses differential entropy of the marginal as the contrast function. 

	Our algorithm is based on the well-known fact that among the probability distributions on $\br$ with zero mean and unit variance, the standard Gaussian distribution is the unique distribution with maximum differential entropy (see, e.g., \cite{CoverThomas}). 
	Thus, the directions corresponding to the Gaussian subspace are exactly those along which the marginals have the maximum differential entropy.	
    More precisely,	Gaussian directions are $u \in \br^n$ such that $\norm{u}_2=1$ and $h(\ip{u}{X}) = h(g)$ where $g$ is the density of the standard Gaussian and $h(\cdot)$ is the differential entropy.
    Our algorithm will use simple gradient descent on $f(u) := h(\ip{u}{X})$ (actually on the more technically convenient relative entropy, which needs to be minimized instead of maximized; but in this section we use differential entropy) to find a Gaussian direction; in particular, unlike \cite{vempala2011structure}, it does not need to use second order optimization methods.
    Our algorithm uses projected gradient descent with $u$ being restricted to the unit sphere. 
	Having found one Gaussian direction, the algorithm projects the data to subspace 
	orthogonal to this direction. We show that, provided that we have a good enough approximation of the Gaussian direction, the data in the subspace also satisfies the NGCA model with Gaussian dimension reduced by one.
    Thus, solving the problem inductively provides us with	all Gaussian directions. 
    Finally, the subspace orthogonal to these Gaussian directions is the desired non-Gaussian subspace. 
		
	As mentioned above, our algorithm can be regarded as an instance of the general method of projection pursuit. 
    The idea of using differential entropy as a contrast function for data analysis, signal processing, and in particular for ICA, is not new and has appeared in many works, e.g., \cite{huber1985, comon, HyvarinenBook, FaivishevskyG08, ComonJutten, Touboul10, Touboul11, BASSEVILLE2013621}.
    However, we are not aware of any algorithms	of this type being applied to NGCA specifically along with theoretical analysis.
    Moreover, to the best of our knowledge, none of these works provide rigorous analysis of these entropy-based algorithms for their specific problems (for example, ICA). 
    Our analysis makes crucial use of some facts about entropy such as the logarithmic Sobolev inequality
	and the characterization of the derivative of entropy under the action of the Ornstein--Uhlenbeck semigroup. 
    These facts, while well-known in some fields, do not seem to have been used in a data analysis context before. 
	We also show that when we have a subgaussian random variable perturbed by a small Gaussian noise, the Lipschitz constant the gradient of the entropy is polynomial in the inverse size of the perturbation and the subgaussian parameter of the random variable, thus facilitating first order optimization techniques to be used.
    The techniques presented here might be of independent interest due to the wide application of entropy as distinguishing function in many applications in statistics and machine learning.
     

    While in the outline above, we said that maximizing the entropy gives us $u$ such that $\ip{u}{X}$ is Gaussian, the algorithm is only able to guarantee that $u$ is close to the Gaussian subspace $\Gamma$.
    Thus when we work in the subspace orthogonal to $u$, the new residual problem need not, a priori, obey the NGCA model.
    We show that in fact it does obey NGCA model. 
    To prove this we make crucial use of the fact that the algorithm finds Gaussian directions as opposed to non-Gaussian components. 
    Furthermore, we show that the errors do not accumulate too rapidly in our iterative algorithm. 
    The rate of error accumulation is one of the primary reasons for the exponential running times in \cite{vempala2011structure} and \cite{tan2017polynomial}. 
	
	We assume that the non-Gaussian components are all different from Gaussian in the following sense: there is a positive integer $r$
	(thought of as a small constant) such that for some positive $i \leq r$ the $i$'th moment of 
	every marginal is sufficiently different from that of the standard Gaussian. (We note that the proofs in \cite{vempala2011structure, tan2017polynomial} seem to assume that the $ r $'th moment of the random variable differs from the Gaussian while all lower moments match).
    We also assume that the non-Gaussian component has good tail behavior, specifically, we assume that the random variable is subgaussian. 
     The technique of Gaussian damping \cite{AndersonGNR15} (introduced in the context of independent component analysis for heavy-tailed data) can likely be used for NGCA
	 to reduce the general case (with mild moment assumptions) to the case when the data has subgaussian tails though we will not pursue this direction here. We remark that while Gaussian damping resembles previously mentioned Reweighted PCA \cite{brubaker2008isotropic, tan2017polynomial}  in the sense that in both cases data is reweighted using Gaussian weights, the two techniques seem to be unrelated and are used for different purposes. 
     
     \subsection{Organization of the Paper}
     In \hyperref[sec:preliminaries]{Section \ref*{sec:preliminaries}}, we provide preliminary definitions and theorems required for the rest of the paper. 
     In \hyperref[sec:Algorithms]{Section \ref*{sec:Algorithms}}, we present our algorithm for the NGCA problem. 
     \hyperref[sec:Proof_Sketch]{Section \ref*{sec:Proof_Sketch}} has a high level sketch of the analysis of our algorithm going over the key ideas informally. 
     In \hyperref[subsec:An_Des]{Section \ref*{subsec:An_Des}}, we shall provide guarantees for the gradient descent to find an approximate critical point of the entropy.
     In \hyperref[subsec:Rel_Ent]{Section \ref*{subsec:Rel_Ent}}, we show lower bounds on the rate of the decay of entropy along the Ornstein--Uhlenbeck process.
     In \hyperref[sec:SEnt_to_SEuc]{Section \ref*{sec:SEnt_to_SEuc}}, we show that finding a direction that has small gradient of entropy also gives us a direction that is close to the Gaussian direction. 
     In  \hyperref[subsec:Est_Grad]{Section \ref*{subsec:Est_Grad}}, we show that we can estimate the gradient of the entropy to the desired accuracy using polynomially many samples of  the random variables. 
     Next, in \hyperref[subsec:Lip_Grad]{Section \ref*{subsec:Lip_Grad}}, we bound the Lipschitz constant of the gradient of the entropy. 
     In \hyperref[subsec:mult_run]{Section \ref*{subsec:mult_run}}, we show how the errors add up in multiple iterations.
      Finally, in \hyperref[sec:Put_It_Together]{Section \ref*{sec:Put_It_Together}}, we give bounds on the total running time of the algorithm in terms of the number of gradient steps and number of samples required. 

	\section{Preliminaries} \label{sec:preliminaries}
	    We will work with the vector space $ \br^n $ with the standard Euclidean inner product structure. 
    We denote by $ \norm{.} $ the standard Euclidean $ 2 $-norm, that is for $x \in \br^n$, $ \norm{x} = \sqrt{\ip{x}{x}} = \sqrt{\sum_{i=1}^{n} x_i^2 } $. 
	Given any subspace $ V \subseteq \br^n$, we talk about the orthogonal projection $ \P_V $ onto $ V $.
    It is the unique idempotent (i.e., $ \P_V^2 = \P_V $), self adjoint (i.e., $ \P_V = \P_V^T $) operator with range $ V $.
    Note that $ \P_V $ restricts to the identity on $ V $ and to the zero operator on $ V^{\perp} $.


\begin{defi}[e.g., \cite{CoverThomas}]
		Let $ W $ be a real-valued random variable with mean zero and variance one with density $ f_{W} $ with respect to the Lebesgue measure. Define relative entropy with respect to the standard Gaussian as
		\begin{equation*}
		S(W) := \int_{\br} f_{W}\left( x\right) \log \left( \frac{f_{W}(x)}{g(x)}  \right)    dx,
		\end{equation*}
		relative Fisher information as
		\begin{equation*}
		J(W) := \int_{\br } \left| \frac{\partial }{\partial x}  \log\left( \frac{f_{W}(x)}{g(x)} \right)     \right|^2 f_{W}(x) dx,
		\end{equation*}
        and the (Shannon) differential entropy as 
        \begin{equation*}
            h(W) := - \int_{\br}   f_{W}\left( x\right) \log  f_{W}\left( x\right) dx,
        \end{equation*}
		 whenever the integrals above exist. Here,  $ g(x) = \frac{1 }{\sqrt{2 \pi}} e^{-x^2 / 2 } $  is the standard Gaussian density. All the logarithms in this paper are to the base $ e $.
	\end{defi}
    
    A central object of study will be the standard multivariate Gaussian distribution. It is a natural generalization of the standard Gaussian to $ n $ dimensions and is given by the following density
    \begin{equation*}
        g_n(x_1 \dots x_n) = \frac{1}{\sqrt{ \left( 2 \pi   \right)^n  }} e^{ - \frac{1}{2} \norm{x}^2  }.
    \end{equation*}
    We denote the standard Gaussian distribution in $\br^n$ by $\mathcal{N}(0,I)$, where $I$ is the $n\times n$ identity matrix and its dimension, as well as that of $0$, are understood from the context.
    The projection $\P_V X$ of the standard multivariate Gaussian random variable $X \in \br^n$ to any subspace $V \subset \br^n$ is
    standard multivariate Gaussian in $V$. 
    Thus, projection to a one-dimensional vector space gives the univariate standard Gaussian.
    Moreover, the projections of the standard multivariate Gaussian along orthogonal subspaces are independent: if $V_1, V_2 \subset \br^n$ are two mutually orthogonal subspaces then $\P_{V_1}X$ and $\P_{V_2} X$ are independent Gaussian random variables. 
    
	Let $ M_r(f) $ denote the integral $ \int x^r f(x) dx $. When $f(\cdot)$ is the density of a real-valued random variable $X$, overloading notation we also use $M_r(X)$ to refer to the previous integral and in general to refer to $ \mathbb{E} \left(X^r\right) $ if the density doesn't exist.
     Recall that $M_r(X)$ is known as the $r$'th moment of $X$. 
	 For standard univariate Gaussian random variable $Z$ and odd positive integer $r$, we have $M_r(Z) = 0$. For even positive integers $r$, 
	we have 
	\begin{align} \label{eqn:gaussian_moments}
	M_r(Z) = \left(r-1 \right)!! = 1\cdot 3 \cdot \ldots \cdot (r-1) < (r-1)^{(r-1)/2}.
	\end{align}
    
    We require that the random variables we work with satisfy certain regularity and tail decay properties. 
    A reasonable class of random variables for which we have good control on the tail decay properties is the class of subgaussian random variables defined below.
    Roughly speaking, the tails of these random variables decay like the tails of a Gaussian distribution.
    
    \begin{defi}[see \cite{vershynin2016high}]
        For a positive constant $K$, a random variable  $ X $ on $ \br $ is said to be $ K $-subgaussian if for all $ t > 0 $, we have
        \begin{equation*}
        \Pr\left[ |X| \geq t   \right] \leq 2 e^{ - \frac{t^2}{K^2} }. 
        \end{equation*}
        A random variable $ Y $ on $ \br^n $ is said to be $ K $-subgaussian if for all $ x \in \mathbb{S}^{n-1} $, the one dimensional random variable $ \ip{Y}{x} $ is $ K $-subgaussian. Here $\mathbb{S}^{n-1}$ is the $(n-1)$-dimensional unit sphere in  $\br^n$. 
    \end{defi}
	
    Note that from the above definition it follows that the projections of subgaussian random variables onto subspaces are also subgaussian with the same parameter.
    Also, the scaled sum (preserving variance) of independent subgaussian random variables is subgaussian.
    We note this in the following fact. 
    
    \begin{fact}\label{subGauss}
    	Let $ X $ and $ Y $ be two independent subgaussian random variables on $\br^n$ with subgaussianity parameters $ K_1 $ and $ K_2 $  respectively. 
    	Then, for any subspace $ V $, $ \P_V X $ is $K_1$-subgaussian and $ \alpha_1 X + \alpha_2 Y  $ is subgaussian with parameter $ \sqrt{\alpha_1^2 K_1^2 + \alpha_2^2 K_2^2 } $.  
    \end{fact}
    
    Subgaussian distributions are important partly because they satisfy a version of the Chernoff--Hoeffding inequality. 
    Examples of subgaussian random variables include the Gaussian random variable, bounded support random variables and sums of other subgaussian random variables.

	The Gaussian distribution derives part of its importance from the central limit theorem which has many proofs including information theoretic ones.
	There is considerable literature on information theoretic proofs of the central limit theorem and its refinements; see, e.g., \cite{carlen1991entropy} and \cite{MR2109042}. 
	These utilize the maximum entropy property of the Gaussian distribution, and some of the tools from that literature are relevant to our study. 
	We require the following two well-known properties of relative entropy. 
	These have been influential in analysis and probability, because of their close connection to isoperimetry and concentration and have seen applications from areas ranging from random matrix theory to approximation algorithms. 
	More information along with bibliographic remarks can be found e.g. in \cite{carlen1991entropy} and \cite{MR3155209}.
	
\begin{theorem}[Gaussian Log-Sobolev Inequality; see Proposition 5.5.1 in \cite{MR3155209}]\label{LSI}
	Let $ X $ be a zero mean and unit variance real-valued random variable with $ S(X) < \infty $. Then we have
	\begin{equation*}
	S(X) \leq \frac{1}{2} J(X).
	\end{equation*}
\end{theorem}
	
\begin{theorem}[Derivative of Entropy; see Proposition 5.2.2 in \cite{MR3155209}]\label{derivative_theorem}
	Let $ X $ be a zero mean and unit variance real-valued random variable with $ S(X) < \infty $. Let $ Z $ be an independent standard Gaussian random variable. Then,
	\begin{equation*}
	\frac{d}{dt} S\left(e^{-t} X + \sqrt{1 - e^{-2t}}Z  \right) = - J\left(   e^{-t} X + \sqrt{1 - e^{-2t}}Z\right).
	\end{equation*}
\end{theorem}
%

	A simple corollary of the two theorems above is the following statement about the derivative of the entropy. 
	
	\begin{cor}\label{thm:ent_der}
			Let $ X $ be a zero mean and unit variance real-valued random variable with $ S(X) < \infty $. Let $ Z $ be an independent standard Gaussian random variable. Then,
			\begin{equation*}
				S\left( e^{-t} X + \sqrt{1 - e^{-2t}}Z   \right) \leq \frac{1}{2}\abs{\frac{d}{dt} S\left(e^{-t} X + \sqrt{1 - e^{-2t}}Z  \right)} .
			\end{equation*}
	\end{cor}

    We will now define the non-Gaussian component analysis model which is our main object of study. The model captures random variables that can be split into a Gaussian component and a non-Gaussian component. For a linear subspace $V \subset \br^n$, we denote by $V^{\perp }$ the orthogonal complement of $V$ in $\br^n$ under the standard inner product. Thus, $\br^n$ has the direct sum decomposition $\br^n = V \oplus V^{\perp }$.

	\begin{defi}[Isotropic NGCA] \label{Iso_NGCA}
		We say that a random variable $ X \in \br^n $ follows the isotropic NGCA model if $ X = (Z, \tilde{X}) \in \Gamma \oplus \Gamma^{\perp }$ (that is to say $Z \in \Gamma$ and $\tilde{X} \in \Gamma^\perp$)
		for some linear subspace $\Gamma \subset \br^n$. Here $ Z \sim \mathcal{N}(0,I) $ is the standard Gaussian on $\Gamma$ and $X$ is isotropic: $ \Ex X = 0 $ and $ \Ex X X^T = I_n $.
	\end{defi}
	We will henceforth refer to $ \Gamma $ as the Gaussian subspace.
	In the most interesting applications of the above model the noise is comparable in magnitude to the actual interesting directions, requiring the use of structural assumptions on the noise to separate them out.
	Data can be made isotropic by applying an affine transformation assuming mild regularity conditions (briefly discussed in \hyperref[sec:prev_work]{Section \ref*{sec:prev_work}}) and it is a standard preprocessing step for many data processing algorithms, and is used in the previous work on NGCA such as \cite{blanchard2006search}.
	The isotropy assumption is largely without loss of generality since under mild conditions the covariance matrix can be approximated using polynomial number (actually near linear in the dimension) of samples of the random variable (for example, see \cite{MR2963170}). 
    In the setting where the noise is much smaller compared to the signal then we can indeed use other methods such as PCA to remove the noisy components without having to resort to structural assumptions about the noise. 
	The isotropy assumption also ensures that the Gaussian and the non-Gaussian components are orthogonal to each other. 
     
	 Finally, we define the distance between two subspaces. 
	 The distance is measured as the Frobenius norm between the orthogonal projections onto the two subspaces. 
	 This will be used to measure how close the subspace output by the algorithm is to the true subspace.
	 
	\begin{defi}[Subspace Distance]
		Let $ V,W $ be two $ k $ dimensional subspaces of $ \br^n $ and let $ U_1 $ and $ U_2 $ be matrices whose columns form an orthonormal basis for $ V $ and $ W $ respectively. Define the subspace distance between the subspaces by
		\begin{equation*}\label{key}
			d(V,W) := \norm{U_1U_1^T - U_2U_2^T}_{F},
		\end{equation*}
		where $ \norm{A}_{F} = \sqrt{\sum_{i,j} A_{ij} ^2} $ denotes the Frobenius norm of the matrix $ A $.
	\end{defi}
    
  	The following well-known fact about the subspace distance implies that if the algorithm outputs a good approximation $V$ to the Gaussian subspace, then the true non-Gaussian subspace is close to the orthogonal complement $V^\perp$.
  	The proof follows by noticing that if $ \P_V $ is an orthogonal projection onto a subspace, then $ I - \P_V $ is the orthogonal projection onto its complement. 
  
  \begin{fact}[e.g. \cite{tan2017polynomial}]
  	Given any $ k $-dimensional subspaces $ W $ and $ W' $, we have $ d(W,W') = d(W^{\perp} , W'^{\perp} )$.
  \end{fact}

%
%
%
%
    
	\section{The Algorithm} \label{sec:Algorithms}
	In this section, we describe our algorithm for NGCA.
While in the introduction we mentioned that we use entropy as the contrast function, for technical convenience we will instead use the relative entropy with respect
to the Gaussian distribution.
The two are closely related.
    Note that, from the definitions, it follows that for a random variable with mean zero $ S(W) = - h(W) + \frac{\mathsf{Var}(W)}{2} + \log \sqrt{2 \pi}  $, where $ \mathsf{Var}(W) $ is the variance of $ W $. 
	The motivation for our algorithm is the following property of relative entropy showing that the Gaussian is the unique minimizer of relative entropy with respect to the Gaussian.
	Another way of stating the result is that the Gaussian is the unique maximizer of differential entropy under variance constraints.  

\begin{theorem}[see, e.g., \cite{CoverThomas}]\label{thm:minrelent}
	Let $ Y \in \br $ be a zero mean and unit variance random variable. Then $ S(Y) \geq 0 $ and $ S(Y) = 0 $ if and only if $ Y \sim \mathcal{N}(0,1) $.
\end{theorem}

The above theorem is applicable in our setting because we work with isotropic random variables $X \in \br^n$ which means that all its marginals have zero mean and unit variance: $\Ex \ip{u}{X} = 0$ and $\Ex \ip{u}{X}^2 = 1$ for all unit vectors $u \in \br^n$.
We use projected gradient descent on the unit sphere to search for the Gaussian directions, that is to say $u \in \Gamma$. The above theorem tells us that the Gaussian directions are global minima of the relative entropy functional.
In the algorithms below, we use the following conventions to make the notations and descriptions less cumbersome.
We say that the random variables $X$ is part of the input to mean that the algorithm has access to independent samples of $X$. 
Also, in the full algorithm, we ``update" random variables by applying an orthogonal projection; by this we mean that the projection is applied to all the samples of the random variables in subsequent steps. 

In the following $ \nabla S \left( \ip{X}{u}  \right)$  and $ \nabla_u S \left( \ip{X}{u}  \right)$  mean $ \nabla_y S( \ip{X}{y})  |_{y=u}$, where $ \nabla_y $ denotes the Euclidean gradient.

	\begin{algorithm}[H]
		\caption{Gradient Descent}
		\begin{algorithmic}[1]
            \Require{Dimension $ \ell  $, Isotropic Random Variable $ X \in \br^\ell $, Error Parameters $ \epsilon_1  $, $ \epsilon_2 $, Step Size $ \eta $ }
			\Procedure{GradDes}{$ \ell , X , \epsilon_1 , \epsilon_2,  \eta$}
			\State Pick $ u \leftarrow \mathbb{S}^{\ell -1}  $ randomly.
			\For{$ \tau \leq O(\epsilon_1 ^{-2}) $}
				\State Compute $ \Delta_{\tau} $ approximating $ \nabla_u S \left(  \ip{X}{u} \right)  $ to within error $ 0.2 \epsilon_1 $. 
				  \State  Set $ u \gets  \frac{u - \eta \Delta_{\tau}}{\norm{u - \eta\Delta_{\tau}}} $.
			\EndFor
			\If{$ \nabla_r S\left(\ip{X}{u} \right) \leq \epsilon_1 \text{ and } S\left(\ip{X}{u} \right) \leq \epsilon_2 $}
				\State \Return $ u $.
			\Else 
				\State \Return FAILURE. 
			\EndIf
			\EndProcedure
            \Ensure{Vector $ u $ or FAILURE.}
		\end{algorithmic}
	\end{algorithm}

%

%
	 
	 We then apply this procedure iteratively on the orthogonal complement of the directions we have already found. We stop the algorithm if sufficient progress is not made in a specified number of steps.
	 We will show later that with high probability the algorithm will indeed find all the Gaussian directions and we will terminate in polynomial time with the choice of parameters stated at the beginning of \hyperref[sec:Put_It_Together]{Section \ref*{sec:Put_It_Together}}). 
	\begin{algorithm}[H]
		\caption{Full Algorithm}
		\begin{algorithmic}[1]
            \Require{Dimension $n$, Isotropic Random Variable $ X \in \br^n$, Error Parameters $ \epsilon_1 $, $ \epsilon_2 $, Step size $\eta$, Noise Parameter $ t' $ }
			\Procedure{FullAlg}{$ X, \epsilon_1 $ , $ \epsilon_2 $, $ \eta $}
            \State $ X \leftarrow \sqrt{1-t'^2}X + t'Z $, where $ Z \in \br^n$ is an independent standard Gaussian random variable. 
			\State Set $ V \gets \emptyset $. 
			\For{$ 0 \leq j \leq n $}
				\State Set $ \lambda \gets $ \Call{GradDes}{$ n - j, X ,\epsilon_1 , \epsilon_2, \eta $}. 
				\If{$ \lambda \neq $ FAILURE. }
					\State Set $ X \gets P_{\lambda^{\perp}} \left( X \right) $. 
					\State Set $ V \gets V \cup \left\{  \lambda \right\} $. 
				\Else 
					\State \Return $ V^{\perp } $. 
				\EndIf 
			\EndFor
			\EndProcedure
            \Ensure{Subspace $ V^{\perp} $}
		\end{algorithmic}
	\end{algorithm}

	
	\section{Proof Sketch} \label{sec:Proof_Sketch}
		In this section we provide an informal sketch of the main ideas behind the proof that our NGCA algorithm works correctly and efficiently. We continue our discussion from the algorithm description in the previous section.
The following is an informal version of \autoref{Main_Theorem}, our main theorem. 
	\begin{theorem}[informal statement of \autoref{Main_Theorem}]
		Given access to samples from a distribution satisfying the isotropic NGCA model with the non-Gaussian part being subgaussian, Full Algorithm above estimates the non-Gaussian subspace within subspace distance $\epsilon>0$ and runs in polynomial time in dimension $ n $  and $1/\epsilon $. 
	\end{theorem}

	The degree of the polynomial depends on $r$ where $r$ is a small positive integer with the property that for each marginal of the non-Gaussian random variable $\tilde{X}$, at least one of its first $r$ moments differs from the corresponding moments of the standard Gaussian by at least $D$ (a parameter of the problem). 
	
	As previously mentioned, the main idea of the algorithm is that in one dimension the standard Gaussian minimizes relative entropy. This is true in a robust sense: directions that approximately minimize relative entropy must be close to the Gaussian component as we will see shortly in some more detail. 
	 Recall that the optimization problem solved by the
	projected gradient descent algorithm is the following,
	\begin{align*}
	\min_{u \in \mathbb{S}^{n-1}} S(\ip{u}{X}).
	\end{align*}
	 This problem is non-convex not only because the domain is non-convex but also because the function being minimized is not convex in any reasonable sense: for example, when $u$ is in the non-Gaussian subspace $\Gamma^\perp$, \emph{a priori} the function can behave in complicated ways. However, if we are at a point $u_0$ on the sphere such that  projection onto $\Gamma$ is not too small, then there is a path
	 to $\Gamma$ that reduces relative entropy to $0$ monotonically. For instance, the geodesic path from $u_0$ to its closest point in $\Gamma$ has this property. Our descent algorithm, however, does not follow this path. It is possible that $u_0$ is such that the gradient at $u_0$ has a large component in the non-Gaussian subspace $\Gamma^\perp$ even though it has a non-zero component towards the Gaussian subspace $\Gamma$ as well (whose magnitude depends on the projection of $u_0$ onto $\Gamma$).
	 If such a case occurs, then the descent algorithm might stay away from the Gaussian component for a number of steps. If this happens, eventually it will be close to a local minimum in the non-Gaussian component and the only way to make further progress would be to move
	 towards the Gaussian component thus giving us our desired direction close to $\Gamma$.
	 
	The gradient descent algorithm gives us a point on the sphere where the gradient of the relative entropy has small norm.  This only tells us that we are at a local optimum. But using the \autoref{thm:ent_der} of the well known Logarithmic Sobolev Inequality (\autoref{LSI}) in \hyperref[optimize]{Theorem \ref*{optimize}}, we show that the entropy at that point must also be small provided that we began at a point with sufficient Gaussian projection. 
	This actually shows that the critical points for the entropy must lie either entirely in the Gaussian space or entirely in the non-Gaussian subspace. Informally, 
	
	\begin{theorem}[informal statement of \autoref{optimize}]
		Suppose that the initial point $u_0$ for the gradient descent algorithm does not have too small a projection onto the Gaussian subspace,
		and the gradient descent algorithm reaches a point $u$ that has small gradient of relative entropy $\norm{\nabla_u S(\ip{u}{X})}$. Then $u$ must also have low entropy, that is, $S(u)$ must be small. 
	\end{theorem}
	
	We also need to prove that given a point of low entropy, it is indeed close to the Gaussian direction in angle. 
	We show that small entropy implies that the moment distance of $\ip{u}{X}$ to Gaussian is small and that in turn implies that
	$u$ must be close to the Gaussian subspace in angle. 
	We note that a difficulty arises due to our weak assumption that there exists a moment among the first $ r $ that is far from the Gaussian rather than the simpler assumption of having the first non zero moment gap between the Gaussian and the random variable of interest being large. With the stronger assumption, the same claim follows by using a robust version of the Cramer--Rao bound. Informally stated, we have 
	
	\begin{theorem}[informal statement of \autoref{entropy_lower_bound}]
		Given any direction such that the entropy is small, the angle between the direction and the Gaussian subspace is small. 
	\end{theorem}
	
		Another property of relative entropy that aids us is that the projection onto the sphere only reduces the relative entropy (see the proof of \autoref{grad_des}).
		This is because, when we move along the radial direction, due to the bilinearity of the inner product, we are scaling the random variable.
		This allows us to use the Euclidean gradient on the entropy and in each gradient descent step and project back on to the sphere with the guarantee that this projection reduces the objective function. 
	
	In order for the gradient descent algorithm to work, we need to ensure that the gradient of the entropy satisfies a certain smoothness condition, the Lipschitz continuity. 
	Lipschitz continuity measures how quickly the value of the function changes with the change in the arguments. 
	This smoothness condition on the gradient is a standard requirement needed for many first order methods. 
	We show that for random variables of interest the Lipschitz constant of the entropy is indeed bounded by polynomials in the relevant parameters. 
    We first reduce the $ n $-dimensional problem to a two dimensional problem. 
    We then bound the derivatives of the density.
	In order to do this, we consider the behavior of the density function under the action of the Ornstein--Uhlenbeck process. 
    We provide upper and lower bounds on the density and its derivatives. 
    We then translate these bounds to bounds on the derivatives of the entropy. 
    
    
	\begin{theorem}[informal statement of \autoref{Lip_Bound1}]
		For any random variable of the form $ Y_t = \sqrt{1-t^2}X + tZ $, where $ X \in \br^n$ is $ K $-subgaussian with zero mean and unit variance and $ Z $ is an independent standard Gaussian random variable, we have that $ \nabla S \left( \ip{Y_t}{a}  \right) $ is Lipschitz continuous with Lipschitz parameter polynomial in $ K $ and $ t^{-1} $. 
	\end{theorem} 
	
	Once we find a direction, we work in the subspace orthogonal to the first vector and repeat the gradient descent algorithm. 
	The next part of the analysis is about the build up of errors as we run the algorithm for multiple iterations. 
	The following technical lemma is useful for showing that the errors do not accumulate too fast. 
	\begin{theorem}[informal statement of \autoref{GS}]
		Let $ \lambda_1 \dots \lambda_k \in \br^n$ be orthonormal, and let $ \gamma_1 \dots \gamma_k \in \br^n$ be such that  $ \gamma_i$ is close to $ \lambda_i$. Then the subspace spanned by the $ \gamma $'s is close to the subspace spanned by the $ \lambda $'s. 
	\end{theorem}
    The fact that the $\lambda_i$'s are orthonormal is important for the above theorem. 
    Note that the above theorem does not assume that $ \gamma_i $ are orthonormal.
    The above claim is easier to prove but in our application that condition cannot be guaranteed, that is, in each iteration we are not necessarily approximating orthogonal Gaussian directions (this occurs due to the error in the estimating the Gaussian direction).
    But, we prove that, if the errors are sufficiently small, we approximate linearly independent Gaussian directions and thus, increase the dimension of the estimate of the Gaussian subspace in each iteration. 
	
	A key observation in the error analysis is that when we project onto the orthogonal subspace, the random variable satisfies the same NGCA model with slightly worse parameters in the sense that the moment distance of the non-Gaussian marginals to the standard Gaussian in the new problem is smaller compared to the original problem.
    Due to this, the algorithm has a higher running time in later iterations. But, using the theorem just mentioned we can bound the parameters of the projected random variables: 
	
	\begin{theorem}[informal statement of \autoref{alt_mom}]
		The projected random variable has the moment distance from the standard Gaussian that is only slightly smaller than the original moment distance. 
	\end{theorem}
		We note that moment bounds on the random variable also give us a bound on the entropy. That is, given a random variable has a large moment difference with Gaussian, we can say that it has high entropy.
        This can be used to check when there are no Gaussian directions left to be found.
	
	What remains to be done is to estimate the relative entropy of the random variable from samples. We achieve this using a histogram based estimator.
     The existing literature does not seem to be applicable for our precise estimation problem. 
	Since we are working with one dimensional random variables with sufficient Gaussian noise added and good tail behavior, we can perform this estimation efficiently. 
    Adding the Gaussian noise smoothens the density of random variable which is helpful in estimating its relative entropy.

		\section{Analysis of the Descent Algorithm} \label{subsec:An_Des}
			As the first part of our analysis, we show that given noisy estimates of the gradient of the entropy, we can arrive at a point such that the gradient of the entropy is small.
	Specifically, we assume that the gradient estimate at step $ i $ is noisy with error $ \epsilon(i) $. 
	We need this additional noise term to account for the error introduced by the empirical estimation of the entropy. 
	The approach taken in the analysis is fairly standard in the optimization literature, for example see \cite{nesterov2013introductory}, but we include the proof here for the sake of completeness and because we are not aware of a theorem statement that can be invoked in a black-box manner
	for our application.
	Recall our notation  $ S(\gamma(t)) = S(\ip{X}{\gamma(t)})$ and $ J(\gamma(t)) = J(\ip{X}{\gamma(t)})$.
	Further, for $r \in \br^n$, we use $ \nabla S \left(r\right) $ to denote $ \nabla_y S( \ip{X}{y})  |_{y=r}$, where $ \nabla_y $ denotes the Euclidean gradient.
	
	\begin{defi}
		Let $ \left( X , d_X\right) $ and $ \left( Y, d_Y  \right)  $ be any two metric spaces. A function $ f : X \to Y $ is said to be $ L $ Lipschitz if for all $ x_1 , x_2 \in X $
		\begin{equation*}
			d_Y\left( f\left( x_1  \right) , f\left( x_2\right)  \right) \leq L\cdot d_X\left( x_1 , x_2  \right).
		\end{equation*}
		The best possible such constant is known as the Lipschitz constant of the function $ f $. 
	\end{defi}
	
    In the following theorem, we show that among positive scalings of a random variable, the one with unit variance has the least relative entropy with respect to the standard Gaussian. 
    We will use this claim to show that in our optimization algorithm, projecting back onto the sphere only reduces the objective function. 
    The claim also is useful to control the derivative of the entropy in the radial direction, allowing us to move from the derivatives on the sphere to the derivatives in the Euclidean space. 
    
	\begin{lem} \label{lem:ent_scale}
		Let $ W $ be a unit variance random variable with finite relative entropy with respect to the Gaussian i.e. $ S(W) < \infty $. Then, for all $ \lambda > 0 $, we have 
		\begin{equation*}
			S(W) \leq S(  \lambda W). 
		\end{equation*}
	\end{lem}
	\begin{proof}
		Let $ f $ be the density of $ W $ and let $ f_{\lambda} $ be the density of $ \lambda W $. 
		From the change of variables formula it follows that $ f_{\lambda}(x) = \frac{1}{\lambda} f\left( \frac{x}{\lambda}  \right)   $. 
		Thus, from the definition of relative entropy, we get 
		\begin{align*}
			S(\lambda W) =& \int_{\br} f_{\lambda} (x)  \log \left( \frac{f_{\lambda} (x)}{g(x)}  \right) dx \\
			=& \int_{\br} \frac{1}{\lambda} f\left( \frac{x}{\lambda}  \right) \log \left( \frac{f \left( \frac{x}{\lambda}  \right)   }{\lambda g(x)}  \right) dx. \\
			\intertext{Setting $ \lambda y = x $, we get }
			S(\lambda W) =& \int_{\br} f(y) \log\left(  \frac{f(y)}{\lambda g(\lambda y)}  \right) dy \\
			=& \int_{\br} f(y) \log\left(  \frac{f(y)}{ g(\lambda y)}  \right) dy - \log\lambda . 
			\intertext{Note that $ \log g(\lambda x)  = -\frac{\lambda^2x^2}{2} - \log ( \sqrt{2 \pi} )  $.}
            S\left( \lambda W   \right) = & \int_{\br}   \left(f(y) \log f(y) + \frac{\lambda^2x^2}{2} f(y)   + f(y)\log ( \sqrt{2 \pi} )  \right) dy - \log\lambda. \\
            \intertext{ But, variance of the random variable is one. Thus, we get}
            S\left( \lambda W   \right) = & \int_{\br}   \left(f(y) \log f(y)  + f(y)\log ( \sqrt{2 \pi} )  \right) dy - \log\lambda  + \frac{\lambda^2 }{2}   \\
             = & \int_{\br}   \left(f(y) \log f(y) + \frac{x^2}{2} f(y)  + f(y)\log ( \sqrt{2 \pi} )  \right) dy - \log\lambda  + \frac{\lambda^2 }{2} - \frac{1}{2}  \\
		     = & \, S(W) - \log \lambda + \frac{\lambda^2 - 1 }{2} .
		\end{align*} 
		Applying the inequality $ \log x \leq x - 1 $ to $ \log \lambda^2 $, we get $ 2 \log \lambda \leq \lambda^2 - 1 $, which gives us that the additive term in the above expression is positive, thus giving us the required result. 
	\end{proof}
	
	\begin{theorem}\label{grad_des}
		Assume that $ \nabla S $ is $L$-Lipschitz. Let the gradient descent algorithm be given access to an oracle that during iteration $ i  $ on input $ x $ returns $ \nabla S\left(x \right) + \epsilon(i)  $, where $ \norm{\epsilon(i)} \leq 0.2 \norm{\nabla S(x)} $. Then, the gradient descent with step size $\eta < 1/(3L)$ takes 
		\begin{equation*}
			25 L \eta^{-1}  \epsilon^{-2} 
		\end{equation*}
		 steps to find a point $ y $ such that $ \norm{ \nabla S \left(y\right) }  \leq \epsilon  $.
	\end{theorem}

	\begin{proof}
		 Let $ x_{t+1} = r_t - \eta \left( \nabla S(r_t) + \epsilon(t)  \right)  $ and let $ r_{t+1} = x_{t+1} / \norm{x_{t+1} }  $. Using the multidimensional fundamental theorem of calculus, we have
		\begin{align*}
		 S(r_t) - S(x_{t+1})  &= \int_0^1 \left\langle \nabla S ( x_{t+1} +  (r_t - x_{t+1}) \lambda   ) , r_t - x_{t+1} \right\rangle d\lambda   \\
		& =  \int_0^1 \left\langle \nabla S(r_t) , r_t - x_{t+1} \right\rangle d\lambda  + \int_0^1 \left\langle \nabla S ( x_{t+1} +  (r_t - x_{t+1}) \lambda   ) - \nabla S(r_t) , r_t - x_{t+1} \right\rangle d\lambda   \\
		& = \eta \left\langle \nabla S(r_t) , \nabla S(r_t) + \epsilon(t) \right\rangle +  \int_0^1 \left\langle \nabla S ( x_{t+1} +  (r_t - x_{t+1}) \lambda   ) - \nabla S(r_t) , r_t - x_{t+1} \right\rangle d\lambda \\
		& =   \eta   \norm{\nabla S(r_t)}^2  + \eta \left\langle \nabla S(r_t) , \epsilon(t) \right\rangle  +  \int_0^1 \left\langle \nabla S ( x_{t+1} +  (r_t - x_{t+1}) \lambda   ) - \nabla S(r_t) , r_t - x_{t+1} \right\rangle d\lambda   \\
		& \geq \eta \norm{\nabla S(r_t)}^2 + \eta \left\langle \nabla S(r_t) , \epsilon(t) \right\rangle   - 
		\left|  \int_0^1 \left\langle \nabla S ( x_{t+1} +  (r_t - x_{t+1}) \lambda   ) - \nabla S(r_t) , r_t - x_{t+1} \right\rangle d\lambda    \right|.  
		\end{align*}
		Using the Cauchy--Schwarz inequality, 
		\begin{align*}
		S(r_t) - S(x_{t+1}) 	& \geq   \eta \norm{\nabla S(r_t)}^2 + \eta \left\langle \nabla S(r_t) , \epsilon(t) \right\rangle  -  \int_0^1 \norm{\nabla S ( x_{t+1} +  (r_t - x_{t+1}) \lambda   ) - \nabla S(r_t)} \cdot \norm{r_t - x_{t+1}} d\lambda. 
		\end{align*}
		Using the Lipschitz continuity of the gradient of $ S $, we have
		\begin{align*}
		S(r_t) - S(x_{t+1}) & \geq \eta \norm{\nabla S(r_t)}^2 +  \eta \left\langle \nabla S(r_t) , \epsilon(t) \right\rangle - \int_0^1 (1-\lambda) L \norm{  r_t - x_{t+1} } ^2 d\lambda \\
		& \geq \eta \norm{\nabla S(r_t)}^2 +  \eta \left\langle \nabla S(r_t) , \epsilon(t) \right\rangle - \frac{L}{2} 
		\norm{r_t - x_{t+1}}^2 \\
		& \geq  \eta \norm{\nabla S(r_t)  }^2 - \eta \norm{\nabla S (r_t)} \norm{\epsilon(t)} - \frac{L\eta^2}{2} \left[ \norm{\nabla S (r_t)}^2 + \norm{\epsilon(t)}^2 + 2 \left\langle \nabla S , \epsilon(t) \right\rangle   \right] \\
		& \geq \frac{2 \eta - L\eta^2}{2} \norm{\nabla S(r_t)}^2 - \left( \eta + L\eta^2   \right) \norm{\nabla S (r_t)} \norm{\epsilon(t)} - \frac{L\eta^2}{2} \norm{\epsilon(t)}^2.
		\end{align*}
		
		Setting $ \norm{\epsilon(t)}  \leq \delta \norm{\nabla S}$, we get 
		\begin{align*}
		S(r_t) - S(x_{t+1})& \geq \left( \frac{2 \eta - L \eta^2 - 2 \eta \delta - 2L \delta \eta^2 - L \eta^2 \delta^2}{2}   \right) \norm{\nabla S (r_t)} ^2 \\
        & \geq \left( \frac{\eta \left( 2 - L\eta -2 \delta - 2L\delta \eta - L\eta\delta^2  \right)  }{2}  \right) \norm{\nabla S (r_t)}^2 \\
        & \geq  \left(\frac{\left( 2(1-\delta) - (1+\delta)^2 L\eta  \right) \eta }{2} \right) \norm{\nabla S (r_t)}^2\\
        & = F(\delta, L, \eta) \cdot \norm{\nabla S (r_t)}^2.
		\end{align*}
		Where $F(\delta, L, \eta)$ is defined in the appropriate way. Note that the above expression is positive, whenever $ 0 < \eta < \frac{2(1-\delta)}{L(1+ \delta^2)} $. 
		Note that 
		\begin{align}\label{eqn:rxent}
		 S(r_{t+1} ) \leq S(x_{t+1})
		\end{align}
		follows from \autoref{lem:ent_scale}.
		Thus,
		\begin{align*} 
		S(r_t) - S(r_{t+1}) \geq F(\delta, L, \eta) \cdot \norm{\nabla S (r_t)}^2.
		\end{align*}
	Summing over $t$, we get
	\begin{align*}
		S(r_0) - S(r_N) & \geq F(\delta, L, \eta) \cdot  \sum_{i = 1}^{N} \norm{ \nabla S(r_i)} ^2 \\
		& \geq  N\cdot F(\delta, L, \eta) \cdot \min_i \norm{\nabla S(r_i)}^2. 
		\intertext{Thus, we get }
		\min_i \norm{   \nabla S(r_i) } &\leq \sqrt{\frac{S(r_0)}{N \cdot F(\delta, L, \eta)}} \, .
		\intertext{Setting $ \epsilon = \sqrt{\frac{S(r_0)}{N \cdot F(\delta, L, \eta)}}$, we get }
		N &= \frac{S(r_0)}{\epsilon^{2} \cdot F(\delta, L, \eta)}.  
		\end{align*}
		Setting $\delta = 0.2$, we get 
		\begin{align*}
		N & =  \frac{1}{\epsilon^2}\cdot\frac{2}{(1.6-3.24 L \eta) \eta}\cdot S(r_0). 
		\end{align*}
		Thus, for $\eta < 1/(3L)$ we have 
		\begin{align*}
			N & =  \frac{20}{8\eta \epsilon^2}\cdot S(r_0). 
		\end{align*}
        Note that since, we have $ \abs{S(r_0)} \leq \pi \sup_{\mathbb{S}^{n-1} } \norm{ \nabla S} \leq \pi C_0$ and $ \C_{0} \leq \pi L $, we get 
        \begin{equation*}
        	N \leq \frac{25 L}{ \eta  \epsilon^2},
        \end{equation*} 
        as required
	\end{proof}

	From the above theorem, we can guarantee that number of iterations required to reach a point with small norm of the gradient is polynomial in the inverse of the error parameter $ \epsilon $ and Lipschitz constant $ L $. 
	In a later section, we shall bound $ L $ to be polynomial in the quantities of interest and see that we need to instantiate the algorithm with $ \epsilon $ to be inverse polynomial in the problem parameters, thus leading to a polynomial time algorithm. 
		
        \section{Controlling the Rate of Entropy Decay} \label{subsec:Rel_Ent}
        	From the previous section, we know how to find directions such that the gradient of the entropy is small. 
	We need to show that this indeed gives us a direction close to the Gaussian direction.
    We would like to claim that if the gradient of the relative entropy is small, then the projection along the Gaussian space is high. 
    Towards this, we consider the problem in two dimensions, that is the case of independent Gaussian noise being added to the random variable.
    Given $W_t =  e^{-t}X + \sqrt{1- e^{-2t}}Z $, we would like to control the rate at which $ S(W_t) $ goes to zero as $ t  $ goes to infinity. 
    If the entropy decayed too rapidly in this process, even if we find a direction with low entropy we could still be far away from Gaussian. 
    We prove a robust version of the \hyperref[thm:minrelent]{Theorem \ref*{thm:minrelent}} for random variables for random variables from this process to show that this is not the case.
    Specifically, using the condition on the moments, we show that entropy cannot become small without sufficient Gaussian noise being added. 
    In \hyperref[sec:SEnt_to_SEuc]{Section \ref*{sec:SEnt_to_SEuc}}, we reduce the high dimensional case to the two dimensional case by taking an appropriate path from the current direction to the Gaussian direction.
    First, we note a few preliminary definitions. 
    
    \begin{defi}
        Let $ \left( X , \mathcal{F} \right)  $ be a measurable space. For two probability measures $ \mu $ and $ \nu  $ on this space, we define the total variation distance $ d_{TV}( \mu , \nu   )  $ between them to be 
        \begin{equation*}
            d_{TV} \left( \mu , \nu  \right) = \sup_{ A \in \mathcal{F} } \abs{ \mu(A) - \nu(A) }.
        \end{equation*}
        For any two random variables, we define the total variation distance to be total variation distance of the induced measures. 
    \end{defi}
    
    It is a well-known fact that  for any random variable with densities $ f_1 $ and $ f_2 $, the total variation distance is given by $ \frac{1}{2} \int \abs{f_1 - f_2} dx    $. 
    
    We consider the following process: $ W_{t} = e^{-t}Y + \sqrt{1- e^{-2t}} Z  $.
    This represents adding Gaussian noise to a fixed random variable. 
    We chose this particular parametrization because it gives particularly nice expressions for the derivative. 
    We bound the rate of evolution of moments under this process towards proving that the fact that when we are at a point with low entropy, we are indeed close to a Gaussian direction. 
    
    \begin{rem}
        Note that below we assume that at least one of the first $ r $ moments differs from the Gaussian moments. 
        If we made the stronger assumption that the $r$'th moment differed from the Gaussian $r$'th moment, then the below claim is easier to prove and can be shown by controlling the error term in the Cramer--Rao inequality using the properties of the Hermite polynomials (see Section 2.3 in \cite{MR2109042}).  
        The possibility of having non-zero lower moments adds corruptions to the behavior of the higher moments and requires a more delicate analysis to control. 
    \end{rem}
    
    \todo[inline, color = green]{Will change $ E[Z]  $ to $M(Z)$ for consistency.\\
    Also, considering changing the assumption to the rth moment and moving the stronger claim to the appendix for ease of exposition. }
    
	\begin{theorem} \label{entropy_lower_bound}
	Consider the random variable $ W_{t} = e^{-t}Y + \sqrt{1- e^{-2t}} Z $ where $ Z \sim \mathcal{N}(0,1)$, $t \geq 0$, and $ Y $ is independent of $Z$. Assume that $ Y $ is $ K $-subgaussian and has mean $ 0 $ and variance $ 1 $.  
	Assume that $\abs{ M_{i} \left(Z\right)   - M_{i} \left(Y\right) } \geq D$ for some $ i \leq r $ and $D>0$. Assume also that $ S(W_t) \leq \epsilon $. Then 
	$ e^{-t} \leq  5 A \cdot r! \cdot (\epsilon/D^2)^{1/2r}$ where $A = 4r^2 K \left(3+\log \frac{K}{D}  \right)$.
	\end{theorem}
\begin{proof}
	First note that by Pinsker's Inequality (see \cite{tsybakov2009introduction}) we have 
	\begin{equation} \label{eqn:Pinsker}
	S(W_t) \geq  2\, d_{TV}(W_t , Z )^2,
	\end{equation}
	allowing us to work with $ L^1 $-distance below. Let $ \rho_t $ be the density of $ W_t $. 
	Let $k \leq r$ be such that $ \left| \mathbb{E}(Z^i) - \mathbb{E}(Y^i)  \right|$ is maximum for $i=k$ and assume w.l.o.g. that 
	\begin{align} \label{eqn:moment_lower_bound}
	\left| M_{k} \left(Z\right)   - M_{k} \left(Y\right) \right| = D.
		\end{align} 
	By the subgaussianity of $Y$, we have 
	\begin{align}
	\frac{1}{2}\int_\br | \rho_t(x) - g(x) | dx & \geq \frac{1}{2}  \int_{\abs{x} \leq A}  | \rho_t(x) - g(x) | dx  \nonumber\\
	&\geq \frac{1}{2A^i} \abs{\int_{\abs{x} \leq A}  x^i\left(\rho_t(x) - g(x) \right) dx} \nonumber\\
	&=  \frac{ \abs{ M_i \left( \rho_t(x) - g(x)     \right)   }  }{4A^i},\label{MB}
	\end{align}
	where $ A $ is chosen so that the last inequality is satisfied. 
	A crude lower bound for satisfying this inequality is $A \geq 4r^2 K \left(3+\log \frac{K}{D}   \right)$, which follows from the subgaussianity of $ \rho_t $.
	
	Combining the previous inequality with our application of Pinsker's inequality \eqref{eqn:Pinsker} and our assumption that $S(W_t) \leq \epsilon$ for all $i \leq r$ we get
	\begin{align} \label{eqn:moment_upper_bound}
	\abs{M_i \left( \rho_t(x) - g(x) \right)}  \leq 2\sqrt{2\epsilon} A^i. 
	\end{align}
	
	Note that since $Y$ has mean $0$ and variance $1$, we have $M_1(\rho(x)-g(x))=0$ and $M_2(\rho(x)-g(x))=0$. 
	We want to use \eqref{eqn:moment_upper_bound} and \eqref{eqn:moment_lower_bound} to prove our desired upper bound on $e^{-t}$. 
	To this end we will use the following relation from Lemma~\ref{lem:moments},
	
	 \begin{equation*}
	M_k\left( \rho_t(x) - g(x) \right) = \sum_{j=1}^{k} K^{(k)}_j e^{-jt} M_j\left( \rho - g \right) \left( \sqrt{1-e^{-2t}}   \right) ^{k-j} ,
	\end{equation*}
	where we set $K^{(k)}_i := \binom{k}{i} M_{k-i}(g(x)) \geq 0$ for brevity; note that $K^{(k)}_k = 1$.
	Using the inequalities we just noted, a first attempt to prove the desired upper bound on $e^{-t}$ might proceed as follows.
	\begin{align*}\label{eqn:moment_sum_inequality}
	2\sqrt{2\epsilon} \, A^k &\geq \abs{M_k \left( \rho_t(x) - g(x) \right)} \\ 
	&\geq e^{-kt} \abs{M_k(\rho(x)-g(x))} - \sum_{j=1}^{k-1} K^{(k)}_j  e^{-jt} \abs{M_j\left( \rho - g \right)} \left( \sqrt{1-e^{-2t}}   \right) ^{k-j}.
	\end{align*}
	Now if the last sum above were small compared to $e^{-kt} \abs{M_k(\rho(x)-g(x))}$ (say, less than by a factor of half) then we could conclude that $e^{-kt} \leq 4\sqrt{2\epsilon}\, A^k / \abs{M_k \left( \rho - g \right)}$, which using \eqref{eqn:moment_lower_bound} would imply 
	$e^{-t}  \leq A \left(4\sqrt{2\epsilon}/D\right)^{1/k}$, which is essentially what we want to prove. But we do not know that the sum is in fact small and so the above attempt 
	does not give us the desired upper bound. However, a more complex argument refining the above idea will do the job.
	
	Let $\epsilon' := 2\sqrt{2\epsilon}$. 
	Let $0 < T_3 < T_4 < \ldots < T_k = D$ be \emph{thresholds} whose values we will fix later when the constraints the $T_i$'s need to satisfy become clear. We will also have another parameter $R>0$ whose value will also be set later. We will show that $e^{-t} \leq R$. 
	
	We do a case analysis of $(k-2)$ cases:
	\begin{description}
		\item[Case $(3)$.] $M_3(\rho(x)-g(x)) > T_3.$ \\
		Using \eqref{eqn:moment_sum_inequality} this implies that $\epsilon' A^k \geq e^{-3t} T_3$, which in turn implies 
		$e^{-t} \leq A (\epsilon'/T_3)^{1/3}$. 
		\item[Case $(4)$.] $M_3(\rho(x)-g(x)) \leq T_3$ and $M_4(\rho(x)-g(x)) > T_4$. \\
		Using \eqref{eqn:moment_sum_inequality} this implies that
		\begin{align*}
		\epsilon' A^k &\geq e^{-4t} \abs{M_4(\rho(x)-g(x))} K^{(4)}_4 - e^{-3t} \abs{M_3(\rho(x)-g(x))} K^{(4)}_3 \sqrt{1-e^{-2t}} \\
			&\geq e^{-4t} T_4 K^{(4)}_4 - e^{-3t} T_3 K^{(4)}_3.
		\end{align*} 
		 
		\vdots 
		\item[Case $(k)$.] $\abs{M_3(\rho(x)-g(x))} \leq T_3, \abs{M_4(\rho(x)-g(x))} \leq T_4, \ldots, \abs{M_{k-1}(\rho(x)-g(x))} \leq T_{k-1}$, $\abs{M_{k}(\rho(x)-g(x))} > T_{k}$. \\
		In words, moments of orders $3$ to $k-1$ are less than their corresponding thresholds in absolute value, but the $k$'th moment is greater than the corresponding threshold. \\
		\begin{align*}
		\epsilon' A^k &\geq e^{-kt} \abs{M_k(\rho(x)-g(x))} K^{(k)}_k - e^{-(k-1)t} \abs{M_{k-1}(\rho(x)-g(x))} K^{(k)}_{k-1}  \left(\sqrt{1-e^{-2t}}\right) -\ldots \\
		& \qquad - e^{-3t} \abs{M_3(\rho(x)-g(x))} K^{(k)}_3 \left(\sqrt{1-e^{-2t}}\right)^{k-3} \nonumber \\
		&\geq  e^{-kt} T_k  - e^{-(k-1)t} T_{k-1} K^{(k)}  -\ldots - e^{-3t} T_3 K^{(k)},
		\end{align*}
		where $K^{(k)}$ is the maximum of $K^{(k)}_3, \ldots, K^{(k)}_{k-1}$. Note that $K^{(k)} < k!$.
	\end{description}

\begin{claim}
At least one of the $k-2$ cases above holds.
\end{claim}
\begin{proof}
If one of the first $k-1$ cases, namely $3, 4, \ldots, k-1$, holds, then we are done. If none of them holds, then this implies that 
all but the last conditions in Case~$(k)$ hold. But the last condition also holds by \eqref{eqn:moment_lower_bound}. Thus Case~$(k)$ holds if none of the earlier cases do.
\end{proof}

We now write a set of equations by replacing the inequalities in each of the above $k-2$ cases by the corresponding equation and replacing $e^{-t}$ by $R$:
\begin{align}
\epsilon' A^k &= R^3 T_3 \tag{E$3$}, \\ 
\epsilon' A^k &= R^4 T_4  - R^3 T_3 K^{(4)} \tag{E$4$}, \\
\ldots \nonumber \\
\epsilon' A^k &= R^k T_k  - R^{k-1} T_{k-1} K^{(k)}  -\ldots - R^3 T_3 K^{(k)} \tag{E$k$}.  
\end{align}
	
We choose $R$ and the thresholds $T_i$'s so that all of the above equations and $T_k = D$ hold true. 
Using the above equations we can explicitly solve for $R^i T_i$ in terms of $\epsilon', A$ and the $K^{(i)}$'s.
This is easy to do because it's a triangular system of linear equations: Equation E$3$ immediately gives $R^3T^3$; 
to solve for $R^4 T^4$ we use Equation E$4$, and substitute the value of $R^3T^3$ we computed, and so on.
There is a simple pattern which we illustrate with the first few solutions.
\begin{align*}
R^4 T_4 &= \epsilon'A^k(1 + K^{(4)}), \\
R^5 T_5 &= \epsilon'A^k(1 + K^{(5)}  + (K^{(5)}K^{(4)} + K^{(5)})), \\
R^6 T_6 &= \epsilon'A^k(1 + K^{(6)} + (K^{(6)}K^{(5)}+K^{(6)}) + (K^{(6)}K^{(5)}K^{(4)}+K^{(6)}K^{(5)}+K^{(6)}K^{(4)}+K^{(6)})).
\end{align*}
The general pattern is that in the expression for $R^i T_i$ there are $2^{i-3}$ terms, and these terms correspond to a 
collection of subsets of indices as illustrated in the above examples.
We will not explicitly write down or prove the general formula: this is straightforward as one can convince oneself from the above examples and perhaps one more for $R^7 T_7$, but would be cumbersome to write and, even more so, to read. 
Using the above solution, we loosely bound $R^k T^k$ from above by a usable quantity. We will use this bound later.
We use the facts $K^{(k)} > K^{(k-1)} > \ldots > K^{(3)} > 1$. Note that 
\begin{align}
R^k T_k &< \epsilon' K^{(k)} K^{(k-1)}\ldots K^{(3)} 2^{k-3}A^k. \label{eqn:R_upper_bound} 
\end{align}
%

Now we prove that if Case~$(i)$ holds for $3 \leq i \leq k$, then $e^{-t} \leq R$. 
We will prove this by contradiction. To this end assume that $e^{-t} > R$ and that Case $(i)$ holds.
Then from Equation (E$i$) above we have 
\begin{align*}
\epsilon' A^k &= R^i T_i - R^{i-1} T_{i-1} K^{(i)}  -\ldots - R^3 T_3 K^{(i)} \\
&< e^{-it} T_i - e^{-(i-1)t} T_{i-1} K^{(i)}  -\ldots - e^{-3t} T_3 K^{(i)} \\
&< e^{-it} \abs{M_i(\rho(x)-g(x))} - e^{-(i-1)t} \abs{M_{i-1}(\rho(x)-g(x))} K^{(i)}  -\ldots - e^{-3t} \abs{M_3(\rho(x)-g(x))} K^{(i)} \\
&< e^{-it} \abs{M_i(\rho(x)-g(x))}  - e^{-(i-1)t} \abs{M_{i-1}(\rho(x)-g(x))} K^{(i)}_{i-1} \left( \sqrt{1-e^{-2t}}   \right) -\\ & \qquad \ldots - e^{-3t} \abs{M_3(\rho(x)-g(x))} K^{(i)}_{3} \left( \sqrt{1-e^{-2t}}   \right) ^{i-3} \\
&\leq |e^{-it} M_i(\rho(x)-g(x)) + e^{-(i-1)t} M_{i-1}(\rho(x)-g(x)) K^{(i)}_{i-1} \left( \sqrt{1-e^{-2t}}   \right) + \\& \qquad \ldots + e^{-3t} M_3(\rho(x)-g(x)) K^{(k)}_{3} \left( \sqrt{1-e^{-2t}}   \right) ^{i-3} | \\
&= \abs{M_i(\rho_t(x)-g(x))}.
\end{align*}
The first inequality above follows from the fact that the polynomial 
$p(x) := x^i T_i - x^{i-1} T_{i-1} K^{(i)}  -\ldots - x^3 T_3 K^{(i)} $
has a positive derivative for $x \geq R$ and Descartes' rule of signs. 
 The second inequality follows because Case~$(i)$ holds. 
But the implication of the above chain of inequalities, namely $\epsilon' A^k <\abs{M_i(\rho_t(x)-g(x))}$, contradicts inequality \eqref{eqn:moment_upper_bound}. 
Thus our assumption that $e^{-t} > R$ must be false, proving that $e^{-t} \leq R$. 

To complete the proof we quickly prove the fact about polynomials just mentioned. By our assumptions $p(R) = \epsilon' A^k > 0$. 
Therefore, the derivative 
\begin{align*}
p'(R) = \frac{1}{iR}\left(R^i T^i - \frac{(i-1)}{i} x^{i-1} T_{i-1} K^{(i)} -\ldots - \frac{3}{i} R^3 T_3 K^{(i)}\right) > \frac{1}{iR} \cdot p(R) > 0.
\end{align*}
 Now since the derivative is positive, $p(x)$ continues to be positive for $x > R$. 
 
 To finish the proof we bound $R$ from above. We use our setting $T_k = D$ and \eqref{eqn:R_upper_bound} to get
 \begin{align*}
 R &\leq (2^{k-3} K^{(k)} K^{(k-1)}\ldots K^{(3)})^{1/k} \cdot A \cdot (\epsilon'/D)^{1/k} \leq 2 A K^{(k)} (2 \sqrt{2\epsilon}/D)^{1/k} \\
 &< 5 A \cdot k! (\epsilon/D^2)^{1/2k} \leq  5 A \cdot r! \cdot (\epsilon/D^2)^{1/2r}. \qedhere 
 \end{align*}
\end{proof}

    Note that from the proof we get the following bound on the entropy which will be useful for the analysis of the termination of the algorithm. 
    \begin{claim} \label{ent_bound}
        For $ \tilde X = \P_{\Gamma^{\perp}} X$ we have, $ S\left( \ip{\tilde{X}} {a} \right) \geq 0.5 D^2A^{-2r} $ for all unit vectors $ a $. 
    \end{claim}
        \begin{proof}
            This follows from \eqref{MB}. 
        \end{proof}

        \section{From Small Entropy to Small Euclidean Distance} \label{sec:SEnt_to_SEuc}
        	From the previous sections, we have shown that we can find critical points and that small entropy implies that the Gaussian component is large.
    But note that just the fact that a given point is a critical point of the entropy does not ensure that the point has small entropy. 
    For example, this point could be a maximizer for the entropy.
    In the following theorem, we show two properties of the entropy in our setting. 
    The critical points lie either entirely in the non-Gaussian space or entirely in the Gaussian subspace. 
    Given any point that does not lie entirely in the non-Gaussian subspace either the derivative is large or the entropy is small. 
    Then, using the earlier claim (\autoref{entropy_lower_bound}) relating the entropy and the projection onto the Gaussian subspace, we claim that the projection onto the Gaussian subspace is large. 
	We show this by relating the projection of the gradient of the entropy to the derivative of the entropy along the geodesic joining the current direction to the Gaussian direction. 
	Overloading notation, set $ S(\gamma(t)) := S(\ip{X}{\gamma(t)})$ and $ J(\gamma(t)) := J(\ip{X}{\gamma(t)})$. 
	In the following, $ \nabla S \left( \ip{X}{r}  \right) $ means $ \nabla_y S( \ip{X}{y})  |_{y=r}$, where $ \nabla_y $ denotes the Euclidean gradient.    
    
	\begin{theorem} \label{optimize}
		Given a point $ \alpha \in \mathbb{S}^{n-1} $ such that $ \frac{ \norm{P_{\Gamma}  \left( \alpha \right)}} { \norm{  P_{\Gamma^{\perp}} \left( \alpha \right) } }\geq \frac{1}{2C}$ for  $ C > 0$ and $ \norm{ \nabla S \left(\left\langle X , \alpha \right\rangle  \right) }\leq \epsilon$, then 
		\begin{equation*}
			\norm{P_{\Gamma} \left( \alpha \right) }\geq  1  -  \left( 5 A r! \left(\frac{C \epsilon}{D^2}\right)^{\frac{1}{2r}}  \right).
		\end{equation*}
	\end{theorem}

	\begin{proof}
		 Without loss of generality, we use the (Cartesian) coordinate system such that $ \alpha = \left( \alpha_1, \alpha_2 , 0 \dots 0 \right) $ where $\alpha_1, \alpha_2 > 0$, and the first coordinate represents the non-Gaussian part and the second the Gaussian part. Consider the path $ \gamma(s) = \left( e^{-s} , \sqrt{1 - e^{-2s}}, 0 , \dots , 0\right) $ and let $ t $ be such that $ \gamma(t) = \alpha $. The premise that $ \frac{ \norm{ P_{\Gamma} \left( \alpha \right) } }{  \norm{ P_{\Gamma^{\perp}} \left( \alpha \right) }}\geq \frac{1}{2C}$ is then equivalent to $\frac{\sqrt{1 - e^{-2t} }}{e^{-t}} \geq \frac{1}{2C} $.
		From \autoref{thm:ent_der}, we have 
			\begin{align} \label{eqn:SJ}
		      S(\ip{X}{\gamma(t)}) \leq \frac{1}{2} \abs{\frac{d}{dt} S(\ip{X}{\gamma(t)})}  .
		\end{align}
		But, by the chain rule, we have 
		\begin{align*}
			\abs{\frac{d}{dt} S(\ip{X}{\gamma(t)})} &= \abs{ \ip{\nabla S(\ip{X}{\gamma(t)}) }{\gamma'(t_0)} } \\
			&\leq  \norm{\nabla S(\ip{X}{\gamma(t)})} \cdot \norm{\gamma'(t_0)} \\
			& \leq \epsilon \cdot \norm{  \left( -e^{-t} , \frac{-e^{-2t}}{\sqrt{1-e^{-2t}}} , 0 \dots , 0  \right) }  \\
			&\leq \epsilon \cdot \frac{e^{-t}}{\sqrt{1 - e^{-2t} } }. 
			\end{align*}
			This along with \eqref{eqn:SJ} gives 
			 \begin{equation*}
			 	\frac{2\sqrt{1 - e^{-2t} }}{e^{-t}}S(\ip{X}{\gamma(t)}) \leq \epsilon.
			 \end{equation*}
			Recalling that $\frac{\sqrt{1 - e^{-2t} }}{e^{-t}} \geq \frac{1}{2C} $, we get $ S(\ip{X}{\gamma(t)}) \leq C\epsilon $. Now by \hyperref[entropy_lower_bound]{Theorem \ref*{entropy_lower_bound}}, we get, 
            \begin{equation*}
                \norm{ P_{\Gamma} \left( \alpha \right)    } \geq \sqrt{1 - \left( 5 A r! \left(\frac{C \epsilon}{D^2}\right)^{\frac{1}{2r}}   \right)^2  }.
            \end{equation*}
            Using the inequality $ 1 - x \leq \sqrt{1 - x^2}  $ for $ x \in \left[0,1\right] $, in the above inequality, we get 
            \begin{equation*}
            \norm{ P_{\Gamma} \left( \alpha \right)    } \geq 1 - \left( 5 A r! \left(\frac{C \epsilon}{D^2}\right)^{\frac{1}{2r}}   \right),
            \end{equation*}
            which is the inequality we need. 
	\end{proof}
	
	The above analysis also shows that given that a point is sufficiently far from completely lying in the non-Gaussian space, then it can be a stationary point only if it lies completely in Gaussian direction. 
	This indeed shows that local minima can only lie either completely in the Gaussian space or completely in the non-Gaussian space and since we are unlikely to start completely in the non-Gaussian space and the gradient always has a projection along the Gaussian direction, we end up in the Gaussian space. 
    Note that this claim also shows that the derivative of the projection onto the Gaussian component is positive (given that there is a Gaussian component left). 
    Thus the guarantee on the minimum Gaussian projection in the initial step remains to hold in the subsequent steps due to choice of the step size in the descent algorithm. 
	
	We use the following elementary theorem to show that a random vector does indeed have sufficient Gaussian projection with high probability.   
    
    \begin{lem}[see Lemma 2.1 \cite{becker2016new}]\label{scap}
    	For $\alpha \in (0,1)$ and $v \in \mathbb{S}^{n-1}$, let $ C_{\alpha,v} $ denote the set $ \{ x \in \mathbb{S}^{n-1} : |\ip{x}{v} | \geq \alpha \}  $. Then, for $ r \in \mathbb{S}^{n-1} $ be uniformly drawn, we have 
    	\begin{equation*}
    		\Pr [ r \in C_{\alpha,v}  ] = \mathsf{poly} \left( \frac{1}{n}  \right) \left( \sqrt{1 - \alpha^2}  \right)^n.
    	\end{equation*}
    \end{lem}
    
	\begin{theorem}[Spherical Concentration] \label{sp_conc}
		Let a uniformly generated $ r \in \mathbb{S}^{n-1} $ and let $ V_1 $ and $ V_2 $ be two nonempty subspaces of $ \br^n $. Then
		the orthogonal projections of $r$ onto $V_1$ and $V_2$ satisfy
        \begin{equation*}
            \Pr\left[ \frac{ \norm{  \P_{V_1}(r) }   }   { \norm{ \P_{V_2}(r)} } \geq \frac{1}{n^2}   \right] \geq \mathsf{poly} \left( \frac{1}{n}  \right)  \left( 1 - \frac{1}{n^4}  \right)^{\frac{n}{2}}.
        \end{equation*}
	\end{theorem}
	\begin{proof}
		We have
		\begin{align*}
		\Pr\left[ \frac{ \norm{\P_{V_1}(r) }   }   { \norm{\P_{V_2}(r)} } \geq \frac{1}{n^2}  \right] 
		&\geq \Pr\left[  \norm{\P_{V_1}(r)} \geq \frac{1}{n^2} \land \, \norm{\P_{V_2}(r) }\leq 1   \right] \\
        &\geq  \Pr \left[ \norm{\P_{V_1}(r)} \geq \frac{1}{n^2} \right]. \;\;\;\;\text{(as $\norm{\P_{V_2}(r) } \leq 1$ is always true).}
        \end{align*}
        This volume is smallest when $V_1$ is one-dimensional and in that case, 
        this is just the volume of the spherical cap with a specified angle and thus from \autoref{scap} with 
        $\alpha = 1/n^2$ and $v$ the unit vector whose span is $V_1$, we get
        \begin{align*}
        \Pr\left[ \frac{ \norm{ \P_{V_1}(r) }   }   { \norm{ \P_{V_2}(r)} } \geq \frac{1}{n^2}  \right] &\geq \mathsf{poly} \left( \frac{1}{n} \right) \left( 1 - \frac{1}{n^4}  \right)^{\frac{n}{2}}.
         \end{align*}
	\end{proof}
        
        \section{Estimation of Gradient of Entropy using Samples} \label{subsec:Est_Grad}
        	For running the optimization algorithm, we need access to the entropy of the projected random variables and also to the gradient of the entropy. 
	First, we note that given access to the entropy, we can approximate the gradient using a finite difference.

	\begin{fact}[Taylor's Approximation Theorem]\label{der}
		Let $ f : (a,b) \to \br $ be a twice-differentiable function. Then, for any $ \left( x , x+h \right) \subseteq \left(  a , b \right) $, we have  
		\begin{equation*}
			\frac{f\left( x + h  \right) - f(x ) }{ h } = \frac{df(x)}{dx} + \frac{d^2f(c)}{dx^2} \frac{h}{2}
		\end{equation*}
		for some $ c \in \left( x , x + h  \right)$. 
	\end{fact}
	
	Now, we use the above approximation theorem for each partial derivative and then use these to construct the gradient vector. 
	
	\begin{theorem}\label{grad_est}
		Let $ f : \br^n \to \br $ be a twice differentiable function with second order partial derivatives bounded by $ L $. Then, given access to a function $ \alpha $ such that for all $ x $, $ \abs{ f(x) - \alpha(x)  } \leq   \delta $, we can find a function $ \tilde{\nabla }f $ that approximates the gradient $ \nabla f $ to within $ 2 \sqrt {L \delta n } $, that is for all $ x $, $ \norm{ \tilde{\nabla }f(x) - \nabla f(x)   } \leq 2 \sqrt {L \delta n } $. 
	\end{theorem}
	\begin{proof}
		We find approximations of the partial derivatives using a finite difference at distance $ h $ in the given direction. 
		Let $ e_i $ be the standard basis elements in $ \br^n $. 
		 Then, using the triangle inequality and \autoref{der}, we get 
		\begin{align*}
			\left| \frac{\alpha(x) - \alpha(x+he_i)}{h} - \frac{\partial f(x)}{\partial x_i }\right| &\leq \left| \frac{\alpha(x) - \alpha(x+he_i)}{h} - \frac{f\left( x   \right) - f(x + he_i ) }{ h } \right| + \left| \frac{f\left( x   \right) - f(x + he_i ) }{ h } - \frac{\partial f(x)}{\partial x_i}\right| \\
			& \leq \frac{2\delta}{h} + \left| \frac{\partial ^2f(c)}{\partial x_i^2} \right| \frac{h}{2}.
			\intertext{We minimize this function in $ h $ using the AM-GM inequality. Let $ h_0 $ be the point of minimum.}
			\left| \frac{\alpha(x) - \alpha(x+h_0e_i)}{h_0} - \frac{\partial f(x)}{\partial x_i}\right| &\leq  2\sqrt{\left| \frac{\partial ^2f(c)}{\partial x_i^2} \right| \delta }.
		\end{align*}
		Let $ \tilde{\frac{\partial f }{\partial x_i }}(x) $ be the above approximations to the partials at $ x $ and  $ \tilde{\nabla } f(x)  $ be the vector obtained by using the approximate partial derivatives in each component. Then,
		\begin{align*}
			\norm{  \tilde{\nabla } f(x) - \nabla f (x) } \leq & \sqrt{ \sum_{i = 1}^{n}  \left( \tilde{\frac{\partial f }{\partial x_i }} (x) - \frac{\partial f }{\partial x_i } (x)  \right)^2  }  \\
			\leq &  \sqrt{ \sum_{i = 1}^{n}  4L \delta  } \\
			\leq & 2\sqrt{L\delta n}.
		\end{align*}
		This gives us the required bound. 
	\end{proof}
	
	Though the estimation of differential entropy is a well studied problem and remains an active area of research (for example, see \cite{CIT-021} and \cite{han2017optimal}), it is a bit difficult to pin down the constants and the exact rate in the existing literature for our setting. Hence we resort to using a histogram based estimator. The algorithm proceeds by truncating the density at a stage $ A $, picking a bucket size $ B $ and the number of samples $ N $. Using the sample histogram to estimate the density, we compute the entropy as the appropriate integral approximation. This is captured by the following theorem, details of the proof of which are deferred to \hyperref[App:Ent_Est]{Appendix \ref*{App:Ent_Est}}. 
    
    
    \begin{theorem} \label{ent_est}
        Let $ f $ be the density of the random variable $   Y_t = \sqrt{1-t^2}X + tZ $ where $ Z $ is standard univariate Gaussian independent of $X$, which is $ K $-subgaussian, for constant $ t $. Then, there exists an algorithm that estimates the entropy of $ f $ with error $ \epsilon $ and probability of correctness $ \gamma $ that uses $ N = \log( 1/ \gamma ) K^2   e^{ \frac{K^4}{t^2}} \epsilon^{ -cK^4 }  $ samples for some constant $ c $.
    \end{theorem}

        \section{Lipschitz Constant of the Gradient of the Entropy} \label{subsec:Lip_Grad}
      	For the analysis of the gradient descent algorithm, we need the gradient of the objective function, in our case the entropy, is Lipschitz continuous and the running time depends polynomially on the Lipschitz parameter $ L $.
In this section, we show that for random variables of our interest, the Lipschitz constant is polynomially related to the parameters of the problem, and thus completing the proof that the number of steps in the gradient descent is polynomially dependent on the parameters of the problem. 
Specifically, we show that for random variables $ Y_t  $ that can be written as $ Y_t = \sqrt{1 - t^2 }X + tZ  $, where $ X $ is a $K$-subgaussian random variable, $ Z $ is an independent Gaussian random variable and $ t $ is a parameter, we show that the Lipschitz constant is polynomial in $ K $ and $ t^{-1} $.
Note that the Lipschitz parameter improves upon increasing $ t $. 
This is because the addition of Gaussian noise smoothens the density of the random variables and providing better behavior of the derivatives.
In our setting, we will initialize this theorem with $ t $ at least $ \poly(n^{-1}) $. 
We can do this because we have shown that when there is a Gaussian component left to be found, a random initialization, with high probability, has a projection onto the Gaussian subspace with norm at least  $ n^{-2} $ and  each gradient step only improves the Gaussian projection, giving us sufficiently large $ t $.
 
    \begin{theorem} \label{Lip_Bound1}
        Let $ X \in \br^n$ be an isotropic random variable that is $K$-subgaussian and $ Z \in \br^n$ be a standard Gaussian random variable independent of $ X $.
        Let $ Y_t = \sqrt{1-t^2}X +tZ  $ for $t \in [0,1]$.
        Let $ S(a) = S \left(\ip{Y_t}{a} \right)  $. 
        Then, $ \nabla S $ is Lipschitz continuous with Lipschitz constant bounded by a polynomial in $ K $ and $ t^{-1} $. 
    \end{theorem}    

    The proof of the above theorem is fairly technical. 
    First, note that given a density $ \rho_a(z) $ of a random variable of the form $ \ip{X}{a} $, there are two notions of differentiation to be considered: First, there is the differentiation of the density with respect to its argument i.e. $ \partial_{z} \rho_a (z) $. 
    Secondly, we think of the derivative of the density with respect to the direction $ a $ i.e. $ \partial_a \rho_a (z) $. 
    We first give upper and lower bounds for the density and its derivatives of the first kind for random variables acted upon by the Ornstein--Uhlenbeck semigroup.
    We then note that since we only need to bound the largest eigenvalue of the Hessian matrix, we need to bound the double derivative along its eigenspace. 
    To do this, we bound the derivative of the entropy for arbitrary two dimensional subgaussian random variables perturbed with Gaussian noise. 
    We then need to translate the bound on the derivatives of densities of the first kind to the derivatives of the density of the second kind.
    Since, we are differentiating the entropy with respect to the direction, we use the bounds on the derivatives of the second kind to bound the derivatives of the entropy.
    But, the main issue that turns up is that expression for the derivative of the entropy has in its denominator the density and using the lower bound on the density directly gives us an exponential bound. 
    To resolve this, we exploit the structure of the Ornstein--Uhlenbeck semigroup to get relations between the density and its derivatives, which leads to cancellations which we use to get polynomial bounds on the Lipschitz constant. 
    We defer this proof to \autoref{Lip_con}. 
        
       	\section{Analysis of Multiple Runs of the Descent Algorithm} \label{subsec:mult_run}
        In this section, we analyze the error accumulation from running the algorithm for multiple iterations. 
Recall that in iteration $k+1$ of the Full Algorithm, we work by projecting our random variable $X$ to the space orthogonal to the orthonormal vectors  $\lambda_1, \ldots, \lambda_k$ (which are close to the Gaussian subspace)
found so far. 
A key observation here is that the random variable thus obtained obeys the NGCA model. The parameters of the model
worsen with $k$ in the sense that while the model remains NGCA, the distance of the marginals of the non-Gaussian component from Gaussian
becomes smaller and hence in the next iteration the algorithm has to work harder to find $\lambda_{k+1}$ that is close to the Gaussian subspace. The reason for the distance of the marginals of the non-Gaussian component from Gaussian
becoming smaller is that while the vectors $\lambda_1, \lambda_2, \ldots$ output by the algorithm are close to $\Gamma$, they may not actually be
in $\Gamma$, and because of this the non-Gaussian component gets a small additive Gaussian term.
We will analyze this process and show that the worsening of the parameters is mild and hence the algorithm remains efficient
and the final subspace $V$ found by the algorithm is close to the true Gaussian subspace in the subspace distance $d(\cdot, \cdot)$ defined 
earlier. 
%

We first prove that if each vector $\lambda_i$ output by the Full Algorithm is sufficiently close to the Gaussian subspace $\Gamma$, then so is their span in subspace distance.  

Our proof of the following lemma is perhaps overly long and perhaps follows from arguments in matrix perturbation theory (such as Wedin's theorem on perturbations of singular spaces). Invoking existing perturbation bounds seems to require non-trivial work; we instead provide a self-contained proof in \autoref{App:Subspace}. 

 Note that in the application of the lemma below we cannot assume that the $ \gamma_i $ are orthonormal.
  That is, though the directions that we find in each iteration are orthogonal to each other, we cannot ensure that the directions that they approximate are orthogonal. 

\begin{lem}\label{GS} 
	Let $ 0< \epsilon < \frac{1}{25n^2}$ and let $ \lambda_1, \dots, \lambda_k \in \br^n $ be orthonormal vectors and $ \gamma_1, \dots, \gamma_k \in \br^n  $ be unit vectors and 
	satisfying the following condition for all $i$
	\begin{equation*}
	\ip{\lambda_i}{\gamma_i} \geq 1 - \epsilon.
	\end{equation*}
	Denote by $ \Lambda_k$ and $ \Gamma_k$ their respective spans. 
	Then, $ \Lambda_{k} $ and $ \Gamma_{k} $ have the same rank and $ d\left( \Lambda_k, \Gamma_k  \right) \leq 6 k^2 \sqrt[4]{\epsilon}$.  
\end{lem}

%

We will now prove that after $k$ iterations of the Full Algorithm, the random variable obtained (which now lives in $\br^{n-k}$) obeys the NGCA model and we also show how the distance of the non-Gaussian component changes through these iterations. 

After $k$ iterations of the Full Algorithm the situation looks like this
(we will prove below by induction on $k$ that this is indeed the case): 
we have found orthonormal vectors $\lambda_1, \ldots, \lambda_k \in \br^n$ each close to the Gaussian subspace $\Gamma$. 
Set $\Lambda_k := \mathsf{span}\{\lambda_1, \ldots, \lambda_k\}$. In the $(k+1)$'th iteration, we work in the orthogonal 
subspace $\Lambda_k^\perp$. In other words, our random variable now is $\P_{\Lambda_k^\perp}X$, which can be decomposed using orthogonal 
projections as 

\begin{align} \label{eqn:decomposition}
\P_{\Lambda_k^\perp}X =  \P_{\Lambda_k^\perp \cap \Gamma}X  + \P_{(\Lambda_k^\perp\cap\Gamma)^\perp\cap\Lambda_k^\perp}X. 
\end{align}

We will now show that the above equation provides the decomposition of $\P_{\Lambda_k^\perp}X$ into mutually independent Gaussian and non-Gaussian components respectively, thus showing that $\P_{\Lambda_k^\perp}X$ satisfies the NGCA model. 

Define $\gamma_i := \P_\Gamma \lambda_i$ and $\Gamma_k := \mathsf{span}\{\gamma_1, \ldots, \gamma_k\}$.
\begin{claim} \label{claim:gammaperp}
    Let $ \Gamma_{k} $ and $ \Lambda_{k} $ be as above. Then, 
    \begin{equation}
        \Gamma_k = \Gamma \cap (\Lambda_k^\perp\cap\Gamma)^\perp.
    \end{equation}
\end{claim}
\begin{proof}
	First we will show that $ \Lambda_k^\perp\cap\Gamma = \Gamma_{k}^{\perp} \cap \Gamma $. To see this consider $ a \in \Lambda_k^\perp\cap\Gamma $. Now, for $ \gamma \in \Gamma_k $ note that  
	\begin{align*}
		\ip{a}{\gamma} = & \ip{a}{\P_{\Gamma} a_1 } \\
		= & \ip{\P_{\Gamma} a  }{a_1} \\
		= & \ip{a}{a_1}\\
		= & \, 0.
	\end{align*} 
	For the other direction, consider $ b \in \Gamma_{k}^{\perp} \cap \Gamma  $. Then, for any $ \lambda \in \Lambda_{k} $, note that 
	\begin{align*}
		\ip{b}{\lambda} = &  \ip{\P_{\Gamma}  b }{\lambda} \\
		=& \ip{b}{\P_{\Gamma } \lambda } \\
		= & \, 0 .
	\end{align*}
	We now need to show that $\Gamma_k = \Gamma \cap (\Gamma_k^\perp\cap\Gamma)^\perp$. But this follows from the observation that $ \Gamma_{k} \subseteq \Gamma $ and $ (\Gamma_k^{\perp})^\perp = \Gamma_k $
\end{proof}

\begin{claim}
    Assume that $ \epsilon < \frac{1}{25n^2} $. Then, 
	$\P_{\Lambda_k^\perp \cap \Gamma}X$ is the standard Gaussian r.v. in $\br^{\dim(\Gamma)-k}$.
\end{claim}
\begin{proof}
	From the above claim, we have $\Lambda_k^\perp \cap \Gamma = \Gamma_{k}^{\perp} \cap \Gamma $ which is a subspace of the Gaussian subspace $\Gamma$, and hence by the property of the standard Gaussian that its projections are standard Gaussian, it follows that $\P_{\Lambda_k^\perp \cap \Gamma}X$ is standard Gaussian. 
	It remains to show that it has the full dimension, namely $\mathsf{dim}(\Gamma)-k$. 
	This is true because $\dim(\Lambda_k^\perp \cap \Gamma) = \dim\left( \Gamma_{k}^{\perp} \cap \Gamma   \right) = \dim(\Gamma)-k$, because $ \Gamma_k $ has rank $ k $ as shown in \autoref{GS}.
\end{proof}

\begin{claim}
	The two random variables in the R.H.S. in \eqref{eqn:decomposition} are independent.
\end{claim}
\begin{proof}
	$\P_{\Lambda_k^\perp \cap \Gamma}X $ and 
	$\P_{(\Lambda_k^\perp\cap\Gamma)^\perp\cap\Lambda_k^\perp}X$ are projections onto orthogonal subspaces. Note that since $\P_{\Lambda_k^\perp \cap \Gamma}X $ is a projection of $ X $ onto to the Gaussian component, we have that it is independent of the projection onto its orthogonal subspace i.e. $\P_{\left(\Lambda_k^\perp \cap \Gamma \right)^{\perp}  }X $. The claim follows by noting that $\P_{(\Lambda_k^\perp\cap\Gamma)^\perp\cap\Lambda_k^\perp}X$ is projection on to a subspace of $ \left(\Lambda_k^\perp \cap \Gamma \right)^{\perp}  $. 
    
    As we have proved, 
	the first part of  \eqref{eqn:decomposition} is Gaussian. The second projection gets the non-Gaussian part of $X$ which by the NGCA model is independent of 
	the Gaussian subspace.
    However, the second projection also gets part of the Gaussian component coming from $\Gamma_k$ which 
	is orthogonal to $\Lambda_k^\perp \cap \Gamma$ by our construction in the proof of  \autoref{claim:gammaperp}, and hence is independent
	of $\P_{\Lambda_k^\perp \cap \Gamma}X $ (by the property of the standard multivariate Gaussian that projections to orthogonal subspaces are independent).  
\end{proof}

    The claims above proves that the model remaining after we take orthogonal projections is NGCA with the right dimension as long as our errors are small enough.
    Note that we need the errors to be sufficiently small to ensure that we are indeed approximating linearly independent vectors in successive iterations. 
    We summarize this in the following theorem. 
\begin{theorem}\label{Proj_NGCA}
	Let $ \Gamma  $ be the Gaussian subspace with dimension at least $ k+1 $ of an NGCA model. Let $ \lambda_1 \dots \lambda_k  $ be the vectors such that there exist $ \gamma_1 \dots \gamma_k \in \Gamma$ such that $ \ip{\lambda_i}{\gamma_i} \geq 1 - \epsilon $ with $ \epsilon \leq 1/(50n^2) $. Let $ \Lambda_k $ be their span. Then $ \P_{\Lambda_k^{\perp }} X $ follows the NGCA model and the Gaussian part has dimension $ \dim(\Gamma) -k $. 
\end{theorem}

The following claim shows that in the new NGCA random variable obtained from the orthogonal projection, the Gaussian ``noise" from directions already approximated is small. 

\begin{claim}\label{Gaus_Proj}
	 For any $ v \in (\Lambda_k^\perp\cap\Gamma)^\perp\cap\Lambda_k^\perp $
	 \begin{equation*}
	 	\Ex \ip{\P_{\Gamma_k}v}{X}^2 \leq \norm{\P_{\Gamma_k}-\P_{\Lambda_k}}_2^2.
	 \end{equation*}
\end{claim}
\begin{proof}

	For any $v \in (\Lambda_k^\perp\cap\Gamma)^\perp\cap\Lambda_k^\perp$ write $v = \P_\Gamma v + \P_{\Gamma^\perp} v$.
	Note that in fact $\P_\Gamma v \in \Gamma_k$ as $v \in (\Lambda_k^\perp\cap\Gamma)^\perp\cap\Gamma = \Gamma_k$ (shown in the proof of \autoref{claim:gammaperp}.
	Thus 
	\begin{align*}
	\ip{v}{X} = \ip{\P_{\Gamma_k}v}{X} + \ip{\P_{\Gamma^\perp}v}{X}. 
	\end{align*}
	By the definition of $\Gamma$, we have that $\ip{\P_{\Gamma^\perp}v}{X}$ is non-Gaussian and $\ip{\P_{\Gamma_k}v}{X}$ is Gaussian. 
	We will now show that $\norm{\P_{\Gamma_k} v}$ is very small for all $v \in (\Lambda_k^\perp\cap\Gamma)^\perp\cap\Lambda_k^\perp$.
	Thus the Gaussian component in $\ip{v}{X}$ is small. 
	\begin{align*}
	\norm{\P_{\Gamma_k}v}_2
	&\leq \max_{u \in \br^n \,:\; \norm{u}_2 = 1} \norm{\P_{\Gamma_k} \P_{(\Lambda_k^\perp\cap\Gamma)^\perp\cap\Lambda_k^\perp} u}_2\\
	&= \norm{\P_{\Gamma_k} \P_{(\Lambda_k^\perp\cap\Gamma)^\perp\cap\Lambda_k^\perp}}_2 \\
	&\leq \norm{\P_{\Gamma_k} \P_{\Lambda_k^\perp}}_2 \\
	&=\norm{\P_{\Gamma_k} (I-\P_{\Lambda_k})}_2 \\
	&\leq \norm{\P_{\Gamma_k}-\P_{\Lambda_k}}_2, \;\;\;\;
	\end{align*}
		where the last step used the following fact (see, e.g., \cite{MeyerBook}, pages 453-454):
	For subspaces $U, V$ we have 
	\begin{align}
	\norm{\P_U - \P_V}_2 = \max\left\{\norm{\P_V(I-\P_U)}_2, \norm{\P_U(I-\P_V)}_2\right\}.
	\end{align}
	Thus the Gaussian random variable $\ip{\P_{\Gamma_k}v}{X}$ has variance
	\begin{align*}
	\Ex \ip{\P_{\Gamma_k}v}{X}^2 = \norm{\P_{\Gamma_k}v}_2^2 \leq \norm{\P_{\Gamma_k}-\P_{\Lambda_k}}_2^2. 
	\end{align*}
    as required. 
\end{proof}

	In the following theorem, we show that the deterioration of parameters caused by the projections of the already approximated Gaussian directions is small. Controlling this allows us to claim that the running time of the algorithm does not blow up due to the worsened parameters.
	
	\begin{lem}\label{alt_mom}
		Let $Y \in \br$ be a random variable with expectation $0$ and variance $1$, and let $ Z $ be an independent standard Gaussian random variable. Assume that there exists $ k \leq r $, such that $ \abs{M_k(Y) - M_k(Z)} \geq D $, let $k$ be the smallest value
		for which this inequality holds true. Consider $ W_t = \sqrt{1 - t^2} Y + tZ $. Then,
		\begin{equation*}
		\abs{M_{k} \left( W_t  \right) - M_k\left( Z \right)}   \geq D \left(  \left(1-t^2\right)^{\frac{k}{2}}  - t\left( 1 - t^2 \right)^{\frac{3}{2}} \left( 1+ \sqrt{k-3} \right)^{k}\right). 
		\end{equation*} 
	\end{lem}
	Note that the above function takes the value $ 1 $ for $ t =0 $ and for $ t $ sufficiently close to $0$, the lower bound above is positive.
	Thus, with close enough approximation the above estimate can be used to bound the running time of the algorithm.
	\begin{proof}
		Denote by $ \rho $ the density of $ Y $ (the assumption that $Y$ has density here is for ease of notation).	
		Note that $k > 2$ as $M_1(Y) = M_1(Z)$ and $M_2(Y) = M_2(Z)$.
		By \autoref{lem:moments}, for $k>3$ we have the following expression(the case of $k=3$ is simpler and will be discussed below).
		For simplicity of notation, let $s = \sqrt{1-t^2} $. 
		\begin{align*}
		\abs{M_k(W_t) - M_k(Z)}  &= \abs{\sum_{j = 1}^{k} \binom{k}{j} s^j t^{k-j}M_{k-j} (Z) M_{j} \left( \rho - g \right)} \\ 
		&\stackrel{(i)}{=} \abs{ s^k M_k \left( \rho - g \right) + \sum_{j = 3}^{k-1} \binom{k}{j} s^j t^{k - j} M_{k-j} (Z) M_{j} \left( \rho - g \right)} \\ 
		& \stackrel{(ii)}{\geq} s^{k} D - D  \abs{\sum_{j = 3}^{k-1} \binom{k}{j} s^{j} t^{k - j} M_{k-j} (Z)} \\
		& \geq s^{k} D - D s^{3} t\sum_{j = 3}^{k-1} \binom{k}{j}  \abs{M_{k-j} (Z)} \\
		&\geq s^{k} D - D s^{3} t \sum_{j = 3}^{k-1} \binom{k}{j}  (k-j)^{(k-j) / 2} \\
		& \geq s^{k} D - D s^{3} t \sum_{j = 3}^{k-1} \binom{k}{j}  (k-3)^{(k-j) / 2} \\
		& \geq D \left(  s^{k} - s^{3} t\left( 1+ \sqrt{k-3} \right)^{k}\right).
		\end{align*}
		For equality $(i)$ above we used our assumptions $\Ex Y = 0$ and $\Ex Y^2 =1$. For inequality $(ii)$ above we used our assumption that 
		$k$ is the smallest number for which $\abs{M_k(Y) - M_k(Z)} \geq D$ and hence $\abs{M_i(Y) - M_i(Z)} < D$ for $i < k$. 
		When $k=3$ the above inequality simplifies to 
		\begin{equation*}
		\abs{M_3(W_t) - M_3(Z)} \geq D s^3. \qedhere
		\end{equation*} 
	\end{proof}
	
	\begin{cor} \label{Bern}
		Let $ W_t $, $ X $ and $ Y $ be as above. Then, 
		\begin{equation*}
			\abs{M_{k} \left( W_t  \right) - M_k\left( Z \right)}  \geq D \left( 1 - \frac{kt^2}{2} - tk^{\frac{k}{2}} \right)
		\end{equation*}
	\end{cor}
	\begin{proof}
		Note that from Bernoulli's inequality, we have $ \left(  \left(1-t^2\right)^{\frac{k}{2}}  - t\left( 1 - t^2 \right)^{\frac{3}{2}} \left( 1+ \sqrt{k-3} \right)^{k}\right) \geq \left (1 -  \frac{kt^2}{2} - tk^{\frac{k}{2}}    \right) $. 
		This in conjunction with the \autoref{alt_mom}, gives us the result. 
	\end{proof}
    
    \begin{rem}
    	Note that if one assumed the $ r $th moment is different from the Gaussian, while lower moments are equal, then the above claim is easier and we get better bounds,  that is $ \abs{  M_r\left(W_t\right) - M_r\left( Z\right)  } \geq \left( 1 - t^2  \right)^{\frac{r}{2}} D $. 
    \end{rem}
    
     Recall that the algorithm had two error parameters $ \epsilon_1 $ and $ \epsilon_2 $. The first, $ \epsilon_1 $ is the threshold for the modulus of gradient in the termination check while $ \epsilon_2 $ is the threshold for the entropy in the termination check. Below we list the set of events which cause the algorithm to fail. 
    \begin{enumerate}
        \item Let $ E_i $ be the event that at the beginning of the $ i $th iteration, we pick a vector such that the projection onto the Gaussian part is too small i.e. $ \P_{\Gamma} \left( a  \right) \leq 1/C  $ where $ a $ is the random vector drawn at the start of round $ i $. Note that we will later set this $ C = n^2 $. Therefore, $ \Pr[E_i] \leq n^{-3} $ from \hyperref[sp_conc]{Theorem \ref*{sp_conc}}. 
        \item Let $ F_i $ be the event that there was an error in the estimation of entropy at attempt $ i $. We will later set the parameters such that the $ \Pr[ F_i ] \leq N^{-1}n^{-2} $. 
    \end{enumerate}
    Thus, probability of a bad event happening is $ \Pr \left[ \cup E_i \cup F_j   \right] \leq \sum_{i=1}^{n}  \Pr\left[ E_i \right]  + \sum_{i=1}^{N}\Pr[ F_i   ] \leq  n^{-2} + NN^{-1}n^{-2} $, where $ N $ is the running time of the whole algorithm. 
    We bound $ N $ to be polynomial.
    Thus, we have $ o(n^{-1}) $ rate of error. 
    In the section ahead, we shall work conditioned on the event that an error doesn't occur, unless otherwise specified.


\begin{claim}\label{term}
	After $ k $ steps if there is no Gaussian direction left i.e. $ \Lambda_{k+1} \cap \Gamma$ has dimension zero, the algorithm halts with high probability.
\end{claim}
\begin{proof}
	If there are no remaining Gaussian directions, then by \hyperref[ent_bound]{Claim \ref*{ent_bound}} the entropy of each of the projections is larger than $ 0.5 D^2A^{-2r} $. Thus, setting $ \epsilon_2 = 0.5 D^2A^{-2r}$ and $ \epsilon_1 \ll \epsilon_2 $ in the full algorithm, we have that the descent algorithm never reaches a point that accepts. Thus, we output the space constructed so far. Also, we know that when there is a Gaussian direction, the gradient of the entropy becomes small along with the entropy and thus the accept conditions are satisfied. 
\end{proof}

\begin{claim}
	The subspace output by the algorithm has the same dimension as the non-Gaussian subspace with high probability.
\end{claim}
\begin{proof}
	This follows from the correctness of the optimization step and \hyperref[term]{Claim \ref*{term}}
\end{proof}

    
     To facilitate the estimation of entropy, we add Gaussian noise to the random variable at each step. Let us denote by $ t $ the noise that we add i.e. we work with the random variable $  \sqrt{1-t^2} X + t Z  $ where $ Z $ is independent Gaussian. Note that the addition of noise makes the estimation and the Lipschitz constant of the random variable better but it does make the moment difference decay. We begin with with moment gap of $ D $ but from \autoref{alt_mom} we know that the moment gap reduces to $ D \left(  \left(1-t^2\right)^{\frac{k}{2}}  - t\left( 1 - t^2 \right)^{\frac{3}{2}} \left( 1+ \sqrt{k-3} \right)^{k}\right) $ for where $ k $ is the first moment that disagrees by $ D $. In the next lemma, we provide a bound on the required amount of noise such that the gap does not become too small. 
    
    \begin{lem} \label{noise}
        Let $ Y $ be a random variable with unit variance and zero mean such that one of first $ r $ moments vary from the Gaussian ones by $ 2D $ i.e.  $\abs{M_k(Y) - M_k(Z)} \geq 2D$ for some $ k \leq r $. There exists a $ t' $ depending only on $ r $ such that $ \abs{M_{k'}(Y_t) - M_{k'}(Z)} \geq D$ for some $ k' \leq r $, where $ Y_{t'} = \sqrt{1-t'^2} Y + tZ $. 
    \end{lem}
    \begin{proof}
        We know from \autoref{Bern} that $ \abs{M_{k} \left( Y_t  \right) - M_k\left( Z \right)}   \geq 2D \left(1 -  \frac{kt^2}{2} - tk^{\frac{k}{2}}    \right)  $. Since, the multiplier of $2D$ is one when $ t=0 $ and is monotonically decreasing till its first root in $ t $, we are guaranteed a value of $ t $ for which the multiplier is $ 0.5 $. But also note that in this range, the multiplier is monotone in $ k $. Thus, we can pick the bound obtained by plugging in $ k = r $. 
    \end{proof}

    \section{Putting it All Together} \label{sec:Put_It_Together}
    We instantiate the Full Algorithm, with the choice of step size $ \eta $ and parameters $ \epsilon_1 $ and $ \epsilon_2 $ as follows: 
    \begin{align*}
        \eta = & O \left( \frac{1}{L}  \right), \\
        \epsilon_1 = & O\left(\frac{D^2 \epsilon^{8r}  }{r^{2r^2} \left( 5A r!  \right)^{2r  }   n^{16r + 4} } \right),   \\
        \epsilon_2 = & O \left(  \frac{D^2}{n^2  A^{2r} }.  \right)
    \end{align*}
    In the next claim, we bound the number of gradient steps we need to take to find the Gaussian subspace to within the required error, and show that above choice of parameter does indeed work. 
    \begin{theorem} \label{iter_comp}
    	Let $\left( Z, \tilde X  \right) \in \Gamma \oplus \Gamma^{\perp} = \br^n $ be a random variable satisfying the isotropic NGCA model. Suppose that $ \tilde X  $ and a positive integer $r$ satisfy the following conditions:
    	\begin{itemize}
    		\item  For each $ a \in \mathbb{S}^{n-1} $, there exists a $ k \leq r $ such that $  \left| \mathbb E \left\langle \tilde X , a \right\rangle ^k - \mathbb E \left\langle Z , a \right\rangle^k \right| > D $. In words, for each direction there is a $k$ such that the $ k $'th moment of non-Gaussian component along the direction differs from the $ k $'th moment of the Gaussian.
    		\item $ \tilde X $ is $ K $-subgaussian. 
    		\item Let $ S\left( u  \right)  $ denote $ S\left( \ip{X}{u}  \right) $. The gradient $ \nabla S \left( u\right) $ of the entropy is $ L $-Lipschitz. 
    	\end{itemize}
    	Then, with probability of error at most $ o(n^{-1}) $ the Full Algorithm returns a subspace $ \Lambda $  with $ \dim\left( \Lambda \right) = \dim\left( \Gamma^{\perp}  \right) $ and $ d\left(\Lambda , \Gamma^{\perp}  \right) \leq \epsilon$ in 
    	\begin{equation*}
    	N \leq  \frac{6^{25r}L^2 K^{4r} r^{12r^2 + 8r} \left( 3 + \log\left( \frac{K}{D}  \right) \right)^{4r}   }{\epsilon^{16r}  D^4 }  n^{48r + 9 } 
    	\end{equation*}
    	gradient descent steps. 
    \end{theorem}
    \begin{proof}
    	We shall estimate the time required to estimate the $(m + 1)$th direction up to error $ \delta $ given that we have already estimated the first $ m $ directions up to error $ \delta $.
    	This $ \delta $ will be used to specify $ \epsilon_1 $ in the main algorithm and will be chosen later.
    	As above, let $ \lambda_1 \dots \lambda_{m} $ be the directions found in the first $ k $ steps, $ \gamma_{1} \dots \gamma_{m} $ be their projections onto the Gaussian subspace and $ \Gamma $ respectively, and $ \Lambda_{m}$ and $ \Gamma_{m} $ be their spans.
    	We will set $ \delta $ small enough to ensure that their ranks match. 
    	The case when all the Gaussian directions have been found is dealt with by \hyperref[term]{Claim \ref*{term}}.
    	Thus, we will solely deal with the case in which $ \Lambda_{m}^{\perp} \cap \Gamma  $ is nonempty. 
    	
    	We first note that by \autoref{Proj_NGCA}, $ \P_{\Lambda_m^{\perp }} X$ obeys the NGCA model. 
    	What we need to check is how the parameters deteriorate due to the addition of Gaussian noise from the projection of the random variable onto the Gaussian component of $ \Lambda_m $, namely $(\Lambda_m^\perp\cap\Gamma)^\perp\cap\Lambda_m^\perp$.
    	In that regard, let us consider any unit vector $ v $ in the non-Gaussian component of $ \Lambda_m^{\perp} $. 
    	By \autoref{Gaus_Proj}, we have that $ \ip{X}{v} = \ip{\P_{\Gamma_m}  v }{X} + \ip{\P_{\Gamma^{\perp} } v }{X} $.
    	The first component is Gaussian and the second component is non-Gaussian independent of the first component. Again, by \hyperref[Gaus_Proj]{Claim \ref*{Gaus_Proj}}, the first component has standard deviation less than $\norm{ \P_{\Gamma_m } - \P_{\Lambda_m}}_2$. By \hyperref[GS]{Lemma \ref*{GS}}, $\norm{ \P_{\Gamma_m } - \P_{\Lambda_m}}_2 \leq \norm{ \P_{\Gamma_m } - \P_{\Lambda_m}}_F =d\left( \Lambda_{m} , \Gamma_{m} \right) \leq 6m^2 \sqrt[4]{\delta}  $.
    	Let $ k \leq r$ be the first moment of $ \ip{X}{v} $ that differs from the Gaussian. Then, by applying \autoref{Bern} with $ t = 6m^2 \sqrt[4]{\delta} $, we have 
    	\begin{align*}
    		\abs{M_{k} \left( \ip{X}{v} - M_{k} \left(  Z \right)   \right)  } \geq & D\left( 1 - {18km^4 \sqrt{\delta}  } - 6m^2 \sqrt[4]{\delta} k^{\frac{k}{2}}   \right) \\
    		\geq & D\left( 1 - {18rm^4 \sqrt{\delta}  } - 6m^2 \sqrt[4]{\delta} r^{\frac{r}{2}}   \right). 
    	\end{align*}
    	
    	This holds for every $ v $ in the non-Gaussian component.
    	 Let us call this bound $ D_{m+1} $.
    	 Also, note that from \autoref{subGauss}, we have the projected random variable satisfies the same bound on the subgaussian parameter. 
    	In \autoref{optimize}, setting $ C := n^2 $, to get within $ \delta $ of a Gaussian direction, i.e., to make $ \ip{\lambda_{k+1}}{ \gamma_{k+1}} \geq 1 -\delta $, we need 
    	\begin{equation*}
    	\norm{ \nabla S } \leq \frac{ D_{m+1}^2    \delta ^{2r } }{n^2 \left( 5 A r! \right) ^{2r} },
    	\end{equation*} 
    	where $ D_{m+1} $ is the bound on the moment in iteration $ m+1 $.
    	Note that from \hyperref[sp_conc]{Theorem \ref*{sp_conc}}, we have $ C \leq n^2 $ occurs with probability at least $ \left( 1 - n^{-3} \right) $.
    	Thus, from the union bound we get that all of the random initial vectors drawn have the required $ C $ with probability $ 1 - n^{-2} $. As noted earlier we work conditioned on the event that all vectors drawn have sufficient Gaussian component. 
    	We will denote by $ \alpha_1  = \left( 5 A r! \right)^{-2r} $.
    	But we know from \hyperref[grad_des]{Theorem \ref*{grad_des}}, setting $ \eta \sim \left( 3L \right)^{-1}  $, that finding a direction with accuracy specified above takes 
    	\begin{equation*}
    	N_{m+1} = \frac{ 80L^2 n^4}{ \alpha_1^2  \delta^{4r } D_{m+1}^4}.
    	\end{equation*}
    	steps. 
    	Thus, the total number of steps of the algorithm is bounded by 
    	\begin{align*}
    	N \leq &  \sum_{i=1}^{n-1} N_i \\
    	\leq & \sum_{i=1}^{n-1} \frac{ 80L^2n^4}{  \alpha_1^2 \delta^{4r } D_{i+1}^4} \\
    	\leq & \frac{80L^2n^4  }{\alpha_1^2  \delta^{4r} } \sum_{i=1}^{n-1} \frac{1}{D_{i+1}^4}.
    	\intertext{ To get the final subspace (using \autoref{GS}) to within error $ \epsilon $ and for $ D_i  $ to be sufficiently large, we set $ \delta \leq \frac{\epsilon^4}{6^5r^{2r}n^8} $. That gives }
    	N \leq & \frac{6^{21r}L^2 r^{8r^2} n^{48r + 4 }}{\alpha_1^2\epsilon^{16r}} \sum_{i=1}^{n-1} \frac{1}{D_{i+1}^4} .
    	\end{align*}
    	For this setting of $ \delta $, we get $ D_{i} \geq D/n  $, thus giving
    	\begin{equation*}
    	N \leq \frac{6^{21r}L^2 r^{8r^2} n^{48r + 9 }}{\alpha_1^2\epsilon^{16r} D^4 }.
    	\end{equation*}
    	Recall that $ \alpha_1 = \left(5  A r!  \right)^{-2r} $ with $ A = 4r^2K \left( 3 + \log \left( \frac{K}{D}  \right) \right) $. Plugging this back into the equation, we get 
    	\begin{equation*}
    		N \leq \frac{6^{25r}L^2 K^{4r} r^{12r^2 + 8r} \left( 3 + \log\left( \frac{K}{D}  \right) \right)^{4r}   }{\epsilon^{16r}  D^4 }  n^{48r + 9 } .
    	\end{equation*}
    	as required.
    \end{proof}
    
     
    Note that, since from \autoref{Lip_Bound1}, we have a polynomial bound on the Lipschitz constant $ L $, we require polynomially many gradient descent steps in our algorithm.  
    Next, we bound the number of samples required. 
    We shall denote by $ T_{Ent} ( \chi, \beta ) $ the number of samples of the random variable required to estimate the entropy to within error $ \chi  $ and with probability of error $ \beta $.  
    
    \begin{lem} \label{Ent_Fin}
    	Let the random variable be as above. Then, in each gradient descent step, we need 
    	\begin{equation*}
    		T_{Ent} \left(   \frac{D^4  \epsilon^{16r}  }{  6^{26r}  r^{12r^2} A^{4r}  L}    n^{-20r - 9} , \frac{1}{Nn^2} \right)
    	\end{equation*}
    	samples of the random variable. 
    \end{lem}
	\begin{proof}
		In each step, we need to estimate the entropy and its derivative. We will bound the time required for this and the probabilities of error in this process. Note that for the gradient descent algorithm, we require that $ \nabla S $ is computed to within error $ 0.2 \cdot \norm{  \nabla S} $. We are always going to take a descent step only in the the case that $ \norm{\nabla S} \geq  \frac{ D_{m+1}^2    \delta ^{2r } }{n^2 \left( 5 A r! \right) ^{2r} } $ where $ \delta = \frac{\epsilon^4}{6^5r^{2r}n^8} $.  Thus, we need to estimate $ \nabla S $ to within error $\frac{ D_{m+1}^2    \delta ^{2r } }{5n^2 \left( 5 A r! \right) ^{2r} }$. From \hyperref[grad_est]{Theorem \ref*{grad_est}}, we are required to estimate $ S $ to within error $ \frac{ D_{m+1}^4    \delta ^{4r } }{100 n^5 \left( 5 A r! \right) ^{4r} L } $. 
		By definition, we require $ T_{Ent}  \left( \frac{ D_{m+1}^4    \delta ^{4r } }{100 n^5 \left( 5 A r! \right) ^{4r} L }, \beta  \right) $ to estimate the entropy to within probability of error $ \beta $. 
		We know that the number of estimates of the entropy required is $ N $ which was bounded in the previous theorem to be polynomial. Thus, using the union bound, to get an error $ o(n^{-1}) $, we can set $ \beta = N^{-1}n^{-2} $. 
		Thus, we require  $ T_{Ent}  \left( \frac{ D_{m+1}^4    \delta ^{4r } }{100 n^5 \left( 5 A r! \right) ^{4r} L } ,\frac{1}{Nn^2}  \right)   $ samples. 
        Setting the required value of $ \delta $ as $  \frac{\epsilon^4}{6^5 r^{2r} n^8} $ and noting that $ D_{m+1} \geq Dn^{-1} $ in this setting, we get that we require 
		\begin{equation*}
            T_{Ent} \left(   \frac{D^4  \epsilon^{16r}  }{  6^{26r}  r^{12r^2} A^{4r}  L}    n^{-20r - 9} , \frac{1}{Nn^2} \right)
        \end{equation*} 
        samples. 
	\end{proof}

	Finally, putting these together, we prove the main theorem which gives a polynomial time algorithm for isotropic NGCA. 
	
    \begin{theorem}[Main Theorem] \label{Main_Theorem}
        Let $\left( Z, \tilde X  \right) \in \Gamma \oplus \Gamma^{\perp}  = \br^n$ be a random variable satisfying the isotropic NGCA model. Suppose that $ \tilde X$ and a positive integer $r$ satisfy the following conditions
        \begin{itemize}
            \item  For each $ a \in \mathbb{S}^{n-1} $, there exists a $ k \leq r $ such that $  \left| \mathbb E \left\langle \tilde X , a \right\rangle ^k - \mathbb E \left\langle Z , a \right\rangle^k \right| > 2D $. In words, for each direction there is a $k$ such that the $ k $'th moment of non-Gaussian component along the direction differs from the $ k $'th moment of the Gaussian.
            \item $ \tilde X $ is $ K$-subgaussian.
        \end{itemize}
     With probability of error at most $ o(n^{-1}) $ Full Algorithm returns a subspace $ \Lambda $ 
     such that $ \dim\left( \Lambda \right) = \dim\left( \Gamma^{\perp}  \right) $ and 
     $ d\left(\Lambda , \Gamma^{\perp}  \right) \leq \epsilon$ using time
    \begin{equation*}
    	\frac{6^{100rcK^4 + 25r}K^{4r}  e^{\frac{2K^4}{t'^2}} r^{12r^2 + 8r + 32cr^2K^4} \left( 3 + \log\left( \frac{K}{D}  \right) \right)^{16rcK^4 + 4r}   }{\epsilon^{64rcK^4 + 16r}  D^{16cK^4+4} }  n^{128cK^4 + 37ck^4 + 48 r + 9 } .
    \end{equation*}
    \end{theorem}
\todo[inline, color =green]{what do I call this? Steps is confusing. Samples is true but we also need to say that the actual running time is poly. }
\begin{proof}
    Note that from \autoref{Lip_Bound1}, we have $ L \leq \mathsf{poly}\left(K,t^{-1}\right) $. 
    We use $ t'$ from \autoref{noise}, which is a constant independent of dimension.
    This changes the moment gap from  $ 2D $ to $ D $.  
    From \autoref{ent_est}, we get that $ T_{Ent} \left( \epsilon , \gamma  \right) $ = $ \log \left( 1/ \gamma \right)  K e^{\frac{K^4}{t'^2}}  \epsilon^{-cK^4} $. 
    Plugging this into \autoref{Ent_Fin}, we get that we need 
    \begin{equation*}
    	  K e^{\frac{K^4}{t'^2}}  \left(\frac{100^4n^{37}  \left(  5Ar! \right)^{16r} }{D^{16} \delta^{16r}} \right)^ {cK^4}
    \end{equation*}
    samples per gradient step. 
    Setting  $ \delta = \frac{\epsilon^4}{6^5r^{2r}n^8} $, we get 
    \begin{equation*}
    	\log \left( Nn^2 \right)  K e^{\frac{K^4}{t'^2}}  \left(\frac{100^4 6^{80r} r^{32r^2} n^{128r+37} \left(  5Ar! \right)^{16r} }{D^{16} \epsilon^{64r}} \right)^{cK^4}. 
    \end{equation*} 
     From \autoref{iter_comp}, we need $ N = \frac{6^{25r}L^2 K^{4r} r^{12r^2 + 8r} \left( 3 + \log\left( \frac{K}{D}  \right) \right)^{4r}   }{\epsilon^{16r}  D^4 }  n^{48r + 9 } $ steps. 
    Putting it all together we need time, 
    \begin{equation*}
    	\frac{6^{100rcK^4 + 25r}K^{4r}  e^{\frac{2K^4}{t'^2}} r^{12r^2 + 8r + 32cr^2K^4} \left( 3 + \log\left( \frac{K}{D}  \right) \right)^{16rcK^4 + 4r}   }{\epsilon^{64rcK^4 + 16r}  D^{16cK^4+4} }  n^{128cK^4 + 37ck^4 + 48 r + 9 }  .
    \end{equation*}
    where we absorb the polynomial in $ K $ and $( t')^{-1} $ terms into the exponential term. 
\end{proof}

    \section*{Acknowledgments}
    We would like to thank the anonymous reviewers of a previous version of this paper for useful remarks and for bringing the algorithm sketched in \autoref{sec:Cumulant} to our attention. 
    
	\newpage
    
	\bibliographystyle{alpha}
	\bibliography{entropy}
	
    \newpage 
	\appendix
	\section{Cumulant Based Approaches to NGCA} \label{sec:Cumulant}
Since the NGCA problem mainly depends on constructing a method for distinguishing a Gaussian random variable from a random variable far from Gaussian, there are several approaches to the problem depending on the choice of the ``Gaussian test". 
For example, \cite{tan2017polynomial} relies on characterizing radial and angular distributions of a Gaussian vector and \cite{vempala2011structure} looks at the moments of the Gaussian vector. 
One characterization of the Gaussian random variable, related to the above characterizations, is in terms of its cumulants \cite{Cumulant}. 

\begin{defi}
    Let $ X $ be a real random vector. Define the cumulant generating function of $X$ as the logarithm of the moment generating function, i.e. 
    \begin{equation*}
    K_X(t) = \log\left( \mathbb{E}   \left[ e^{ \ip{t}{X}  }   \right]   \right) = \sum_{1\leq i} \kappa^{(i)}(X) \frac{t^i}{i!}. 
    \end{equation*}
\end{defi}
Here $\kappa^{(i)}(X)$ is the order-$i$ cumulant of $X$.
Thus one can check that the first cumulant is the expectation, the second cumulant is the variance, and cumulant of order $i$ can be
expressed as a polynomial in moments of order up to $i$. 
From this definition it follows easily that a cumulant of the sum of independent random variables is the sum of the cumulant: 
$\kappa^{(i)}(X+Y)= \kappa^{(i)}(X)+\kappa^{(i)}(Y)$ for independent $X$ and $Y$. 

It is easy to check that the cumulant generating function of the scalar random variable $ Z \sim \mathcal{N}\left(\mu , \sigma ^2 \right) $ is $ \mu t + \frac{\sigma^2t^2}{2} $. 
One characterizing property of the Gaussian random variables is that they are the unique random variables for which the cumulant generating function is a polynomial, that is, for every non-Gaussian random variable some cumulant of order greater than three does not vanish. As just mentioned, this polynomial is in fact quadratic. 
This property of the Gaussian can be used to provide a Gaussianity test, simply by estimating the cumulants (which can be done through estimating the moments) and then checking to see if they vanish. 
This suggests using the cumulant generating function as a contrast function as follows. Let $X \in \br^n$ be the input to the NGCA problem as described earlier and let $u \in \br^n$ be a unit vector. Fixing $t$ to some non-zero value,  
$K_{\langle u, X \rangle}(t)$ could act as a contrast function in direction $u$: If $u$ is in the Gaussian subspace then 
$K_{\langle u, X \rangle}(t) = t^2/2$ (recall that $X$ is in isotropic position). Whereas, if $u$ is in a non-Gaussian direction
then presumably $K_{\langle u, X \rangle}(t)$ has a different value in general because as noted above the Gaussian distribution is the unique one for which the cumulant generating function is polynomial. We do not know of a direct effective realization of this idea. 

One could consider truncation of the cumulant generating function to order $k \geq 3$ as a potentially more tractable surrogate: 
\begin{align*}
K_X^{(k)}(t) =  \sum_{1 \leq i \leq k} \kappa^{(i)}(X) \frac{t^i}{i!}. 
\end{align*}

Truncated function also has the property of contrast functions and presumably an algorithm in the style of \cite{vempala2011structure}
could be used here combined with our error analysis (which may be possible if one can guarantee finding Gaussian directions approximately) to give an efficient algorithm.

The following expression in terms of joint cumulants of $X$ (again analogues of multivariate moments) is standard:

\begin{align*}
\kappa^{(k)}(\langle u, X\rangle) = \sum_{j_1, \ldots, j_k} \kappa^{(k)}_{j_1, \ldots, j_k}(X) u_{j_1}\ldots u_{j_k}.
\end{align*}
Thus $\kappa^{(k)}(\langle u, X\rangle)$ can be written in terms of a tensor $T^{(k)}_X(u, \ldots, u)$.

We now describe a different way to realize the property of cumulants into an algorithm; it was suggested by an anonymous reviewer of a previous version of this paper. We assume that for every direction $u$ 
in the non-Gaussian component for some $i \leq r$ we have $\abs{\kappa_i(\langle u, X \rangle)-\kappa_i(Z)} = \abs{\kappa_i(\langle u, X \rangle)} \geq D$. This condition
is similar to the moment distance condition that we use and one implies the other with some loss in the distance. 
For describing the algorithm it will be useful to assume that the Gaussian and non-Gaussian components
are spanned by $e_1, \ldots, e_p$ and $e_{p+1}, \ldots, e_n$ respectively. For $ u \in \br^n$, let $u_\Gamma$ be the vector $(u_1, \ldots, u_p)^T$ and $u_{\Gamma^\perp}$ be the vector $(u_{p+1}, \ldots, u_n)^T$. Then for $i \geq 3$ we have (recall our notation
$X = Z + \tilde{X}$)
\begin{align*}
\kappa^{(k)}(\langle u, X \rangle) &= \kappa^{(k)}(\langle u_\Gamma, Z\rangle + \langle u_{\Gamma^\perp}, \tilde{X}\rangle) \\
&= \kappa^{(k)}(\langle u_\Gamma, Z\rangle) + \kappa^{(k)}(\langle u_{\Gamma^\perp}, \tilde{X}\rangle) \\
&= \kappa^{(k)}(\langle u_{\Gamma^\perp}, \tilde{X}\rangle).
\end{align*}

This implies that $\kappa^{(k)}_{j_1, \ldots, j_k}(X) = 0$ if $\{j_1, \ldots, j_k\}$ intersects with $\{1, \ldots, p\}$.
Hence the linear map $M^{(k)}: u \mapsto T^{(k)}_X(\cdot, \ldots, \cdot, u)$ has the property that 
$\norm{M^{(k)}u} \geq T^{(k)}_X(\cdot, \ldots, \cdot, u) = \kappa^{(k)}(\langle u, X\rangle)$ for unit vector $u$ in the non-Gaussian 
subspace, and $\norm{M^{(k)}u} = 0$ for $u$ in the Gaussian subspace (here we treat all arguments of norms as vectors and the norm is the 2-norm).
This implies that the concatenated linear map 
$M^{(3, \ldots, k)}: u \mapsto [T^{(3)}_X(\cdot, \ldots, \cdot, u), \ldots, T^{(k)}_X(\cdot, \ldots, \cdot, u)]$
has the property that $\norm{M^{(3, \ldots, k)}(u)} \geq D$ for unit vector $u$ in the non-Gaussian subspace: this is because of our assumption $\abs{\kappa_i(\langle u, X \rangle)} \geq D$. On the other hand 
$\norm{M^{(3, \ldots, k)}(u)} = 0$ for $u$ in the Gaussian subspace. 
Thus, the right kernel of the map $M^{(3, \ldots, k)}$ corresponds to the Gaussian subspace. Concretely,
$M^{(3, \ldots, k)}(u)$ is a concatenation of the flattening of the tensors $T^{(i)}$ of dimensions 
$(n^2+ \ldots + n^{r-1}) \times n$. 

To analyze this algorithm one would need to estimate all the entries of $M = M^{(3, \ldots, k)}(u)$, and then analyze the error in computing the singular subspace corresponding to non-zero singular values. 
 As stated above, the algorithm needs space $\Omega(n^r)$ because we need to work with a matrix of size 
$\Omega(n^{r-1}) \times n$. We observe that this is easily rectified by 
computing $M^TM$ which is an $n \times n$ matrix and using this new matrix for computing the kernel.

\section{Moment Equalties}
    \begin{lem} \label{lem:moments}
        Let $X$ be a real-valued random variable. Let $ W_t := tX + \sqrt{1 - t^2} Z $ where $ Z $ is the standard Gaussian random variable independent of $ X $ and let $ \rho_t $ be the density of $W_t$. Then, we have 
        \begin{equation*}
            M_k\left( \rho_t - G \right) = \sum_{j=1}^{k} \binom{k}{j} t^j M_j\left( \rho - G \right) \left( \sqrt{1-t^2}   \right) ^{k-j} M_{k-j} \left( Z \right) .
        \end{equation*}
    \end{lem}
    \begin{proof}
        Let $Z_1 := tZ_2 + \sqrt{1 - t^2} Z_3$ for independent standard Gaussians $Z_2, Z_3$. Then we have
 
        \begin{align*}           
            M_k \left( \rho_t - G  \right) & =  \mathbb{E} \left[ \left( tX + \sqrt{1 - t^2}Z  \right) ^k \right]  - \mathbb{E} \left[ \left( tZ_2 + \sqrt{1 - t^2}Z_3  \right)^k \right] \\
            & = \sum_{j=0}^{k} \binom{k}{j} \mathbb{E}\left[ \left(tX  \right)^j  \right]    \mathbb{E} \left[ \left( \sqrt{1 - t^2}Z  \right) ^{k - j }    \right] -  \sum_{j=0}^{k} \binom{k}{j} \mathbb{E}\left[ \left(tZ_2  \right)^j  \right]    \mathbb{E} \left[ \left( \sqrt{1 - t^2}Z_3  \right) ^{k - j }    \right] \\
            & =  \sum_{j=1}^{k} \binom{k}{j} t^j \; \mathbb{E}\left[ X ^j - Z_2^k  \right] \left( \sqrt{1-t^2}   \right) ^{k-j} \mathbb{E} \left[ Z_3^{k-j}  \right] .
        \end{align*}
        The last expression is equal to the right hand side in the statement of the lemma.
    \end{proof}

\section{Proof of Perturbation Bound} \label{App:Subspace}
    \begin{lem}[Restatement of \autoref{GS}]
    Let $ 0< \epsilon < \frac{1}{25n^2}$ and let $ \lambda_1, \dots, \lambda_k \in \br^n $ be orthonormal vectors and $ \gamma_1, \dots, \gamma_k \in \br^n  $ be unit vectors and 
    satisfying the following condition for all $i$
    \begin{equation*}
    \ip{\lambda_i}{\gamma_i} \geq 1 - \epsilon.
    \end{equation*}
    Denote by $ \Lambda_k$ and $ \Gamma_k$ their respective spans. 
    Then, $ \Lambda_{k} $ and $ \Gamma_{k} $ have the same rank and $ d\left( \Lambda_k, \Gamma_k  \right) \leq 6 k^2 \sqrt[4]{\epsilon}$.  
\end{lem}

\begin{proof}
	Recall that 
	$ d\left( \Lambda_k, \Gamma_k  \right)$ is given by the Frobenius norm of the difference between the orthogonal projections onto $\Lambda_k$ and $\Gamma_k$, whenever the ranks of $ \Gamma_k $ and $ \Lambda_k $ match.
     We shall show that the ranks match later in the proof.  
     However, $\gamma_1, \ldots, \gamma_k$ are not orthonormal and so constructing an orthogonal projection
	to $\Gamma_k$ requires some more work. 
    We will do this by finding an orthonormal basis $\gamma_1', \ldots, \gamma_k'$ of $\Gamma_k$, and 
	using those to specify the orthogonal projection to $\Gamma_k$. 
    We then bound the subspace distance using the Frobenius norm. 
	
	Denote by $ \gamma'_{1} \dots \gamma'_{k} $ the Gram-Schmidt orthonormalization of the $ \gamma_1 \dots \gamma_k $. That is, we inductively define $ \gamma_j' $ as follows, 
	\begin{align*}
	\gamma'_1 :=& \gamma_1, \\
	\ldots\\
	\gamma'_j := & \frac{\gamma_j - \sum_{i=1}^{j-1} \ip{\gamma_j}{\gamma'_i}  \gamma'_i  }{\norm{\gamma_j - \sum_{i=1}^{j-1} \ip{\gamma_j}{\gamma'_i} \gamma'_i}}.
	\end{align*}
	We write down a determinantal formula for the orthonomalized vectors (see, e.g., \cite{gantmakher1998theory}). The determinant below has vector entries in the last row, which we interpret by expanding it using the cofactor expansion along the last row resulting in a linear combination of $ \gamma_1, \dots ,\gamma_j $. 
	\begin{align}  \label{eqn:GS_determinant} 
	\gamma'_j = \frac{1}{\sqrt{D_j D_{j-1}}} \begin{vmatrix}
	\ip{\gamma_1}{\gamma_1} & \ip{\gamma_1}{\gamma_2} & \cdots & \ip{\gamma_1}{\gamma_j} \\
	\ip{\gamma_2}{\gamma_1} & \ip{\gamma_2}{\gamma_2} & \cdots & \ip{\gamma_2}{\gamma_j} \\
	\vdots & \vdots & \ddots & \vdots \\
	\ip{\gamma_{j-1}}{\gamma_1} & \ip{\gamma_{j-1}}{\gamma_2} & \cdots & \ip{\gamma_{j-1}}{\gamma_j} \\
	\gamma_1 & \gamma_2 &\cdots& \gamma_j
	\end{vmatrix},
	\end{align}
	where $ D_j := \dt\left( \left[{\ip{\gamma_i}{ \gamma_k}  } \right]_{ i,k \leq j  }  \right) $ is the Gram determinant. We now prove upper and lower bounds on $D_j$ by showing that it is highly diagonally dominant. 
	\begin{claim}
		The determinant $D_j$ satisfies
		\begin{align}
		\left( 1 - 5j \sqrt{\epsilon} \right) ^{j} \leq  D_j \leq 1.
		\end{align}
	\end{claim}
\begin{proof}
	Note that the diagonal entries of the matrix are one. We next bound the off-diagonal entries of the matrix. 
	By our assumption $\ip{\lambda_i}{\gamma_i} \geq 1 - \epsilon$. For $\mu_i := \gamma_i-\lambda_i$, we have 
	$ \norm{ \mu_i  } = \sqrt{ \ip{\lambda_i}{\lambda_i} + \ip{\gamma_i}{\gamma_i} - 2 \ip{\gamma_i}{\lambda_i}  } \leq \sqrt{2\epsilon} $.
	Thus, we can write 
	\begin{align} 
	\abs{\ip{\gamma_i}{\gamma_j}} &= \abs{\ip{\lambda_i + \mu_i}{\lambda_j + \mu_j}} \nonumber \\
	&\leq  \abs{\ip{\lambda_i}{\lambda_j}} + \abs{\ip{\lambda_i}{\mu_j}} + \abs{\ip{\mu_j}{\lambda_i}} + \abs{\ip{\mu_i}{\mu_j}} \nonumber\\
	& \leq \norm{ \mu_j } + \norm{ \mu_i  } + \norm{\mu_i} \norm{\mu_j}  \;\;\;\; \text{(by the orthonormality of the $\lambda_i$'s and Cauchy--Schwarz)} \nonumber\\
	& \leq 2\sqrt{2 \epsilon}  + 2 \epsilon \nonumber \\
	& \leq 5 \sqrt{\epsilon}, \nonumber
	\end{align}
	where the last inequality used that $\epsilon \leq \sqrt{\epsilon}$. 
	Thus all the off-diagonal entries of the Gram matrix corresponding 	to $D_j$ have absolute value bounded above by $ 5 \sqrt{\epsilon} $.
     For $ \epsilon <  1/{25n^2} $, the Gram matrix is strictly diagonally dominant. 
	\begin{defi}
		Let $ A $ be a real $ n \times n $ matrix. $ A $ is said to be diagonally dominant if for all $ i $ we have $ \left| A_{i,i}  \right| \geq \sum_{i \neq j } \left| A_{i,j} \right|  $. It is said to be strictly diagonally dominant if the above inequality is strict for each $ i $. 
	\end{defi}
	The following well known facts regarding determinants combined with the above claim about diagonal dominance of our Gram matrix will give us bounds for $D_j$.
	\begin{fact}[Lower bound on the determinant of diagonally dominant matrices \cite{ostrowski1952note}]\label{sd} 
		Let $A$ be a strictly diagonally dominant symmetric matrix with non-negative diagonal entries, then it is positive definite and 
		\begin{equation*}
		\dt(A) \geq \prod_{j = 1}^{n} \left( a_{ij} - \sum_{i: i \neq j } \abs{a_{i,j}}  \right).
		\end{equation*}
	\end{fact} 
	\begin{fact}[Hadamard's inequality]\label{Had}
		Let $A$ be a positive semi-definite matrix, then we have 
		\begin{equation*}
		\dt(A)  \leq \prod_{j = 1}^n A_{i,i}.
		\end{equation*}
	\end{fact}
	Now \hyperref[Had]{Fact \ref*{Had}} gives 
	\begin{align}\label{eqn:Hadamard_application}
	D_j \leq 1,
	\end{align}
	and \hyperref[sd]{Fact \ref*{sd}} with the fact that $    \sqrt{\epsilon}  < (5n)^{-1} $ gives 
	\begin{align}\label{eqn:Ostrowski_application}
	D_{j} \geq \left( 1 - 5j \sqrt{\epsilon} \right) ^{j}.  
	\end{align}
\end{proof}	
	
    Also, note that this also shows that $ \Gamma_k $ has the same rank as $ \Lambda_k $, since the Gram determinant is non-zero. 
    We note this below in the following claim. 
    
    \begin{claim}
        Let $ \Gamma_k $ and $ \Lambda_k $ be as above. Then, dimension of these two subspaces match. 
    \end{claim} 
    
	We write $ \gamma'_j = \sum_{i \leq j } \alpha^j_{i} \gamma_i $ where $ \alpha^j_i $ come from the above determinantal formula. We denote by $ M $ and $ M' $ the matrices with rows $ \gamma_i^T $ and $ (\gamma'_i)^T $ respectively, and denote by $ A $ a lower triangular matrix with entries $A_{j,i} = \alpha_{i}^j $. We can rewrite the above equation as a matrix equation, 
	\begin{equation*}
	M' = AM.
	\end{equation*}
	Noting that $ A $ is invertible, we can also write this as
	\begin{equation*}
	M = A^{-1} M' .
	\end{equation*}
	From this equation, we note that, since the $\gamma_i'$'s are orthonormal, we get 
	\begin{align*}
	M' M^T = M' (A^{-1} M')^T = M'M'^T (A^{-1})^T = (A^{-1})^T.
	\end{align*}
	Equating the diagonal entries we obtain
	\begin{equation*}
	\ip{\gamma_j}{\gamma'_j} = \left( A^{-1} \right)_{j,j}.
	\end{equation*}
	Since $ A $ is a lower triangular matrix, Cramer's rule gives
	\begin{align*}
	(A^{-1})_{j,j} = & \frac{1}{A_{j,j}}.
	\end{align*}
	Noting the coefficient of $ \gamma_j $ in the determinantal expansion for $ \gamma'_j $ in \eqref{eqn:GS_determinant} we obtain
	\begin{align} \label{eqn:gamma_gamma_prime}
	\ip{\gamma_j}{\gamma'_j} = \left(A_{j,j}\right)^{-1} =  \sqrt{\frac{D_{j-1}} {D_j}} 
	\geq  \left( 1 - 5j \sqrt{ \epsilon} \right)^{j/2}.
	\end{align}
	Let $M_{\lambda,k} := \sum_{i=1}^k   \lambda_i \lambda_i^{T}, M_{\gamma, k} := \sum_{i=1}^k   \gamma_i \gamma_i^{T}$, 
	and $M_{\gamma', k} := \sum_{i=1}^k   \gamma_i' \gamma_i'^{T}$. Note that $M_{\lambda,k}$ is the orthogonal projection on $\Lambda_k$
	and $M_{\gamma', k}$ is orthogonal projection on $\Gamma_k$, but $M_{\gamma, k}$ does not have such an interpretation as the $\gamma_i$'s need not be orthonormal. 
	By the definition of subspace distance and the orthonormality of the $\lambda_i$'s and of the $\gamma_i'$'s, we have 
	$ d(\Lambda_k, \Gamma_k) = \norm{M_{\lambda,k} - M_{\gamma', k}}_{F}  $. By the triangle inequality we obtain
	\begin{align*}
	d(\Lambda_k, \Gamma_k) = & \norm{ M_{\lambda,k} - M_{\gamma', k} }_{F} \\
	\leq & \norm{ M_{\lambda,k} - M_{\gamma, k} }_{F} + \norm{ M_{\gamma,k} - M_{\gamma', k} }_{F}.
	\end{align*}
	We bound the two terms in the R.H.S. separately but in similar ways using the triangle inequality and the bounds we proved earlier.  
	\begin{align*}
	\norm{ M_{\lambda,k} - M_{\gamma, k} }_{F} =& \norm{ \sum_{i=1}^k   \lambda_i \lambda_i^{T}  - \sum_{i=1}^k \gamma_i\gamma_i^{T}  }_{F} \\
	\leq & \sum_{i=1}^k \norm{  \lambda_i \lambda_i^{T}  -  \gamma_i\gamma_i^{T}  }_{F} \\
	= & \sum_{i=1}^k \sqrt{   \tr \left(   \left(\lambda_i \lambda_i^{T}  -  \gamma_i\gamma_i^{T} \right)^T \left(\lambda_i \lambda_i^{T}  -  \gamma_i\gamma_i^{T} \right)  \right)    } \\
	= &  \sum_{i=1}^k \sqrt{ 2 - 2\ip{\lambda_i}{\gamma_i} ^2   }  \\
    \leq & \sum_{i=1}^k \sqrt{ 2 - 2  \left( 1 - \epsilon \right) ^2   } \\ 
    \leq & k \sqrt{  4 \epsilon - \epsilon^2} \\
	\leq & 2k\sqrt{  \epsilon}.  
	\end{align*}
	
	Now for the second term:
	\begin{align*}
	\norm{ M_{\gamma,k} - M_{\gamma', k} }_{F} \leq & \sum_{i=1}^k \sqrt{   \tr \left(   \left(\gamma'_i \gamma_i^{'T}  -  \gamma_i\gamma_i^{T} \right)^T \left(\gamma'_i \gamma_i^{'T}  -  \gamma_i\gamma_i^{T} \right)  \right)    } \\
	\leq & \sum_{i=1}^k \sqrt{2 - 2 \ip{\gamma_i}{\gamma'_i}^2} \\
	\leq & \sum_{i=1}^{k} \sqrt{  2 - 2 \left( 1 - 5i  \sqrt{\epsilon} \right)^{i} }  \;\;\;\; \text{(using \eqref{eqn:gamma_gamma_prime})}\\
	\leq & \, k \sqrt{  2 - 2 \left( 1 - 5k \sqrt{\epsilon} \right)^{k} } \\
	\leq & \, k \sqrt{  10 k^2 \epsilon     } \\
	\leq & \, k^2\sqrt{ 10   \sqrt{\epsilon} } \\
    \leq & \, 4 k^2 \epsilon^{ \frac{1}{4} }      .
	\end{align*}
	
	Putting together the two bounds we get $ d(\Lambda_k, \Gamma_k) \leq 4k^2 \sqrt[4]{\epsilon}  + 2 k\sqrt{ \epsilon} 
	< 6 k^2 \sqrt[4]{\epsilon}.$
\end{proof}

\section{Upper and Lower Bounds on Density after Smoothening}
	To lower bound the density, we use a standard analysis similar to the one present in \cite{carlen1991entropy}. 
	\begin{theorem}[Lower Bound] \label{den_low_bound}
		Let $ X $ be a real-valued random variable with variance one and mean zero, and let $ Y_t = \sqrt{1-t^2}X + tZ $ where $ Z $ is standard univariate Gaussian independent of $X$. 
		Then for $ t < 1  $, $ Y_t $ has density $ \rho $ such that 
		\begin{equation*}
			\rho ( x ) \geq C(t) e^{-\frac{x ^2 }{ t^2}},
		\end{equation*}
		where $ C(t) = \frac{ \left( 1 - \frac{1}{  2\left(1-t^2\right) }\right) e^{-\frac{2}{t^2}}  }{\sqrt{2 \pi }t}  $. 
	\end{theorem}
	\begin{proof}
		The fact that $ Y_t $ has density is easy to see and follows by differentiating the convolution expression for its cumulative distribution function, using the rapid decay of the Gaussian to justify the exchange of limits.  
		Let $ F_t $ be the cumulative distribution function of $ \sqrt{1-t^2}X $.
		The density $ \rho(\cdot)$ of $Y_t$ has the following form using the convolution formula for the density of the sum of independent random variables. 
		\begin{align*}
			\rho(x) &= \int \frac{1}{\sqrt{2 \pi  } t } e^{-\frac{(x-y) ^2 }{2t^2} } dF_t(y). \\
			\intertext{Let $  A_1 > 0$ be a parameter to be set later. We shall truncate the integral at $ A_1 $.  }
			\rho(x) & \geq \int_{- A_1}^{A_1}  \frac{1}{\sqrt{2 \pi  }t } e^{-\frac{(x-y) ^2 }{2  t^2 } } dF_t(y) \\
			& \geq \frac{1}{\sqrt{2 \pi  }t } \min\left\{e^{-\frac{(x-y) ^2 }{2  t^2} } : |y| \leq A_1  \right\} \Pr\left[ |\sqrt{1-t^2}X| \leq  A_1   \right] \\
			& \geq \frac{\Pr\left[ |\sqrt{1-t^2}X| \leq  A_1   \right]}{\sqrt{2 \pi }t } e^{-\frac{(\abs{x}+ A_1 ) ^2 }{2 t^2} }. \\
			\intertext{We use the fact that $(a+b)^2 \leq 2a^2 + 2b^2 $ with the fact $ e^{-x^2} $ is decreasing for positive $ x $.}
			\rho \left( x \right) & \geq \frac{\Pr\left[ |\sqrt{1-t^2}X| \leq  A_1   \right]}{\sqrt{2 \pi }t } e^{-\frac{x ^2 }{ t^2}} e^{-\frac{A_1^2 }{  t^2}} .
			\intertext{From Chebyshev inequality, it follows that,}
			\rho(x)& \geq \frac{\left( 1 - \frac{1}{A_1^2 (1-t^2)}  \right)}{\sqrt{2 \pi } t } e^{-\frac{x ^2 }{t^2}} e^{-\frac{A_1^2 }{ t^2}}. \\
			\intertext{We set $A_1^2 = 2$ }. 
			\rho\left(x\right)& \geq  \frac{ \left( 1 - \frac{1}{  2\left(1-t^2\right) }\right) e^{-\frac{2}{t^2}}  }{\sqrt{2 \pi }t}   e^{-\frac{x ^2 }{ t^2}}
			\intertext{We set the leading term as $ C(t)  $. We get the following }
			\rho(x) & \geq C(t) e^{-\frac{x ^2 }{ t^2}}, 
		\end{align*}
		as required. 
		\todo[inline]{How did you go from $(\abs{x}+A)^2$ to $2x^2$ and $A^2$? Presumably you are using $(x+y)^2 \leq 2 x^2 + 2 y^2$. If so, you'll get different constant factors. I think the application of $K$-subgaussianity gives a factor of 2 in front of e but no 2 in the exponent. Please say what $C$ is both in the proof and in the statement of the theorem. $C$ has a dependence on $t$ and this needs care to make sure that $C$ doesn't blow up for some choice of $t$. It's misleading to call it $C$ unless you can set $A$ in a way to eliminate dependence on $t$. $100 Kt$ seems to large as it would make the last factor very small. Also, in entropy estimation, right before discretization error you make a statement about $\log C$ which seems to have an error: the upper bound should be $Kt^2/(1-t^2) + \log C_2$. Again $C_2$ appears without prior announcement and the reader has to figure out by looking at the proof. In general, try to not invoke the proof, only the statements of lemmas and theorems (so here it means that $C_2$ and the dependence on $t$ should be made explicit in the statement of B1 itself.). Only very rarely is it required to invoke the proof.}
	\end{proof}
	
	\begin{theorem}[Upper Bound] \label{den_upper_bound}
			Let $ X $ be a $K$-subgaussian random variable and let $ Y_t = \sqrt{1 - t^2 }X + tZ $ where $ Z $ is standard Gaussian. Then, $ Y_t $ has density $ \rho $ such that 
		\begin{equation*}
		\rho ( x ) \leq  \frac{3}{\sqrt{2\pi } t} e^{-\frac{x ^2 }{8K^2} }.
		\end{equation*} 
	\end{theorem}
	\begin{proof}
		As in the proof of the previous theorem $\rho(x)$ can be expressed by convolution,
		\begin{align*}
		  \rho(x) &= \int_{\br} \frac{1}{\sqrt{2 \pi  } t } e^{-\frac{(x-y) ^2 }{2 t^2} } dF_t(y) \\
		  & = \int_{\abs{y} \leq x/2  } \frac{1}{\sqrt{2 \pi  }t^2 } e^{-\frac{(x-y) ^2 }{2 t^2} } dF_t(y) + \int_{\abs{y} \geq x/2  } \frac{1}{\sqrt{2 \pi  } t } e^{-\frac{(x-y) ^2 }{2t^2} } dF_t(y). \\
		  \intertext{We use the normalization of the measure and the subgaussianity to get,}
		  \rho\left(x\right)& \leq  \frac{1}{\sqrt{2 \pi  } t }  e^{-\frac{x ^2 }{8 t^2} } +  \frac{2}{\sqrt{2 \pi }t } e^{-\frac{x ^2 }{8K^2} } \\
		  & = \frac{1}{\sqrt{2 \pi  }t } \left( e^{-\frac{x ^2 }{8 t^2} } +  2e^{-\frac{x ^2 }{8K^2} }  \right). \\
		  \intertext{Note that $ K \geq 1 \geq t $.  }
		  \rho\left( x\right) & \leq  \frac{3}{\sqrt{2\pi  }t } e^{-\frac{x ^2 }{8K^2} }, \\ 
		\end{align*}
		as required.
		\todo[inline]{The two integrals above need absolute values for $y$ (this will change the constant factor in the exponent in the upper bound on the first integral). There's again an error in where constant factor of 2 occurs in the application of subgaussianity. The last step seems incorrect as the exponential dependence on $t$ disappeared. Check downstream applications of	this theorem to make sure that it doesn't cause serious issues.}
		\todo[inline,color=green]{I have added the absolute values to the $y$ and have made the edit regarding the factor of 2. \\
		Regarding the exponential dependence on t, since it is a small term we can essentially absorb it into the next term. }
	\end{proof}
	
	\begin{theorem}[Upper Bound on Derivative] \label{den_der_bound}
		Let $ X $ be a $K$-subgaussian random variable and let $ Y_t = \sqrt{1 - t^2 }X + tZ $ where $ Z $ is standard Gaussian. Then, $ Y_t $ has density $ \rho $ such that 
		\begin{equation*}
		\abs{\rho' ( x ) }\leq  \frac{ e^{ - \frac{x^2}{8K^2 } } }{ t^2} \left[ \frac{\abs{x}}{\sqrt{2\pi t}}  + 2   \right].	
		\end{equation*} 
	\end{theorem}
	\begin{proof}
		The proof follows the same idea as the previous proof. 
		We take advantage of the behavior of the Gaussian density under differentiation. 
		We integrate in two intervals using the maximum value of the function in the second one. 
		\begin{align*}
			\abs{\partial_{x} \rho(x)}  =&  \abs{ \partial_{x}  \int \frac{1}{\sqrt{2 \pi  }t } e^{-\frac{(x-y) ^2 }{2t^2} } dF_t(y) }. \\
			\intertext{Differentiating the Gaussian term under the integral sign, we get,}
			\abs{\partial_{x} \rho(x)}=& \abs{ \frac{1}{t^3\sqrt{2 \pi  } } \int (x - y) e^{-\frac{(x-y) ^2 }{2 t^2} } dF_t(y) }  \\
			=& \abs{   \frac{1}{t^3\sqrt{2 \pi } } \int_{\abs{y} \leq \abs{x} /2 } (x - y) e^{-\frac{(x-y) ^2 }{2 t^2} } dF_t(y )+ \frac{1}{t^3\sqrt{2 \pi } } \int_{\abs{y} \geq \abs{x}/2} (x - y) e^{-\frac{(x-y) ^2 }{2 t^2} } dF_t(y)} .\\
			\intertext{Using the maximum value of the second integrand, }
			\abs{\partial_{x} \rho(x)}  \leq & \abs{  \frac{ \abs{x} e^{-\frac{x ^2 }{8t^2} }}{t^3\sqrt{2 \pi } } } + \abs {  \frac{\sqrt{  e} t}{t^3\sqrt{2 \pi } }    \Pr \left[ t|X| \geq  \abs{x}/2  \right] } \\
			\leq & \frac{  \abs{x} e^{-\frac{x ^2 }{8 t^2} }}{t^3\sqrt{2 \pi  } }  + \frac{1}{t^2}   \Pr \left[ t|X| \geq  x/2  \right] \\ 
			\leq & \frac{\abs{x} e^{-\frac{x ^2 }{8 t^2} }}{t^3\sqrt{2\pi  }  }  + \frac{2e^{ - \frac{x^2}{8K^2 }  }}{t^2 } .  \\ 
			\intertext{Again using the fact that $ K > 1 > t $, }
			\abs{\partial_{x} \rho(x)}  \leq & \frac{ e^{ - \frac{x^2}{8K^2 } } }{t^2 } \left[ \frac{\abs{x}}{\sqrt{2\pi } t }  + 2   \right],			
		\end{align*}
		as required. 
	\end{proof}

	\section{Entropy Estimation from First Principles} \label{App:Ent_Est}
	 We will give a simple entropy estimation procedure	using histograms to approximate the density and using that estimates the entropy. 
	We use the fact that the addition of Gaussian noise gives us upper and lower bounds on the density.
	We also note that the addition of Gaussian noise gives us control over the derivatives of the density.  
	We do not attempt to optimize the complexity of our procedure. 
	
	\begin{theorem}
		Let $ f $ be the density of the random variable $   Y_t = \sqrt{1-t^2}X + tZ $ where $ Z $ is standard univariate Gaussian independent of $X$, which is $ K $-subgaussian, for constant $ t $. Then, our entropy estimation algorithm below estimates the entropy of $ f $ with error $ \epsilon $ and probability of correctness $ \gamma $ that uses $ N = \log( 1/ \gamma ) K^2 e^{ \frac{K^4}{t^2}}  \epsilon^{ - c K^4 }   $ samples for some constant $ c $.
	\end{theorem}
	\begin{proof}
	We need to estimate the relative entropy of the distribution with respect to the Gaussian but since we are looking at distribution of unit variance, we can just concentrate on the estimation of $ \int (f \log f ) \, dx  $. 
	We use a histogram based estimator for density which we then use to estimate the entropy.  
	We first truncate the density at level $ A $, i.e. we integrate only up to $ \abs{x} \leq A $. 
	Next, we use a finite Riemann sum as an approximation for the integral.
	That is, we divide the interval into ``buckets" $ B_i $ of size $ B $ and estimate the integral as the sum of values in each bucket. 
	To estimate the density in the bucket, we just use a simple counting estimator, which just counts the fraction of the total number of samples the fall into the bucket. 
	We will look at the various errors in the estimation separately.
	
	\begin{algorithm}[H] 
		\caption{Entropy Estimation}
		\begin{algorithmic}[1]
			\Require{Random Variable $ W $, Parameters $ A  $, $ B$, $ N $}
			\Procedure{EstEnt}{$W, A, B $}
			\State Set up buckets $ B_i $ of width $ B $ on the interval $ \left[-A,A\right] $. Let $ N_i $ be the counters of the buckets.
			\For{$ \tau \leq N $}
			\State Draw $ W_{\tau} $ from the distribution of $ W $.
			\State Update $ N_i \leftarrow  N_i +  1\left( W_{\tau} \in B_i  \right) $ where $ 1 $ is indicator function.  
			\EndFor
			\EndProcedure
			\Ensure{$ \sum_i \frac{N_i  }{N}\log \frac{N_i}{NB} $.}
		\end{algorithmic}
	\end{algorithm}

	\todo[inline]{Say what the procedure is. Introduce the bucket size $B$, say what is the sum we compute. Later on
	$B$ occurs in the proof without having been introduced.}
	
	\paragraph{Truncation Error:}
	As a first step, we shall truncate the integral to a finite interval $[-A,A]$. This leads to an error of the form, 
	\todo[inline]{What is $C$ below? Also it seems that $\log C$ is missing parentheses: $(\log C)$.}
	
	\begin{align}
		\abs{ \int_{\abs{x} \geq A }    f(x) \log f(x) dx   } & \leq \int_{\abs{x} \geq A }    f(x) \abs{\log f(x)} dx \\
		& \leq \int_{\abs{x} \geq A } \frac{3}{\sqrt{2\pi  }t } e^{-\frac{x ^2 }{8K^2} } |\log f(x)| dx    \qquad  ( \text{From  \hyperref[den_upper_bound]{\ref*{den_upper_bound}} } ) \\
		\intertext{We need to take two cases : one for $ f \leq 1 $ and $ f \geq 1 $. But, again since $ K \geq  1 $ and $ \abs{ \log C} \geq 1 $, we can just use the lower the bound from the \autoref{den_low_bound} to cover both cases. }
		\abs{ \int_{\abs{x} \geq A }    f(x) \log f(x) dx   } & \leq \frac{3}{\sqrt{2\pi  }t } \left[ \int_{\abs{x} \geq A } \abs{\log C}e^{-\frac{x ^2 }{8K^2} } dx + \int_{\abs{x} \geq A }  \frac{2x^2}{t^2} e^{-\frac{x ^2 }{8K^2} } dx \right]  \qquad ( \text{From \hyperref[den_low_bound]{\ref*{den_low_bound}}}  ) \\
		& \leq \frac{3}{\sqrt{2\pi  }t }  \left[ \frac{\sqrt{8}\abs{\log C} K  e^{-\frac{A ^2 }{8K^2} } }{A} + \int_{\abs{2Ky} \geq A }  \frac{4K y^2}{t^2} e^{-y^2/2} dy \right]  \label{d1} \\
		& \leq \frac{3}{\sqrt{2\pi  }t }  \left[ \frac{\sqrt{8}\abs{\log C}Ke^{-\frac{A ^2 }{8K^2} } }{A} + \frac{8K \sqrt{2 \pi } }{t^2}\int_{2Ky \geq A }  \frac{y^2}{\sqrt{2 \pi }}e^{-y^2/2} dy \right] \label{d2} \\
		& \leq \frac{3}{\sqrt{2\pi  }t }  \left[ \frac{\sqrt{8}\abs{\log C} Ke^{-\frac{A ^2 }{8K^2} } }{A} + \frac{8K \sqrt{2 \pi } }{t^2} \left[\frac{Ke^{-\frac{A^2}{4K^2}}}{A} + \frac{Ae^{-\frac{A^2}{4K^2}}}{2K} \right]  \right]   \label{c1} \\ 
		& \leq \frac{3 e^{-\frac{A ^2 }{8K^2} } }  {\sqrt{2\pi }t }  \left[ \frac{\sqrt{8}\abs{\log C} K}{A} + \frac{8K \sqrt{2 \pi } }{t^2} \left[\frac{K}{A} + \frac{A}{2K} \right]  \right]. \nonumber \\
		\intertext{We will set $A \geq 10K$ . In this regime, we have the following inequality.}
		\abs{ \int_{\abs{x} \geq A }    f(x) \log f(x) dx   } & \leq \frac{3 e^{-\frac{A ^2 }{8K^2} } }  {\sqrt{2\pi  }t }  \left[ \frac{\sqrt{8}\abs{\log C} K}{A} + \frac{8K \sqrt{2 \pi } }{t^2} \frac{10K}{A}\right] \\
		& \leq \frac{3 e^{-\frac{A ^2 }{8K^2} }  K }  {A\sqrt{2\pi }t }  \left[ {\sqrt{8}\abs{\log C}} + \frac{80K \sqrt{2 \pi } }{t^2}\right].\\
		\intertext{Note that from \autoref{den_low_bound}, we have that $ \abs{\log C} \leq  \log\left(\frac{ \left( 1 - \frac{1}{  2\left(1-t^2\right) }\right) e^{-\frac{2}{t^2}}  }{\sqrt{2 \pi }t} \right)  $. }
		\abs{ \int_{\abs{x} \geq A }    f(x) \log f(x) dx   } &  \leq \frac{\sqrt{2} e^{-\frac{A ^2 }{8K^2} }  K }  {A\sqrt{\pi  } t }  \left[ {\sqrt{8} \abs{\log \left( \frac{ \left( 1 - \frac{1}{  2\left(1-t^2\right) }\right) e^{-\frac{2}{t^2}}  }{\sqrt{2 \pi }t}\right)}} + \frac{80K \sqrt{2 \pi } }{t^2}\right] \\
		& \leq  \frac{C_3 (\log t)K^2 e^{-\frac{A ^2 }{8K^2} }   }  {A t^3 }. 
	\end{align}
	In the above $ C_3 $ is an absolute constant. \autoref{d1} follows using a change of variables. \autoref{c1} follows from the fact that $ x\Phi(x)  - xg(x)$ is an antiderivative of $ x^2 g(x) $, where $ g $ is the Gaussian density and $ \Phi  $ is the cumulative distribution function of the Gaussian and the bound $ 1 - \Phi(x) \leq \frac{e^{-x^2/2}}{x}$. \\
	
	\noindent \textbf{Discretization Error:}
	Next, we consider the error caused by the considering a finite Riemann sum instead of the integral. Let $ B_i $ denote the $ i $th bucket. Note that the $ 2A/B $ is the number of buckets, where $ B $ is the width of the buckets. We denote by $ V_{-A}^{A} $ the total variation of a function between $ -A $ and $ A $ defined as
	
	\begin{defi}
		Let $ g : \left[ a , b \right] \to \br  $ be a function. We define the total variation of the function to be 
		\begin{equation*}
			V_{a}^{b} \left(g \right)= \sup_P \sum_{i} \left| g(p_i) - g(p_{i+1}) \right| ,
		\end{equation*}
		where the supremum is taken all partitions of the interval. If the above quantity is finite, then the function is said to be of finite variation. 
	\end{defi}

	 The following lemma, which gives another way of calculating the total variation of a function, is fairly standard. 
	
	\begin{lem}\cite{apostol1974mathematical}
		For any differentiable function $ g : \left[ a , b  \right]  \to \br$ with a Riemann integrable derivative, we have 
		\begin{equation*}
			V_{a}^{b}(g) = \int_a^b \left| g'\left(x\right)  \right| dx  .
 		\end{equation*}
	\end{lem}
	
	Using the above lemma, we write down the error in discretization.
	Note that the $ \Pr[ X \in B_i ] = B \int_{B_i} f(x) dx $. 
	Note that the by the mean value theorem, this value is taken somewhere in the interval.
	Let $ h : [a,b] \to \br  $ be a continuously differentiable function such that $ h(x) \geq 0 $. Then, note the following for any $ t \in [a,b] $. 
	\begin{align*}
		\abs{\int_{a}^b h(x ) \, dx   - \left(b-a\right)  h(t)  } & \leq  \int_{a}^{b}  \abs{ h(x ) - h(t) } \, dx \\
		& \leq \max_x  \abs{ h(x ) - h(t) } (b-a) \\
		& \leq  \left(b-a\right) V_a^b (h).
 	\end{align*}
	\todo[inline]{Where did $B$ and $B_i$ come from below? How does the first inequality below follow? The second last expression is missing absolute values on the integrands. I guess you are applying the above lemma but I don't see how. Also the lemma uses absolute value and care should be taken to make sure that the lemma is applicable. Where did $C_7$ come from?}
	\todo[inline, color=green]{$ B_i $ and $ B $ have been defined in the paragraph prefacing the definitions and theorems in the discretization error part. \\
	Would it be better if I moved it to later in the text, closer to the theorem?
	The first inequality is a classical bound on the error of the Riemann sum. I could not find a reference but I could put a proof.}
	We use the above well known estimate for the error in a finite Riemann sum to get the following.
	\begin{align*}
		\abs{ \int_{\abs{x} \leq A }    f(x) \log f(x) \; dx  - B\sum \frac{\Pr[X \in B_i]  }{B}\log \frac{\Pr[X \in B_i]}{B}} \leq &B \times V_{-A}^A ( f \log f   ) \\
		\leq& B \int_{|x| \leq A } \abs{ f'(x) \left( \log f(x) + 1  \right) }dx \\
		\leq&  B \left[ \int_{-A}^A  \abs{f' \log f} dx + \int_{-A}^A \abs{ f'} dx   \right]. \\
	\end{align*}
	Using \autoref{den_der_bound}, \autoref{den_low_bound} and \autoref{den_upper_bound}, we get the following. 
	\begin{align*}
	 \left[ \int_{-A}^A  \abs{f' \log f} dx + \int_{-A}^A \abs{ f'} dx   \right]	&\leq  \int_{-A}^{A}   \frac{ e^{ - \frac{x^2}{8K^2 } } }{t^2 } \left[ \frac{\abs{x}}{\sqrt{2\pi t^2}}  + 2   \right] + \frac{2 x^2 e^{ - \frac{x^2}{8K^2 } } }{ t^2} \left[ \frac{\abs{x} }{\sqrt{2\pi t^2}}  + 2  \right]  \\
	 & \qquad + \frac{ e^{ - \frac{x^2}{8K^2 } } }{t^2 } \left[ \frac{\abs{x}}{\sqrt{2\pi t^2}}  + 2  \right] \abs{\log C(K,t) } dx. \\
	 \intertext{Note that this integral is polynomial in $ K $ and $t^{-1} $. For brevity, we denote this function as $ C_2(K,t) $.  }
	 \left[ \int_{-A}^A  \abs{f' \log f} dx + \int_{-A}^A \abs{ f'} dx   \right]	& \leq C_2(K,t) .
	\end{align*}
	Putting the two inequalities together we get the following bound. 
	\begin{align*}
		\abs{ \int_{\abs{x} \leq A }    f(x) \log f(x) \; dx  - B\sum \frac{\Pr[X \in B_i]  }{B}\log \frac{\Pr[X \in B_i]}{B}} \leq & C_2(K,t) B.
	\end{align*}
	
	\noindent \textbf{Estimation Error:}
	We will use a histogram estimator for the density. Let $ \bar{N_i} $ denote the number of elements in the $ i $th bucket and let $ N $ be the total number of samples. Using Chernoff Bounds, we get  
	\begin{align*}
	\Pr\left[    \abs{\frac{N_i}{N} - \Pr[ X \in B_i  ]  } \geq \delta   \right] \leq 2e^{-2N\delta^2}.
	\end{align*}
	We use this $ \delta $ approximation to $ \Pr[X \in B_i] $ in the sum. This gives us the error terms, 
	\begin{align*}
		\abs{B\sum_i \frac{\Pr[X \in B_i]  }{B}\log \frac{\Pr[X \in B_i]}{B} - B\sum_i \frac{N_i  }{NB}\log \frac{N_i}{NB}} & \leq B\abs{  \sum_i  \left( \frac{\Pr[X \in B_i]  }{B} \log \frac{\Pr[X \in B_i] }{B}    - \frac{N_i  }{NB}\log \frac{N_i}{NB}   \right)      } \\
        &\leq B  \sum_i \abs{  \frac{\Pr[X \in B_i]  }{B} \log \frac{\Pr[X \in B_i] }{B}    - \frac{N_i  }{NB}\log \frac{N_i}{NB}       } .\\
        \intertext{Using the Taylor theorem to order two to control the error, we get, }
		\abs{B\sum_i \frac{\Pr[X \in B_i]  }{B}\log \frac{\Pr[X \in B_i]}{B} - B\sum_i \frac{N_i  }{NB}\log \frac{N_i}{NB}}& \leq B\sum_{i}  \vline  \frac{\delta}{B} \left( 1 - \log \frac{\Pr[ X \in B_i ]}{B}    \right)   \\& \qquad +\frac{\delta^2}{2 B\left( \Pr\left[ X \in B_i  \right] - \delta   \right)  }  \vline \\ 
		& \leq  \frac{2A\delta}{B} + \delta  \sum_{i}  \abs{ \log \Pr [ X \in B_i ]} + \frac{\delta^2}{2 B\left( \Pr\left[ X \in B_i  \right] - \delta   \right)  }. \\
        \intertext{We will set $ \delta = \alpha \Pr[ X \in B_i ] $ for the least value of $ \Pr[X \in B_i] $}
        & \leq   \frac{2A\delta}{B} + \delta  \sum_{i}  \abs{ \log \Pr [ X \in B_i ]} + \frac{\alpha^2 \Pr[X \in B_i ]  }{2 B\left(1 - \alpha \right)  }. \\
        \intertext{We will pick $ \alpha \leq 0.5$. }
        & \leq  \frac{2A\delta}{B} + \frac{\alpha^2}{B} + \delta \sum_{i}  \abs{ \log \Pr [ X \in B_i ]} \\
        &  \leq  \frac{2A\delta}{B} + \frac{\alpha^2}{B} + \delta \sum_{i}  \abs{ \log\left( B C(t)  e^{- \frac{A}{t^2}}     \right) } .\\
        \intertext{We use the lower bound on the density multiplied by the bucket size to give a lower bound on the probability of each bucket.}
        & \leq  \frac{2A\delta}{B} + \frac{\alpha^2}{B} + \delta \frac{2A}{B}  \abs{ \log\left( B C(t)  e^{- \frac{A}{t^2}}     \right) } \\
		& \leq  \frac{2A\delta}{B} + \frac{\alpha^2}{B} + \delta \frac{A^2  \delta }{Bt^4}  \abs{ \log\left(   \frac{2B}{\sqrt{2 \pi } t}   \right) } .
	\end{align*}
	The total error is the bounded by the sum of these terms by the triangle inequality. Thus, 
	
	\begin{align*}
		\abs{ B\sum_i \frac{N_i  }{NB}\log \frac{N_i}{NB} - \int f(x) \log f(x) dx     } & \leq  C_3 \frac{e^{ -\frac{A^2}{K^2}  } t K^3 }{A t^2 } + \frac{2A\delta}{B} + \frac{\alpha^2}{B} + \delta \frac{A^2  \delta }{Bt^4}  \abs{ \log\left(   \frac{2B}{\sqrt{2 \pi } t}   \right) }  + C_2 B .
	\end{align*} 
	
	Now, to get within error of $ \epsilon $ and probability of correctness $ 1-\gamma $, we set  $ B =  \epsilon / 100C_2$, $ A = 100 (|\log C_3| )  K \log K \sqrt{ \log\left( 1 / t \right)  \log(1/\epsilon)}   $, $ \delta =  B \epsilon C(t ) e^{ -\frac{A^2}{t^2} }   / (A \abs{\log(B / (1-t^2) ) } )   $ and $ N = \log(2A/ B\gamma) / \delta^2 $. 
    \qedhere
	\end{proof}

 \section{Lipschitz Continuity of the Gradient of Entropy of the Marginals} \label{Lip_con}
    We would like to bound the Lipschitz constant of gradient of the function defined by $ S \left( \ip{X}{a}  \right) $. 
    We do this by bounding the norm of the Hessian matrix. For any twice continuously differentiable function $ f : \br^n \to \br $, we have 
    \begin{equation*}
     \norm{ \nabla^2 f }_2 \leq L \implies  \forall x ,y  :  \norm{ \nabla f(x) - \nabla f(y)  } \leq L \norm{x - y }.
    \end{equation*}
    Note that since the Hessian matrix is symmetric, we have 
    \begin{align*}
        \norm{ \nabla^2 f }_2  = & \max_{\norm{v} =1 } \norm{ (\nabla^2 f ) v  } \\
        = & \max_{\norm{v} =1 } \sqrt{ \ip{ v  }{ (\nabla^2 f )^2 v   } } \\
        = &   \max_i   | \lambda_i |   .    
    \end{align*}
    Picking an eigenbasis for $ \nabla^2 f  $, we get this to be equal to $ \partial_{x_i}^2 f $ for some choice of basis. So, it suffices to bound $ \partial_{x_i}^2 f $ for arbitrary directions $ x_i $ at every point on the sphere.
    Now we restrict to the plane generated by current point on the sphere and the direction along which we take the derivative. 
    Let $ (X,Y) $ denote the random variable obtained after the restriction. 
    Note that by the isotropy of the original random variable, all the one-dimensional projections of $ (X,Y) $ also have unit variance. 
     Recall from our discussion in Section~\ref{subsec:Lip_Grad} that we can assume that our random variable has been smoothened by isotropic Gaussian noise. Smoothening helps control the Lipschitz constant.
    Adding isotropic Gaussian noise in $n$ dimensions and projecting to two dimensions is the same as projecting first to two dimensions and then adding isotropic Gaussian noise. Hence we can consider the smoothened random variable $ \left( X_t , Y_t \right) $ defined below in \eqref{eqn:XtYt}. 
    Thus, we will work in an arbitrary plane, i.e. consider an arbitrary two dimensional isotropic random variable that satisfies the required regularity conditions (twice continuously differentiable density and having finite relative entropy) and is perturbed by a small amount of Gaussian noise. Let us restate our result for the Lipschitz constant for convenience:
    
    \begin{theorem} \label{Lip_Bound}
        Let $ X  $ be random variable of unit variance and zero mean which is $ K $ subgaussian and $ Z $ be a standard Gaussian random variable independent of $ X $.
        Let $ Y_t = \sqrt{1-t^2}X +tZ  $.
        Let $ S(a) = S \left(\ip{Y_t}{a} \right)  $. 
        Then, $ S $ is Lipschitz continuous with Lipschitz contant bounded by a polynomial in $ K $ and $ t^{-1} $. 
    \end{theorem}
    We prove some lemmas for the proof in the next subsection and the proof of the Lipshitz continuity appears \autoref{Lip_Ent}. 
    We first consider the case of random variables with bounded density, say with support contained in $ \left[ -B , B  \right]^n \subseteq \br^n $. 
    
    \subsection{Upper Bound on the Density and its Derivatives}
    Let $ \rho(x,y) $ be the joint density of random variable with support bounded by radius $ B $. 
     We assume that the random variable has density for convenience of notation.
     Now, consider the random variable obtained by adding a small amount of Gaussian noise, i.e., consider the random variable given by 
     \begin{equation}\label{eqn:XtYt}
     ( X_t , Y_t ) = \sqrt{ 1 - t^2}\left( X,Y \right)  + tZ .
     \end{equation}
    We will bound the derivatives of $ S( \ip{\left( X_t , Y_t \right)}{a}  ) $. 
    As a first step, we bound the derivatives of the density.
    Then, through an application of the chain rule, translate this to bounds on the derivatives of the entropy.
    It is more convenient to work with polar co-ordinates here. 
    We shall abuse notation and $ \rho $ also as a function of $ \left( r , \theta \right) $, implicitly referring to the function value at the corresponding point denoted in Cartesian co-ordinates. 
     We first write the density for $ \left( X_t , Y_t \right) $ as $ \rho_t $ in Cartesian coordinates,
     \begin{equation*}
     	\rho_t\left(x,y\right) =  \frac{1}{2 \pi (1 - t^2)t^2 }    \int_{\br} \int_{\br} \rho\left( \frac{u}{\sqrt{1 - t^2}} , \frac{v}{  \sqrt{1 - t^2}} \right)  e^{ - \frac{  \norm{ \left( x - u , y - v \right) }^2   }{2 t^2}  }  \,du \, dv . 
     \end{equation*} 
     Note that the point $ \left( x - u , y - v  \right) $ has squared norm $ R^2 + r^2 - 2 Rr \cos\left( \theta - \phi \right) $ where $ R $ and $ r $ are the norms of $ \left( u , v \right) $ and $ \left( x , y \right) $ respectively and $ \phi $  and $ \theta $ are the corresponding angles. Now, writing this integral in polar coordinates, by noting that the Jacobian $\det\frac{\partial(u,v)}{\partial(R,\phi)}$ is $ R $,
    \begin{equation*}
            \rho_t \left( r, \theta  \right) = \int_{0}^{\sqrt{   1 - t^2}B} \int_{0}^{2 \pi}  \frac{1}{2 \pi (1 - t^2)t^2 } \cdot \rho\left( \frac{R}{\sqrt{1 - t^2}} , \phi \right) e^{- \frac{ R^2 + r^2 - 2 Rr \cos\left( \theta - \phi  \right)  }{2t^2}} R \, d\phi \, dR. 
    \end{equation*}
    We first compute the first derivative of this function with respect to $ \theta $. 
    \begin{align*}
        \partial_{\theta} \rho_t \left( r, \theta  \right) = & \partial_{\theta} \int_{0}^{\sqrt{1 - t^2}B} \int_{0}^{2 \pi}  \frac{1}{2 \pi (1 - t^2) t^2 } \cdot \rho\left( \frac{R}{\sqrt{1 - t^2}} , \phi \right) e^{- \frac{ R^2 + r^2 - 2 Rr \cos\left( \theta - \phi  \right)  }{2t^2}} \,R \,d\phi\, dR  \\
        = & \frac{1}{ 2 \pi (1 - t^2)t^2 }   \int_{0}^{\sqrt{1 - t^2}B} \int_{0}^{2 \pi} \rho\left( \frac{R}{\sqrt{1 - t^2}} , \phi \right) \partial_{\theta} e^{- \frac{ R^2 + r^2 - 2 Rr \cos\left( \theta - \phi  \right)  }{2t^2}} \,R \,d\phi\, dR \\
        =& \frac{r}{ 2 \pi (1 - t^2) t^4 }   \int_{0}^{\sqrt{1 - t^2}B} \int_{0}^{2 \pi} \rho\left( \frac{R}{\sqrt{1 - t^2}} , \phi \right)  e^{- \frac{ R^2 + r^2 - 2 Rr \cos\left( \theta - \phi  \right)  }{2t^2}} R^2 \sin\left( \phi - \theta  \right) \, d\phi \, dR. 
    \end{align*}
    We bound this in modulus, which gives us. 
    \begin{align}
        \abs{\partial_{\theta} \rho_t \left( r, \theta  \right) } \leq & \frac{r}{  2 \pi (1 - t^2)t^4 }   \int_{0}^{\sqrt{1 - t^2}B} \int_{0}^{2 \pi} \rho\left( \frac{R}{\sqrt{1 - t^2}} , \phi \right)  e^{- \frac{ R^2 + r^2 - 2 Rr \cos\left( \theta - \phi  \right)  }{2t^2}} R^2 \abs{\sin\left( \phi - \theta  \right)} \, d\phi\, dR \nonumber \\
        \leq & \frac{r}{  2 \pi (1 - t^2)t^4 }   \int_{0}^{\sqrt{1 - t^2}B} \int_{0}^{2 \pi} \rho\left( \frac{R}{\sqrt{1 - t^2}} , \phi \right)  e^{- \frac{ R^2 + r^2 - 2 Rr \cos\left( \theta - \phi  \right)  }{2t^2}} R^2   \,d\phi \, dR   \label{bound1}  \\
        \leq & \frac{r  B \sqrt{ 1 - t^2}  }{  2 \pi (1 - t^2) t^4 }   \int_{0}^{\sqrt{1 - t^2}B} \int_{0}^{2 \pi} \rho\left( \frac{R}{\sqrt{1 - t^2}} , \phi \right)  e^{- \frac{ R^2 + r^2 - 2 Rr \cos\left( \theta - \phi  \right)  }{2t^2}}  R \, d\phi\, dR  \label{bound2} \\
        = & \frac{r  B \sqrt{ 1 - t^2} }{t^2 }  \rho_t . \label{b1}
     \end{align}
    
    \autoref{bound1} and \autoref{bound2} follows from the fact that $ \abs{\sin(x)} \leq 1 $ and the fact that $ R \leq B $. 
    We write a similar bound for the second derivative. 
        \begin{align*}
        \abs{\partial_{\theta}^2 \rho_t \left( r, \theta  \right)} = & \abs{ \partial^2_{\theta} \int_{0}^{\sqrt{1 - t^2}B} \int_{0}^{2 \pi}  \frac{1}{ 2 \pi (1 - t^2) t^2 } \rho\left( \frac{R}{\sqrt{1 - t^2}} , \phi \right) e^{- \frac{ R^2 + r^2 - 2 Rr \cos\left( \theta - \phi  \right)  }{2t^2}} R\,d\phi\, dR } \\
        = &  \abs{\frac{1}{2 \pi (1 - t^2) t^2 }   \int_{0}^{\sqrt{1 - t^2}B} \int_{0}^{2 \pi} \rho\left( \frac{R}{\sqrt{1 - t^2}} , \phi \right) \partial^2_{\theta} e^{- \frac{ R^2 + r^2 - 2 Rr \cos\left( \theta - \phi  \right)  }{2t^2}} R\,d\phi\, dR } \\
        =& \abs{\frac{r}{2t^4 \pi (1 - t^2)  }   \int_{0}^{\sqrt{1 - t^2}B} \int_{0}^{2 \pi} \rho\left( \frac{R}{\sqrt{1 - t^2}} , \phi \right)  \partial_{\theta}\left [  e^{- \frac{ R^2 + r^2 - 2 Rr \cos\left( \theta - \phi  \right)  }{2t^2}} R^2 \sin\left( \phi - \theta  \right) \right] \,d\phi \,dR } \\
        = & \abs{\frac{r}{ 2t^4 \pi (1 - t^2)  }   \int_{0}^{\sqrt{1 - t^2}B} \int_{0}^{2 \pi} \rho\left( \frac{R}{\sqrt{1 - t^2}} , \phi \right)  \left [  e^{- \frac{ R^2 + r^2 - 2 Rr \cos\left( \theta - \phi  \right)  }{2t^2}} R^2 \cos\left( \phi - \theta  \right) \right] \, d\phi \, dR } + \\
        & \qquad \abs{ \frac{r^2}{2t^6 \pi (1 - t^2) }   \int_{0}^{\sqrt{1 - t^2}B} \int_{0}^{2 \pi} \rho\left( \frac{R}{\sqrt{1 - t^2}} , \phi \right)  \left [ - e^{- \frac{ R^2 + r^2 - 2 Rr \cos\left( \theta - \phi  \right)  }{2t^2}} R^3 \sin^2\left( \phi - \theta  \right) \right]  \,d\phi \, dR } \\
        \leq & \,  \rho_t \left(   \frac{r  B \sqrt{ 1 - t^2} }{t^2 } +  \frac{r^2  B^2 \left(1 - t^2\right)  }{t^4 }   \right). 
        \end{align*}
    Thus, we have reduced to the problem of upper bounding the density. 
    We now move on to  upper bounding  the density obtained by convolution. To do this, as above, we expand the integral defining the convolution. 
    \begin{align*}
    	\rho_t (r , \theta) = & \int_{0}^{\sqrt{   1 - t^2}B} \int_{0}^{2 \pi}  \frac{1}{2 \pi (1 - t^2)t^2 } \rho\left( \frac{R}{\sqrt{1 - t^2}} , \phi \right) e^{- \frac{ R^2 + r^2 - 2 Rr \cos\left( \theta - \phi  \right)  }{2t^2}} R \, d\phi \,dR  \\
    	=  & \frac{1}{2 \pi (1 - t^2)t^2 }  \int_{0}^{\sqrt{   1 - t^2}B} \int_{0}^{2 \pi} \rho\left( \frac{R}{\sqrt{1 - t^2}} , \phi \right) e^{- \frac{ R^2 + r^2 - 2 Rr \cos\left( \theta - \phi  \right)  }{2t^2}} R \,d\phi \, dR.  \\
    \end{align*}
    For $ r \geq B $, we get
    \begin{align}
    	\rho_t (r , \theta) =& \frac{1}{2 \pi (1 - t^2)t^2 }  \int_{0}^{\sqrt{   1 - t^2}B} \int_{0}^{2 \pi} \rho\left( \frac{R}{\sqrt{1 - t^2}} , \phi \right) e^{- \frac{ R^2 + r^2 - 2 Rr \cos\left( \theta - \phi  \right)  }{2t^2}} R \,d\phi \, dR  \nonumber \\ 
    	\leq & \frac{1}{2 \pi (1 - t^2)t^2 }  \int_{0}^{\sqrt{   1 - t^2}B} \int_{0}^{2 \pi} \rho\left( \frac{R}{\sqrt{1 - t^2}} , \phi \right) e^{- \frac{ R^2 + r^2 - 2 Rr   }{2t^2}} R \,d\phi \, dR \label{cosleq} \\ 
    	\leq & \frac{e^{ - \frac{\left( r - B \right)^2 }{2t^2}  }}{2 \pi t^2 }  \int_{0}^{\sqrt{   1 - t^2}B} \int_{0}^{2 \pi} \rho\left( \frac{R}{\sqrt{1 - t^2}} , \phi \right) R \,d\phi \, dR \label{rleq}  , \\
    	\leq &  \frac{e^{ - \frac{\left( r - B \right)^2 }{2t^2}  }}{2 \pi t^2 }, \label{denleq}
    	\intertext{while for $ r< B$, we have }
    	\rho_t (r , \theta) \leq & \frac{1}{2 \pi t^2}  \label{expleq} ,
    	\intertext{which gives us, }
    	\rho_t (r , \theta) \leq & \frac{1}{2 \pi t^2 } \min \left\{ 1 , e^{ - \frac{\left( r - B \right)^2 }{2t^2} }  \right\}. \label{der_upper_b}
    \end{align}
    In the above, \autoref{cosleq} comes from the fact that $ \cos(x) \leq 1 $, \autoref{rleq} comes from the fact that $ R = B  $ minimizes the exponential term in the range and \autoref{denleq} comes from the fact that the density integrates to one. \autoref{expleq} comes from the fact that the exponential term is always less than equal to one. 

    \subsection{Lipschitz Continuity of the Gradient of Entropy of the Marginals} \label{Lip_Ent}
    
    
        It will be convenient to work with polar coordinates $(r, \phi)$, with $r \geq 0$ and $\phi \in [0, 2\pi)$. In the following, we will
        write $\rho(r, \phi)$ for the value of $\rho$ at the point $(r,\phi)$.
        We need to work with the marginal density of $ \rho  $ along a direction $ v \in S^{1}$. 
        Since we are working in the polar coordinate system, we shall refer to this direction using its angle $ \theta $. 
        Denote by $ \rho_{\theta} $ the one dimensional density of the marginal along the direction specified by $ \theta  $.  
        It can be shown that for $z > 0$, $ \rho_{\theta } $ can be written as 
        \begin{align*}
        \rho_{\theta}(z) = z \int\limits_{\theta - \frac{\pi}{2} }^{\theta + \frac{\pi}{2}} \sec^2(\theta-\phi) \: \rho_t(z\sec(\theta-\phi), \phi) \, d\phi.
        \end{align*}
        Under our regularity conditions $\rho_\theta(z) \rightarrow \rho_\theta(0)$ as $z \rightarrow 0$ from above, where
        \begin{align*}
        \rho_\theta(0) := \int\limits_{-\infty}^{\infty} \rho_t(y, \theta+ \pi/2) \, dy.
        \end{align*}
        We will also define $\rho_\theta(z) := \rho_{\theta+\pi}(-z)$ for $z < 0$, which is easily seen to be the natural definition. 
        Having defined $\rho_\theta$ on $\br$, we can now define the relative entropy of our interest:
        \begin{equation*}
        S(\rho_{\theta}) = \int_\br \rho_{\theta}(z) \; \log \left( \frac{\rho_{\theta}(z)}{g(z)}   \right) dz.
        \end{equation*}

        To get the density of the marginal of the distribution, we integrate along the line perpendicular the direction of interest. Note that the parametric equation of such a line is given by $  r  = z \sec(\theta - \phi )  $, where $ \theta  $ denotes the angle of the direction of interest with respect to the $ x $-axis. The following equation obtained by a simple change of variables will greatly simplify our computations.
        \begin{align*}
        \rho_{\theta}(z) = z \int\limits_{-\frac{\pi}{2} }^{\frac{\pi}{2}} \sec^2(\phi)\: \rho_t(z\sec{\phi}, \phi+\theta) \, d\phi.
        \end{align*}
        Looking at the derivative of $ \rho_{\theta} $ with respect to $\theta$, we get
        \begin{align} 
        \partial_{\theta} \rho_{\theta} (z) = z \int\limits_{-\frac{\pi}{2} }^{\frac{\pi}{2}} \sec^2(\phi)\: 
        \partial_\theta \rho_t(z\sec{\phi}, \phi+\theta) \, d\phi , \label{eqn:dthetarho}\\
        \partial^2_{\theta} \rho_{\theta} (z) = z \int\limits_{-\frac{\pi}{2} }^{\frac{\pi}{2}} \sec^2(\phi)\: 
        \partial^2_\theta \rho_t(z\sec{\phi}, \phi+\theta) \, d\phi . \label{eqn:d2thetarho} 
        \end{align}
        Using the bound for the first derivative from \autoref{der_upper_b} in \autoref{eqn:dthetarho}, we get
        \begin{align}
        \abs{\partial_{\theta} \rho_{\theta}(z)} 
        \leq & z \int\limits_{-\frac{\pi}{2} }^{\frac{\pi}{2}} \sec^2(\phi)\: \abs{\partial_\theta \rho_t(z\sec{\phi}, \phi+\theta)} \, d\phi \nonumber \\
        \leq & \frac{z^2 B  \sqrt{1 - t^2}}{t^2} \int\limits_{-\frac{\pi}{2} }^{\frac{\pi}{2}} \sec^3(\phi)\: \rho_t(z\sec{\phi}, \phi+\theta) \, d\phi   \label{der_den_b} \\
        \leq & \frac{z^2 B  \sqrt{1 - t^2}}{ 2 \pi   t^4} \int\limits_{-\frac{\pi}{2} }^{\frac{\pi}{2}} \sec^3(\phi)\: \min\left\{ 1 , e^{ - \frac{ \left(z \sec{\phi} - B \right)^2 }{2t^2}  }    \right\} \, d\phi .  \label{b2} \\
        \intertext{For $ z \geq 10B $, we get }
        \abs{\partial_{\theta} \rho_{\theta}(z)} \leq & \frac{z^2 B  \sqrt{1 - t^2}}{ 2 \pi   t^4} \int\limits_{-\frac{\pi}{2} }^{\frac{\pi}{2}} \sec^3(\phi)\: e^{ - \frac{ \left( 0.9 z \sec{\phi} \right)^2 }{2t^2}  }    \, d\phi  \label{b3} \\
        \leq & \frac{z^2 B  \sqrt{1 - t^2}}{ 2 \pi   t^4} \int\limits_{-\frac{\pi}{2} }^{\frac{\pi}{2}} \sec^3(\phi)\: e^{ - \frac{  (0.9)^2 z^2 + (0.9)^2 z^2 \tan^2\left(\phi \right)  }{2t^2}  }    \, d\phi \label{b4}  \\
        \leq & \frac{z^2   e^{ - \frac{  (0.9)^2 z^2}{2t^2}  }   B  \sqrt{1 - t^2}}{ 2 \pi   t^4} \int\limits_{-\frac{\pi}{2} }^{\frac{\pi}{2}} \sec^3(\phi)\: e^{ - \frac{  (0.9)^2 z^2 \tan^2\left(\phi \right)  }{2t^2}  }    \, d\phi \\ 
         \leq & \frac{z^2   e^{ - \frac{  (0.9)^2 z^2}{2t^2}  }   B  \sqrt{1 - t^2}}{ 2 \pi   t^3} C_1(z,t). \label{b5}
         \intertext{For $z  \leq 10 B $ , }
         \abs{\partial_{\theta} \rho_{\theta}(z)} \leq & \frac{z^2 B  \sqrt{1 - t^2}}{ 2 \pi   t^4} \int_{ z \sec\left( \phi  \right) \leq  10B } \sec^3(\phi)\ d\phi  +  \frac{z^2 B  \sqrt{1 - t^2}}{ 2 \pi   t^4} \int_{z \sec( \phi ) > 10B }  \sec^3(\phi)\: e^{ - \frac{ \left( 0.9 z \sec{\phi} \right)^2 }{2t^2}  }    \, d\phi \label{b6}\\ 
         \leq &  \frac{10z B^2  \sqrt{1 - t^2}}{ 2 \pi   t^4} \int_{ z \sec\left( \phi  \right) \leq  10B } \sec^2(\phi)\ d\phi  +  \frac{z^2 B  \sqrt{1 - t^2}}{ 2 \pi   t^4} \int_{z \sec( \phi ) > 10B }  \sec^3(\phi)\: e^{ - \frac{ \left( 0.9 z \sec{\phi} \right)^2 }{2t^2}  }    \, d\phi \label{b7}\\ 
         \leq & \frac{10 B^2  \sqrt{1 - t^2}}{ 2 \pi   t^4} \sqrt{100 B^2  - z^2 }  +  \frac{z^2 B  \sqrt{1 - t^2}}{ 2 \pi   t^4} \int_{z \sec( \phi ) > 10B }  \sec^3(\phi)\: e^{ - \frac{ \left( 0.9 z \sec{\phi} \right)^2 }{2t^2}  }    \, d\phi \label{b8}\\ 
         \leq & \frac{100 B^3  \sqrt{1 - t^2}}{ 2 \pi   t^4}  +  \frac{z^2   e^{ - \frac{  (0.9)^2 z^2}{2t^2}  }   B  \sqrt{1 - t^2}}{ 2 \pi   t^3} C_1(z,t). \label{b9}
        \end{align}
        \autoref{der_den_b} follows from \autoref{b1}, \autoref{b2} follows from \autoref{der_upper_b}, \autoref{b3} follows because $ \abs{\sec(x)} \geq 1  $ and \autoref{b4} follows due to the fact that $ \sec^2(x) - \tan^2(x) = 1  $. $ C_1(z,t) $ in \autoref{b5} is a polynomial function in $ t $ and $ z $. 
        \autoref{b6} divides the integral into cases depending on the value of the integrand, \autoref{b7} follows using the fact that $ z\sec(x) \leq 10B $ and \autoref{b8} uses the fact that the antiderivative of $ \sec^2(x)  $ is $ \tan(x) $ and \autoref{b9} uses the fact that the second integral is bounded by the integral in \autoref{b4}.  
        Similarly, using the bound on the second derivative we get 
        \begin{align*}
        \abs{\partial^2_{\theta} \rho_{\theta}(z)} \leq   \frac{z^2 B  \sqrt{1 - t^2}}{t^2} \int\limits_{-\frac{\pi}{2} }^{\frac{\pi}{2}} \sec^3(\phi)\: \rho(z\sec{\phi}, \phi+\theta) \, d\phi +  \frac{z^3 B^2  (1 - t^2)}{t^4} \int\limits_{-\frac{\pi}{2} }^{\frac{\pi}{2}} \sec^4(\phi)\: \rho(z\sec{\phi}, \phi+\theta) \, d\phi .
        \end{align*}
        The first integral is the same as the one used to bound the first derivative. Again, we take cases. 
        For $ z \geq 10B  $,
        \begin{align*}
           \abs{\partial^2_{\theta} \rho_{\theta}(z)} \leq & \frac{z^2   e^{ - \frac{  (0.9)^2 z^2}{2t^2}  }   B  \sqrt{1 - t^2}}{ 2 \pi   t^3} C_1(z,t) +  \frac{z^3 B^2  (1 - t^2)}{t^4} \int\limits_{-\frac{\pi}{2} }^{\frac{\pi}{2}} \sec^4(\phi)\: \rho(z\sec{\phi}, \phi+\theta) \, d\phi \\
           \leq & \frac{z^2   e^{ - \frac{  (0.9)^2 z^2}{2t^2}  }   B  \sqrt{1 - t^2}}{ 2 \pi   t^3} C_1(z,t) +\frac{z^3 B^2  \left(1 - t^2\right) }{    t^4} \int\limits_{-\frac{\pi}{2} }^{\frac{\pi}{2}} \sec^4(\phi)\: e^{ - \frac{  (0.9)^2 z^2 + (0.9)^2 z^2 \tan^2\left(\phi \right)  }{2t^2}  }    \, d\phi \\
           \leq & \frac{z^2   e^{ - \frac{  (0.9)^2 z^2}{2t^2}  }   B  \sqrt{1 - t^2}}{ 2 \pi   t^3} C_1(z,t) +\frac{z^3 e^{ - \frac{  (0.9)^2 z^2}{2t^2}  } B^2  \left(1 - t^2\right) }{    t^4} C_2(z,t).
        \end{align*}
        For $z \leq 10B$, 
        \begin{align*}
              \abs{\partial^2_{\theta} \rho_{\theta}(z)} \leq & \frac{100 B^3  \sqrt{1 - t^2}}{ 2 \pi   t^4}  +  \frac{z^2   e^{ - \frac{  (0.9)^2 z^2}{2t^2}  }   B  \sqrt{1 - t^2}}{ 2 \pi   t^3} C_1(z,t) +  \frac{z^3 B^2  (1 - t^2)}{t^4} \int\limits_{-\frac{\pi}{2} }^{\frac{\pi}{2}} \sec^4(\phi)\: \rho(z\sec{\phi}, \phi+\theta) \, d\phi \\ 
              \leq & \frac{100 B^3  \sqrt{1 - t^2}}{ 2 \pi   t^4}  +  \frac{z^2   e^{ - \frac{  (0.9)^2 z^2}{2t^2}  }   B  \sqrt{1 - t^2}}{ 2 \pi   t^3} C_1(z,t) +  \frac{z^3 B^2  \left(1 - t^2\right)}{ t^4} \int_{ z \sec\left( \phi  \right) \leq  10B } \sec^4(\phi)\ d\phi  \,    \\ &  \qquad  + \frac{z^3 B^2  \left(1 - t^2\right)}{ t^4} \int_{z \sec( \phi ) > 10B }  \sec^4(\phi)\: e^{ - \frac{ \left( 0.9 z \sec{\phi} \right)^2 }{2t^2}  }    \, d\phi \\  
              \leq &  \frac{100 B^3  \sqrt{1 - t^2}}{ 2 \pi   t^4}  +  \frac{z^2   e^{ - \frac{  (0.9)^2 z^2}{2t^2}  }   B  \sqrt{1 - t^2}}{ 2 \pi   t^3} C_1(z,t) + \frac{100z B^4  \left(1 - t^2\right)}{ t^4} \int_{ z \sec\left( \phi  \right) \leq  10B } \sec^2(\phi)\ d\phi  \, \\
              &\qquad  + \frac{z^3 B^2  \left(1 - t^2\right)}{ t^4}C_4(z,t) \\
               \leq &  \frac{100 B^3  \sqrt{1 - t^2}}{ 2 \pi   t^4}  +  \frac{z^2   e^{ - \frac{  (0.9)^2 z^2}{2t^2}  }   B  \sqrt{1 - t^2}}{ 2 \pi   t^3} C_1(z,t) + \frac{1000 B^5  \left(1 - t^2\right)}{ t^4}   + \frac{z^3 B^2  \left(1 - t^2\right)}{ t^4}C_4(z,t). \\
        \end{align*}
        We now bound the modulus of the second derivative of the entropy.  
        \begin{align*}
        \abs{\partial_{\theta}^2 S(\theta)} & = \abs{\int_{\br} \left[\partial_{\theta}^2 \rho_{\theta}(z) \left( 1 + \log{\rho_{\theta} (z)} \right)  + \frac{\left( \partial_{\theta} \rho_{\theta}  \right)^2 }{\rho_{\theta} (z) } \right] dz } \\
        & \leq  \int_{\br}  \abs{\partial_{\theta}^2 \rho_{\theta}(z) }  dz + \int_{\br} \abs{ \partial_{\theta}^2 \rho_{\theta}(z) }\abs{ \log \rho_{\theta} \left(z\right)  } dz + \int_{\br} \frac{\left( \partial_{\theta} \rho_{\theta}  \right)^2 }{\rho_{\theta} (z) } dz . 
        \end{align*}
        Consider the following sequence of inequalities. First note that from \autoref{der_den_b}, we get the following. 
        \begin{align*}
            \abs{\partial_{\theta } \rho_{\theta} (z) }\leq &\frac{B  \sqrt{1-t^2} }{t^2} \int_{M_{x,z}} r\rho_t  dM \\
            \leq &  \frac{B  }{t^2} \int_{M_{x,z}} r\rho_t  dM.
        \end{align*}
        We will pick our basis such that the integration is on a line parallel to the y axis and we integrate in Cartesian coordinates.
        \begin{align}
            \abs{ \partial_{\theta } \rho_{\theta} (z)} \leq &  \frac{B}{t^2}  \int_{\br} \sqrt{z^2 + y^2 } \rho_t (z,y) dy \\
            \leq & \frac{B}{t^2}  \int_{\br} \left( \abs{z} + \abs{y}   \right) \rho_t (z,y) dy  \nonumber  \\ 
            \leq  & \frac{B}{t^2} \int_{\br} \left( \abs{z} + \abs{y}   \right) \frac{1}{2 \pi t^2 (1 - t^2 )}  \iint_{\br^2}  \rho\left( \frac{u}{\sqrt{1 - t^2} } , \frac{v }{\sqrt{1 - t^2}}  \right) e^{ - \frac{\left( u -z  \right)^2 + \left( v - y  \right)^2 }{2t^2} } du \, dv \,   dy   \label{b10} \\ 
            \leq & \frac{B}{2 \pi t^4 (1 - t^2 )} \iint_{\br^2}  \rho\left( \frac{u}{\sqrt{1 - t^2} } , \frac{v }{\sqrt{1 - t^2}}  \right)  \int_{\br} \left( \abs{z} + \abs{y}  \right)   e^{ - \frac{\left( u -z  \right)^2 + \left( v - y  \right)^2 }{2t^2} } dy \, du \, dv  \label{b11} \\
            \leq & \frac{B}{2 \pi t^4 (1 - t^2 )} \iint_{\br^2}  \rho\left( \frac{u}{\sqrt{1 - t^2} } , \frac{v }{\sqrt{1 - t^2}}  \right) e^{ - \frac{(z - u)^2 }{2t^2} }   \int_{\br} \left( \abs{z} + \abs{y}  \right)   e^{ - \frac{\left( v - y  \right)^2 }{2t^2} } dy \, du \, dv . \nonumber \\ 
            \intertext{Using $ v - y = \alpha $, we get }
            \abs{ \partial_{\theta } \rho_{\theta} (z)} \leq & \frac{B}{2 \pi t^4 (1 - t^2 )} \iint_{\br^2}  \rho\left( \frac{u}{\sqrt{1 - t^2} } , \frac{v }{\sqrt{1 - t^2}}  \right) e^{ - \frac{(z - u)^2 }{2t^2} }   \int_{\br} \left( \abs{z} + \abs{\alpha + v}  \right)   e^{ - \frac{\alpha^2 }{2t^2} } \,d\alpha \, du \, dv \nonumber  \\
             \leq & \frac{B}{2 \pi t^4 (1 - t^2 )} \iint_{\br^2}  \rho\left( \frac{u}{\sqrt{1 - t^2} } , \frac{v }{\sqrt{1 - t^2}}  \right) e^{ - \frac{(z - u)^2 }{2t^2} }   \int_{\br} \left( \abs{z} + \abs{\alpha} + \abs{v}  \right)   e^{ - \frac{\alpha^2 }{2t^2} } \, d\alpha \, du \,dv \nonumber  \\
             \leq & \frac{B}{2 \pi t^4 (1 - t^2 )} \iint_{\br^2}  \rho\left( \frac{u}{\sqrt{1 - t^2} } , \frac{v }{\sqrt{1 - t^2}}  \right) e^{ - \frac{(z - u)^2 }{2t^2} } \left( \abs{z} + \frac{\sqrt{2} t}{\sqrt{\pi}} + \abs{v}  \right)  \int_{\br}    e^{ - \frac{\alpha^2 }{2t^2} } \, d\alpha \, du \, dv . \label{b12} \\ 
             \intertext{Reversing the order of integration again.}
             \abs{ \partial_{\theta } \rho_{\theta} (z)}\leq  & \frac{B}{t} \int_{\br}  \frac{1}{2 \pi t^3 (1 - t^2 )}  \iint_{\br^2}  \left( \abs{z} + \frac{\sqrt{2} t}{\sqrt{\pi}} + \abs{v}  \right)  \rho\left( \frac{u}{\sqrt{1 - t^2} } , \frac{v }{\sqrt{1 - t^2}}  \right) e^{ - \frac{\left( u -z  \right)^2 + \left( v - y  \right)^2 }{2t^2} } du \, dv \,   dy . \\ 
             \intertext{But noting that the support of $ \rho $ is atmost $ B $, we get $ |v| \leq \sqrt{1 - t^2} B $, which gives us, }
             \abs{\partial_{\theta } \rho_{\theta} (z)} \leq & \frac{B  \left( \abs{z} + \frac{\sqrt{2} t}{\sqrt{\pi}} + B \sqrt{1-t^2}    \right)  }{t^2} \rho_{\theta} (z).
        \end{align}
        \autoref{b10} follows from the inequality $ \sqrt{x + y} \leq \sqrt{x} + \sqrt{y} $, \autoref{b11} follows from the convolution formula for the density, \autoref{b12} follows from explicitly computing the integral in the previous equation
        Plugging this into the first integral, we get 
        \begin{align*}
        	\int_{\br} \frac{\left( \partial_{\theta} \rho_{\theta}(z)  \right)^2 }{\rho_{\theta} (z) } dz  \leq & \int_{\br}   \frac{B  \left( \abs{z} + \frac{\sqrt{2} t}{\sqrt{\pi}} + B \sqrt{1-t^2}    \right)  }{t^2} \rho_{\theta} (z) \frac{ \abs{\partial_{\theta} \rho (z)}     }{\rho_{\theta}  (z) } dz \\
            \leq & \int_{\br}   \frac{B  \left( \abs{z} + \frac{\sqrt{2} t}{\sqrt{\pi}} + B \sqrt{1-t^2}    \right)  }{t^2} \abs{\partial_{\theta} \rho  (z)}  dz \\ 
            \leq & \int_{-10B}^{10B}   \frac{B  \left( \abs{z} + \frac{\sqrt{2} t}{\sqrt{\pi}} + B \sqrt{1-t^2}    \right)  }{t^2} \abs{\partial_{\theta} \rho   (z)} dz + \int_{ \abs{z} > 10B }  \frac{B  \left( \abs{z} + \frac{\sqrt{2} t}{\sqrt{\pi}} + B \sqrt{1-t^2}    \right)  }{t^2} \abs{\partial_{\theta} \rho   (z)}  dz . \\ 
            \intertext{Applying the bound from \autoref{b9}, we get }
           	\int_{\br} \frac{\left( \partial_{\theta} \rho_{\theta}  \right)^2 }{\rho_{\theta} (z) } dz \leq & \int_{-10B}^{10B }  \left( \frac{100 B^3  \sqrt{1 - t^2}}{ 2 \pi   t^4}  +  \frac{z^2   e^{ - \frac{  (0.9)^2 z^2}{2t^2}  }   B  \sqrt{1 - t^2}}{ 2 \pi   t^3} C_1(z,t) \right) \frac{B  \left( \abs{z} + \frac{\sqrt{2} t}{\sqrt{\pi}} + B \sqrt{1-t^2}    \right)  }{t^2}  dz \\ & \qquad   + \int_{\abs{z} > 10B}  \left( \frac{z^2   e^{ - \frac{  (0.9)^2 z^2}{2t^2}  }   B  \sqrt{1 - t^2}}{ 2 \pi   t^3} C_1  \right) \frac{B  \left( \abs{z} + \frac{\sqrt{2} t}{\sqrt{\pi}} + B \sqrt{1-t^2}    \right)  }{t^2} \, dz \\
            \leq & \mathsf {poly}\left(B, \frac{1}{t}  \right) .
        \end{align*}
    Similarly, we get 
    \begin{align*}
    	 \int_{\br}  \abs{\partial_{\theta}^2 \rho_{\theta}(z) }  dz \leq & \,\mathsf{poly}\left( B, \frac{1}{t} \right ).
    \end{align*}
    Similarly, 
    \begin{align*}
    	\int_{\br} \abs{ \partial_{\theta}^2 \rho_{\theta}(z) }\abs{ \log \rho_{\theta} \left(z\right)  } dz \leq & \,\mathsf{poly}\left( B, \frac{1}{t} \right).
    \end{align*}
    This gives us that the second derivative with respect to $ \theta $ is bounded. We then express $ \partial_x^2 S $ in terms of $ \partial^2_{\theta} S $.
    First note that from \autoref{lem:ent_scale}, we have $  S(\lambda W) = S(W) -  \log(\lambda) + (\lambda^2 - 1)/2  $  whenever $ W  $ is a random variable with unit variance and zero mean.
    Note that since changing the radial parameter is simply scaling  the random variable, we get $ \partial_r S = r-\frac{1}{r} $. 
     Note that from the chain rule, we get 
    \begin{equation*}
    	\partial_x^2 S = \left( (\cos^2\theta) \partial^2_r-(2\cos\theta\sin\theta) \frac{1}{r} \partial_{\theta} \partial_{r}   +2(\cos\theta\sin\theta) \frac{1}{r^2}\partial_{\theta} +(\sin^2\theta )\frac1r \partial_{r}+(\sin^2\theta) \frac{1}{r^2}\partial^2_{\theta}\right) S
    \end{equation*}
    
     Noting that $ \partial_{\theta} \partial_{r} S = 0 $ and $ \partial_r^2 S $ is constant since we evaluate at $ r=1  $, we get that $ \abs{\partial_x^2 S}$  is bounded by a polynomial of $ \abs{\partial_{\theta}^2 S}  $ and $ \abs{\partial_{\theta } S} $. 
    This gives us that the Lipshitz constant $ L = \mathsf{poly}\left( B , \frac{1}{t} \right)$. 
    Note that the degree of the polynomial can be taken to be ten. 
    
    We have proven the Lipschitz bound for random variables with support contained in $ \left[ - B , B\right]^n $. 
    But, note that since the random variable of interest is subgaussian, we can truncate the random variable at $ \mathsf{poly}(K) $ and have vanishingly small mass outside the truncated region. This would let us repeat the same argument with the truncated random variable.
    That gives us a Lipschitz constant that is bounded by $ \mathsf{poly}\left( K, n , \frac{1}{t}  \right) $. 
    
    \todo[inline,color=green]{The truncation argument need not even be made explicit. Say we truncate the distribution at a point $ A \sim (K^{10} n^{10}) $ and consider the algorithm running on this truncated distribution. Since the two distributions are so close in total variation no test will be able to dsitibguish them with any reasonable probability and thus the algorithm will behave the same way with both distributions (with overwhelmingly high probability)}
    

\end{document}